\documentclass{article}
\usepackage{microtype}
\usepackage{booktabs}
\usepackage[nohyperref,preprint]{icml2026}

\usepackage{algorithm}
\usepackage{algorithmicx}
\usepackage[noend]{algpseudocode}
\usepackage{amsmath}
\usepackage{amssymb}
\usepackage{amsthm}
\usepackage{bbm}
\usepackage{bm}
\usepackage{caption}
\usepackage{color}
\usepackage{colortbl}
\usepackage{dirtytalk}
\usepackage{dsfont}
\usepackage{enumerate}
\usepackage{graphicx}
\usepackage{listings}
\usepackage{mathtools}
\usepackage{natbib}
\usepackage{soul}
\usepackage{subcaption}
\usepackage{tcolorbox}
\usepackage{times}
\usepackage{wrapfig}
\usepackage{xspace}
\usepackage{pgfplots}
\usepackage{pgfplotstable}
\usepackage{tikz}

\usepackage[dvipsnames]{xcolor}
\usepackage[bookmarks=false]{hyperref}
\hypersetup{
  linktocpage=true,
  colorlinks=true,
  citecolor=Green,
  linkcolor=Green,
  urlcolor=Green
}
\usepackage[capitalize,noabbrev]{cleveref}

\definecolor{model1}{RGB}{0,0,0}       % Black for Baseline
\definecolor{model2}{RGB}{255,0,0}     % Red for Total Reward  
\definecolor{model3}{RGB}{0,0,255}     % Blue for Per-Step
\definecolor{model4}{RGB}{34,139,34}   % Green for K=2
\definecolor{model5}{RGB}{255,165,0}   % Orange for K=3
\definecolor{model6}{RGB}{128,0,128}   % Purple for K=4
\definecolor{model7}{RGB}{165,42,42}   % Brown for K=5

\usepackage[textsize=tiny]{todonotes}
\usepackage{thmtools}
\declaretheorem[name=Theorem,refname={Theorem,Theorems},Refname={Theorem,Theorems}]{theorem}
\declaretheorem[name=Lemma,refname={Lemma,Lemmas},Refname={Lemma,Lemmas},sibling=theorem]{lemma}

\newcommand{\cB}{\mathcal{B}}
\newcommand{\cC}{\mathcal{C}}

\newcommand{\cG}{\mathcal{G}}

\newcommand{\cR}{\mathcal{R}}
\newcommand{\cS}{\mathcal{S}}

\newcommand{\cZ}{\mathcal{Z}}

\newcommand{\naturalset}{\mathbb{N}}

\newcommand{\E}[1]{\mathbb{E}\left[#1\right]}
\newcommand{\condE}[2]{\mathbb{E}\left[#1 \,\middle|\, #2\right]}
\newcommand{\Erv}[2]{\mathbb{E}_{#1}\left[#2\right]}
\newcommand{\condErv}[3]{\mathbb{E}_{#1}\left[#2 \,\middle|\, #3\right]}

\newcommand{\Erho}[2]{\mathbb{E}_{#1}^\rho\left[#2\right]}
\newcommand{\prob}[1]{\mathbb{P}\left(#1\right)}

\newcommand{\abs}[1]{\left|#1\right|}

\newcommand{\I}[1]{\mathds{1} \! \left\{#1\right\}}

\newcommand{\normw}[2]{\|#1\|_{#2}}

\newcommand{\set}[1]{\left\{#1\right\}}

\DeclareMathOperator*{\argmax}{arg\,max\,}

\mathchardef\mhyphen="2D

\newcommand{\base}{\ensuremath{\color{Green}\tt Base}\xspace}
\newcommand{\ppg}{\ensuremath{\color{Green}\tt PPG}\xspace}
\newcommand{\ppgk}[1]{\ensuremath{\color{Green}\tt #1\mhyphen Step\mhyphen PG}\xspace}
\newcommand{\ppgone}{\ensuremath{\color{Green}\tt GreedyPG}\xspace}
\newcommand{\ppgseg}{\ensuremath{\color{Green}\tt K\mhyphen Seg\mhyphen PG}\xspace}
\newcommand{\pg}{\ensuremath{\color{Green}\tt PG}\xspace}

\icmltitlerunning{Partial Policy Gradients for RL in LLMs}

\begin{document}

\twocolumn[
\icmltitle{Partial Policy Gradients for RL in LLMs}

\icmlsetsymbol{equal}{*}

\begin{icmlauthorlist}
\icmlauthor{Puneet Mathur}{adobe}
\icmlauthor{Branislav Kveton}{adobe}
\icmlauthor{Subhojyoti Mukherjee}{adobe}
\icmlauthor{Viet Dac Lai}{adobe}
\end{icmlauthorlist}

\icmlaffiliation{adobe}{Adobe Research}

\icmlcorrespondingauthor{Puneet Mathur}{puneetm@adobe.com}

\vskip 0.3in
]

\printAffiliationsAndNotice{}

\begin{abstract}
Reinforcement learning is a framework for learning to act sequentially in an unknown environment. We propose a natural approach for modeling policy structure in policy gradients. The key idea is to optimize for a subset of future rewards: smaller subsets represent simpler policies, which can be learned more reliably because their empirical gradient estimates are more accurate. Our approach allows for modeling and comparison of different policy classes, including full planning, greedy, $K$-step lookahead, and segment policies. We evaluate the policies empirically on multiple persona-alignment conversational problems. Different policies excel in different problems, reflecting their different characteristics and highlighting the importance of our studied policy class.
\end{abstract}

\section{Introduction}
\label{sec:introduction}

\emph{Reinforcement learning (RL)} \citep{sutton98reinforcement} is a framework for learning to act sequentially in an unknown environment with the goal of maximizing future rewards. Due to its generality and broad applicability, RL has been studied extensively and many RL algorithms have been proposed, including temporal-difference learning \citep{sutton88learning}, Q-learning \citep{watkins92qlearning}, policy gradients \citep{williams92simple}, and actor-critic methods \citep{sutton00policy}. RL is more data demanding than supervised learning because it plans for the future \citep{sutton98reinforcement}. Therefore, many early works on planning and RL tried to increase the statistical efficiency of learning algorithms by assuming additional structure, such as a known model of the environment \citep{bellman57dynamic,boutilier95exploiting,koller99computing}, linear value or Q functions \citep{tsitsiklis97analysis,defarias03linear,hauskrecht04linear}, and baseline subtraction in policy gradients \citep{sutton00policy,baxter01infinitehorizon,munos06geometric}.

\emph{Proximal policy optimization (PPO)} \citep{schulman17proximal} and \emph{group-relative policy optimization (GRPO)} \citep{shao24deepseekmath} are arguably the most popular RL algorithms for \emph{large language models (LLMs)} \citep{ouyang22training}. PPO is a policy gradient algorithm where the advantage is estimated using a per-token reward model. When the policy is learned from off-policy data, the advantage is corrected by clipped propensity scores \citep{horwitz52generalization,ionides08truncated}. The challenge with applying PPO in practice is that a per-token reward model has to be learned. GRPO is a policy gradient algorithm where the advantage is estimated by simulation, which eliminates the need for a per-token reward model. The challenge with applying GRPO is that the advantage is attributed equally to all tokens, potentially reducing statistical efficiency.

We propose a natural approach for introducing policy structure in policy gradients. The first step is decomposing the total reward additively over time. This decomposition may be a part of problem definition, for instance when the problem is a \emph{Markov decision process (MDP)} with per-step rewards \citep{bellman57dynamic,puterman94markov}, or obtained by a simple algebraic transformation. The second step is optimization for subsets of future rewards. The subsets lead to simpler policies than full planning, which can be learned more reliably because their empirical gradient estimates are more accurate. We propose both online and offline algorithms for learning our policies, and analyze them. Full planning, greedy, and $K$-step lookahead policies are all instances of our framework.

LLMs are increasingly used to simulate human personas for social role-play \cite{personasurvey,jandaghi2024faithful}. Despite that, off-the-shelf LLMs frequently drift from their assigned personas, contradict to their earlier statements, or abandon role-appropriate behavior \citep{ji-etal-2025-enhancing,abdulhai2025consistently}. We apply our policies to LLMs to maintain consistent and faithful human personas in role-playing dialogues. This problem is challenging because many dialogues have long horizons, on average $20$ steps and up to $60$. Our results show that $K$-step lookahead policies lead to consistent personas in extended dialogues, mitigating the persona drift that plagues base models.

We make the following contributions:
\begin{enumerate}
  \item We propose a framework for modeling policy structure in policy gradients. The key idea is to optimize for a subset of future rewards, which can be either observed or obtained by a simple algebraic transformation. Our formulation is more general than those in the closest related works, which focused only on specific policies \citep{kveton25adaptive,guo25segment}, sometimes even without observations.
  \item Our framework trades off the complexity of learned policies for the statistical efficiency of learning them. Specifically, optimization for smaller subsets of future rewards leads to learning simpler policies, which can be learned more reliably because their empirical gradient estimates are more accurate. We propose both online and offline algorithms for learning the policies.
  \item We instantiate the framework: online policy gradient \citep{williams92simple}, policy gradient for adaptive submodular policies \citep{kveton25adaptive} and segment-level credit assignment \citet{guo25segment}, can be all viewed as instances of our framework. This is the first work that proposes and empirically evaluates $K$-step lookahead policies in LLMs.
  \item We empirically evaluate our policies on four persona-alignment conversational problems across multiple domains: education, therapy, chatting, and generic. We experiment with three LLMs: Qwen \citep{yang2025qwen3}, Llama \citep{grattafiori2024llama}, and Gemma \cite{team2025gemma}. We show that $K$-step lookahead policies lead to consistent persona across extended dialogues (\cref{sec:performance_results}) by preventing persona drift (\cref{sec:stability_results}). The optimal value of $K$ depends on the domain camplexity and the amount of training data (\cref{sec:stat_efficiency_results}).
\end{enumerate}

The paper is organized as follows. We present our setting in \cref{sec:setting}. In \cref{sec:algorithms}, we present our framework for optimizing a subset of future rewards, interpret it, and analyze it. We instantiate the framework in \cref{sec:ppg instances} and evaluate it empirically in \cref{sec:experiments}. Related works are reviewed in \cref{sec:related work}, and we conclude in \cref{sec:conclusions}.

\section{Setting}
\label{sec:setting}

We start with introducing our notation. The marginal and conditional probabilities under the probability measure $p$ are $p(X = x)$ and $p(X = x \mid Y = y)$, respectively; and we write $p(x)$ and $p(x \mid y)$ when the random variables are clear from context. The indicator function is $\I{\cdot}$. For a positive integer $n$, we define $[n] = \set{1, \dots, n}$. We denote the $i$-th entry of vector $v$ by $v_i$. If the vector is already indexed, such as $v_j$, we write $v_{j, i}$.

We consider a general reinforcement learning problem \citep{sutton98reinforcement} where an \emph{agent} interacts with an \emph{environment} in an \emph{episode} consisting of multiple steps. The agent takes actions conditioned on the interaction history and the environment responds. When the episode ends, it is assigned a reward. The reward measures the quality of the interaction and the agent maximizes it.

We formalize the problem as follows. The agent first observes \emph{context} $x \in \cS$, where $\cS$ is the space of all token sequences. The context defines the task. The interaction between the agent and environment consists of steps indexed by $t \in \naturalset$, where $\naturalset$ is a set of positive integers. At step $t$, the agent takes an \emph{action} $a_t \in \cS$ and the environment responds with an \emph{observation} $y_t \in \cS$. The interaction history up to step $n$ is a \emph{trajectory} $\tau_n = (a_1, y_1, \dots, a_n, y_n)$ of $n$ actions and observations. We assume that the number of steps $n$ is fixed. The \emph{reward} is a non-negative function of $x$ and $\tau_n$, denoted by $r(x, \tau_n) \geq 0$, and measures the quality of the interaction. We do not make any additional assumptions on the reward. This is to maintain generality and because our algorithms (\cref{sec:algorithms}) do not require it.

The agent follows a policy conditioned on its history. Specifically, the probability that action $a$ is taken in context $x$ and history $\tau_{t - 1}$ is $\pi(a \mid x, \tau_{t - 1}; \theta)$, and is parameterized by $\theta \in \Theta$. We call $\theta$ a \emph{policy} and $\Theta$ the space of policy parameters. The probability of observing $y_t$ conditioned on history $\tau_{t - 1}$ and action $a_t$ is denoted by $p(y_t \mid x, \tau_{t - 1}, a_t)$. We slightly abuse our notation and denote the probability of trajectory $\tau_n$ in context $x$ under policy $\theta$ by
\begin{align*}
  \pi(\tau_n \mid x; \theta)
  = \prod_{t = 1}^n p(y_t \mid x, \tau_{t - 1}, a_t) \,
  \pi(a_t \mid x, \tau_{t - 1}; \theta)\,.
\end{align*}
The factorization follows from the chain rule of probability. The expected value of policy $\theta$, where $q$ is a distribution over contexts $x$, is defined as
\begin{align}
  V(\theta)
  = \Erv{x \sim q, \, \tau_n \sim \pi(\cdot \mid x; \theta)}{r(x, \tau_n)}\,.
  \label{eq:online objective}
\end{align}
Our goal is to learn a policy $\theta_* = \argmax_{\theta \in \Theta} V(\theta)$ with the highest expected value. Our formulation is more general than those in the closest related works, which focused only on specific policies and did not consider observations \citep{kveton25adaptive,guo25segment}.

Our framework can model multiple use cases. For instance, suppose that we want to maximize the pedagogical value of a conversation over $n$ steps \citep{scarlatos2025}. Then $\tau_n$ is an $n$-step conversation and $r(x, \tau_n)$ is its pedagogical value. We experiment with such problems in \cref{sec:experiments}. In math reasoning problems \citep{shao24deepseekmath}, $\tau_n$ would be agent's reasoning and $r(x, \tau_n)$ would indicate whether its response is correct.

\section{Algorithms}
\label{sec:algorithms}

This section has four main parts. In \cref{sec:pg}, we introduce policy gradients. In \cref{sec:ppg}, we present partial policy gradients. We interpret them in \cref{sec:interpretation} and introduce offline partial policy gradients in \cref{sec:offline ppg}.

\subsection{Policy Gradient}
\label{sec:pg}

\begin{algorithm}[t!]
  \caption{Meta policy gradient algorithm.}
  \label{alg:meta pg}
  \begin{algorithmic}[1]
    \State \textbf{Input:} Learning rate schedule $(\alpha_j)_{j \in \naturalset}$, batch size $B$, gradient estimator $g$
    \State Initialize $\theta$
    \State $j \gets 1$
    \While{(not convergence)}
      \For{$i = 1, \dots, B$}
        \State $x_i \sim q$
        \State $\tau_{i, n} \sim \pi(\cdot \mid x_i; \theta)$
        \State $r_i \gets r(x_i, \tau_{i, n})$
      \EndFor
      \State Batch $\cB \gets \set{(x_i, \tau_{i, n}, r_i)}_{i = 1}^B$
      \State $\theta \gets \theta + \alpha_j g(\cB)$
      \State $j \gets j + 1$
    \EndWhile
    \State \textbf{Output:} Learned policy $\theta$
  \end{algorithmic}
\end{algorithm}

We want to maximize the expected policy value $V(\theta)$ in \eqref{eq:online objective}. While this can be done many ways \citep{sutton98reinforcement}, the most straightforward method for complex policies, like those represented by LLMs, are policy gradients \citep{williams92simple}. The key idea in an \emph{online policy gradient} is to update the policy $\theta$ iteratively by gradient ascent. The gradient of $V(\theta)$ at $\theta$ is
\begin{align}
  \nabla V(\theta)
  & = \E{r(x, \tau_n) \nabla \log \pi(\tau_n \mid x; \theta)}
  \label{eq:pg} \\
  & = \E{r(x, \tau_n) \sum_{t = 1}^n
  \nabla \log \pi(a_t \mid x, \tau_{t - 1}; \theta)}\,,
  \nonumber
\end{align}
where $\E{\cdot} = \Erv{x \sim q, \, \tau_n \sim \pi(\cdot \mid x; \theta)}{\cdot}$ is an expectation over the same random variables as in \eqref{eq:online objective}. This identity can be derived by a direct application of the score identity \citep{aleksandrov68stochastic}.

All policy gradient algorithms in this work are variants of the meta algorithm in \cref{alg:meta pg}. The gradient in Line 10 is estimated from a batch $\cB$ of trajectories (Line 9) of size $B$, which are sampled in Line 7 and assigned rewards in Line 8. We denote the $i$-th trajectory in $\cB$ by $\tau_{i, n}$; and its context and reward by $x_i$ and $r_i$, respectively. The gradient estimator for the classic online policy gradient is
\begin{align}
  g(\cB)
  = \frac{1}{B} \sum_{i = 1}^B r_i \sum_{t = 1}^n
  \nabla \log \pi(a_{i, t} \mid x_i, \tau_{i, t - 1}; \theta)\,.
  \label{eq:empirical pg}
\end{align}
We refer to this instance of \cref{alg:meta pg} as \pg.

\subsection{Partial Policy Gradient}
\label{sec:ppg}

While the gradient in \eqref{eq:pg} is valid, the reward is attributed equally to each step, since all $\nabla \log \pi(a_t \mid x, \tau_{t - 1}; \theta)$ are multiplied by the same reward $r(x, \tau_n)$. If the attribution could be improved, we could reduce the variance in estimator \eqref{eq:empirical pg} and improve statistical efficiency. We propose a general attribution framework that allows for such tradeoffs. The framework has two components.

\textbf{Factored reward.} We assume that the reward factors as
\begin{align}
  r(x, \tau_n)
  = \sum_{t = 1}^n r_t\,, \quad
  r_1, \dots, r_n \geq 0\,,
  \label{eq:additive reward}
\end{align}
where $r_t = r(x, \tau_t)$ denotes the \emph{reward at step $t$}, and is a function of $x$ and $\tau_t$. The reward $r_t$ with such properties can be either observed, for instance if the environment is an MDP with per-step rewards, or designed. We present a general design that guarantees \eqref{eq:additive reward} next.

\begin{lemma}
\label{lem:additive reward} Let $f(x, \tau_t)$ be any function of $x$ and $\tau_t$ such that $0 = f(x, \tau_0) \leq \dots \leq f(x, \tau_n) = r(x, \tau_n)$. Let
\begin{align*}
  r_t
  = f(x, \tau_t) - f(x, \tau_{t - 1})
\end{align*}
hold for all $t \in [n]$. Then $\sum_{t = 1}^n r_t = r(x, \tau_n)$ and $r_t \geq 0$ for all $t \in [n]$.
\end{lemma}

The proof of the claim is in \cref{sec:additive reward proof}. Since $f(x, \tau_t)$ can be viewed as a trajectory reward up to step $t$, $r_t$ is the per-step gain in the total trajectory reward at step $t$. If $f$ is not monotone, a monotone $\bar{f}$ can be obtained by a transformation $\bar{f}(x, \tau_t) = \max_{\ell \in [t]} f(x, \tau_\ell)$. Because of this, and without loss of generality, we assume that $r_t \geq 0$ in the rest of our paper, unless stated otherwise. One natural choice of $f$ is an LLM judge \citep{zheng23judging} that can judge partial trajectories. Similar decompositions were proposed before in more specific problems, without observations \citep{guo25segment} and in optimization of submodular policies \citep{kveton25adaptive}.

\textbf{Reward attribution.} We consider a general \emph{partial policy gradient} of the form
\begin{align}
  \E{\sum_{t = 1}^n \left(\sum_{\ell \in \cR_t} r_\ell\right)
  \nabla \log \pi(a_t \mid x, \tau_{t - 1}; \theta)}\,,
  \label{eq:ppg}
\end{align}
where $\cR_t \subseteq [n] \setminus [t - 1]$ are \emph{reward indices affected by action at step $t$} through $\nabla \log \pi(a_t \mid x, \tau_{t - 1}; \theta)$. Specifically, $\cR_t$ is the subset of future steps whose rewards are affected by $a_t$ conditioned on the history. The gradient can be estimated empirically as $g(\cB) =$
\begin{align}
  \frac{1}{B} \sum_{i = 1}^B \sum_{t = 1}^n
  \left(\sum_{\ell \in \cR_t} r_{i, \ell}\right)
  \nabla \log \pi(a_{i, t} \mid x_i, \tau_{i, t - 1}; \theta)\,,
  \label{eq:empirical ppg}
\end{align}
where $\tau_{i, n}$ is the $i$-th trajectory in batch $\cB$; and $a_{i, \ell}$, $r_{i, \ell}$, and $\tau_{i, \ell - 1}$ are the action, reward, and trajectory history at step $t$ of this trajectory, respectively. When \eqref{eq:empirical ppg} is used in \cref{alg:meta pg}, we call it \underline{p}artial \underline{p}olicy \underline{g}radient (\ppg) since only a subset of future rewards is used in policy optimization. Note that \ppg reduces to \pg when $\cR_t = [n] \setminus [t - 1]$ for all $t \in [n]$, as we show in \cref{sec:full pg}. Our framework allows us to reason about other $\cR_t$ and understand associated statistical tradeoffs.

\begin{figure*}[t!]
  \centering
  \includegraphics[width=6.75in,bb=0in 0.25in 10in 2.5in,clip]{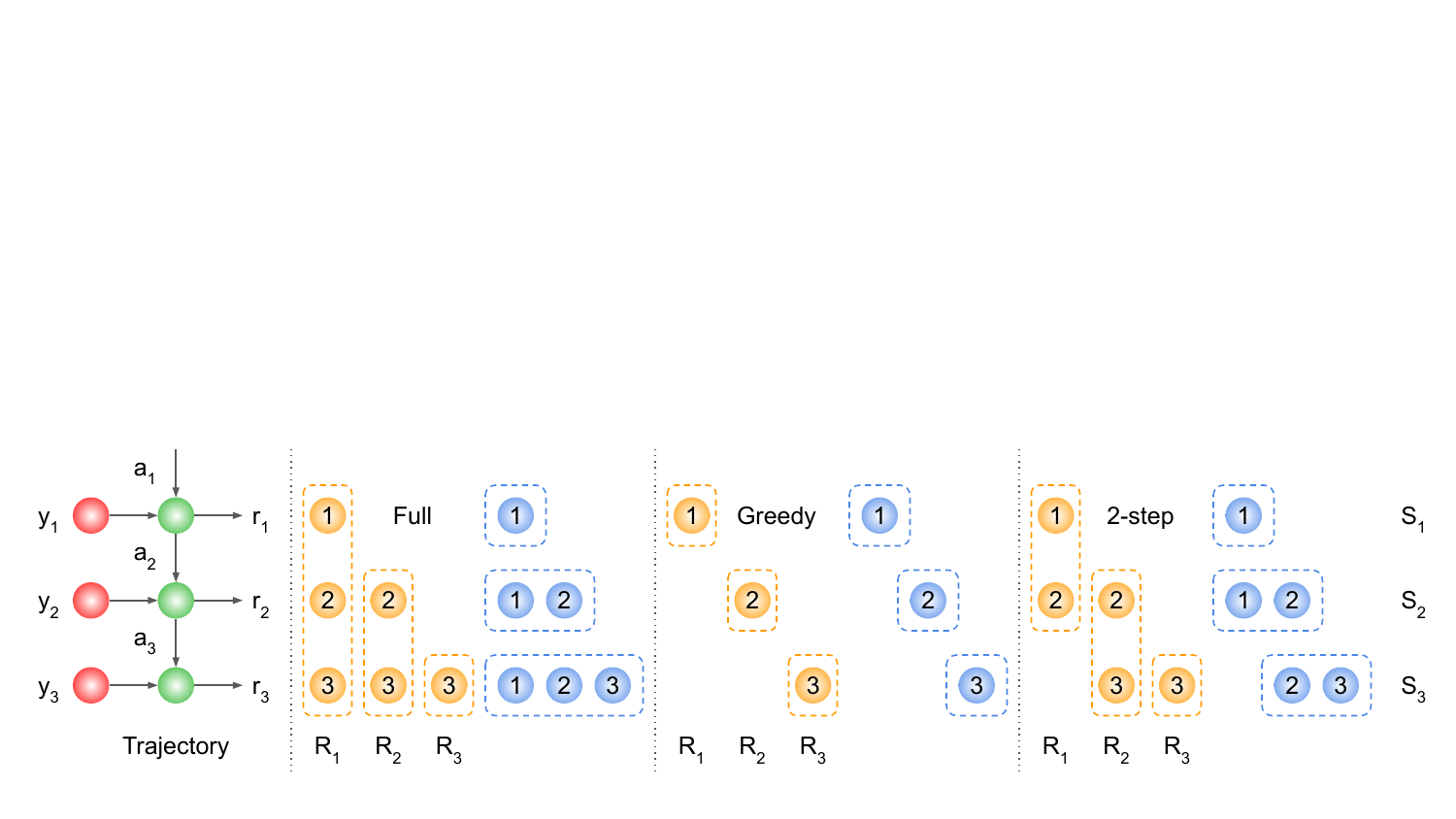}
  \caption{Examples of reward indices $\cR_t$ (yellow circles in columns) and action indices $\cS_t$ (blue circles in rows) in full, greedy, and $2$-step lookahead policies.}
  \label{fig:ppg}
\end{figure*}

\ppg is more statistically efficient than \pg, and we prove this formally in \cref{sec:gradient concentration}. Specifically, \cref{thm:gradient concentration} says that a partial policy gradient estimator $g(\cB)$ that optimizes for a subset of rewards of another estimator $g'(\cB)$ concentrates at $\E{g(\cB)}$ faster that $g'(\cB)$ concentrates at $\E{g'(\cB)}$. The key idea in the proof is to treat \eqref{eq:empirical ppg} as a sum of random vectors, independent across $i$ but dependent across $t$. Then we use Hoeffding's inequality, which is a standard tool in concentration analyses \citep{raginsky12concentration,boucheron13concentration}, to argue that the sum of a subset of random vectors concentrates faster than the sum of the original set, when the vectors are sampled from the same distribution.

\subsection{Interpretation}
\label{sec:interpretation}

We interpret gradient ascent with respect to \eqref{eq:ppg} next. That requires additional notation. Let $\cS_t = \set{\ell \in [t]: t \in \cR_\ell} \subseteq [t]$ be \emph{action indices that affect the reward at step $t$}, $r_t$. Note that by definition,
\begin{align}
  & \sum_{t = 1}^n \left(\sum_{\ell \in \cR_t} r_\ell\right)
  \nabla \log \pi(a_t \mid x, \tau_{t - 1}; \theta)
  \label{eq:action-reward-sum swap} \\
  & = \sum_{t = 1}^n r_t \sum_{\ell \in \cS_t}
  \nabla \log \pi(a_\ell \mid x, \tau_{\ell - 1}; \theta)\,.
  \nonumber
\end{align}
This is because we can first sum over log-probabilities and then over the associated rewards, or the other way around. \cref{fig:ppg} illustrates $\cR_t$ and $\cS_t$ in full, greedy, and $K$-step lookahead policies in \cref{sec:ppg instances}. We also define $\bar{S_t} = [t] \setminus \cS_t$, the complement of $\cS_t$ with respect to the first $t$ steps. For any $\cS, \cZ \subseteq [n]$, we define
\begin{align*}
  \Erv{\cS}{\cdot}
  & = \Erv{x \sim q, \, \tau_\cS \sim \pi(\cdot \mid x; \theta)}{\cdot}\,, \\
  \condErv{\cS}{\cdot}{\cZ}
  & = \condErv{\tau_\cS \sim \pi(\cdot \mid x; \theta)}{\cdot}{x, \tau_\cZ}\,.
\end{align*}
Here $\Erv{\cS}{\cdot}$ is the expectation over actions and observations at steps $\cS$, and context $x$; and $\condErv{\cS}{\cdot}{\cZ}$ is the expectation over actions and observations at steps $\cS$, conditioned on actions and observations at steps $\cZ$, and context $x$. With the new notation in hand, we get the following claim.

\begin{lemma}
\label{lem:ppg} Let $x \sim q$ and $\tau_n \sim \pi(\cdot \mid x; \theta)$. Then for any $\cR_t \subseteq [n] \setminus [t - 1]$, the gradient in \eqref{eq:ppg} is equal to
\begin{align*}
  \E{\sum_{t = 1}^n \nabla \condErv{\cS_t}{r_t}{\bar{\cS}_t}}\,.
\end{align*}
\end{lemma}

The proof is in \cref{sec:ppg proof}. The lemma says that the gradient in \eqref{eq:ppg} increases the sum of expected per-step rewards $r_t$ through actions at steps $\cS_t$ that affect the corresponding expected rewards. The histories $\bar{\cS}_t$ are sampled from $\pi$. We discuss different forms of these policies in \cref{sec:ppg instances}.

\subsection{Offline Partial Policy Gradient}
\label{sec:offline ppg}

So far, we focused on online policy gradients (\cref{sec:ppg}), where the trajectories are sampled from learned policy $\pi$. Another common and practical setting is offline RL \citep{jaques20humancentric,levine20offline}, where the trajectories are sampled from a logged dataset. We represent it by another distribution $\rho$ and propose an \emph{offline partial policy gradient} of the form
\begin{align}
  \Erho{}{\sum_{t = 1}^n \left(\sum_{\ell \in \cR_t} r_\ell\right)
  \nabla \log \pi(a_t \mid x, \tau_{t - 1}; \theta)}\,,
  \label{eq:offline ppg}
\end{align}
where $\Erho{}{\cdot} = \Erv{x \sim q, \, \tau_n \sim \rho(\cdot \mid x; \theta)}{\cdot}$. Note that the only difference from \eqref{eq:ppg} is that the trajectories are sampled from another policy $\rho$. Therefore, the empirical estimator of \eqref{eq:offline ppg} from batch $\cB$ has the same algebraic form as \eqref{eq:empirical ppg}.

Based on our discussion, an offline variant of \ppg can be obtained by changing the sampling distribution in Line 7 of \cref{alg:meta pg} to $\tau_{i, n} \sim \rho(\cdot \mid x; \theta)$. Our gradient estimator is motivated by \citet{mukherjee25offline} and can be related to maximizing the original online objective as follows.

\begin{lemma}
\label{lem:offline ppg} Let all quantities be defined as in \cref{lem:ppg}. Let $\mathbb{E}^\rho$ be an expectation with respect to another policy $\rho$. Then for any $\cR_t \subseteq [n] \setminus [t - 1]$, the gradient in \eqref{eq:offline ppg} is equal to
\begin{align*}
  \mathbb{E}^\rho\Bigg[\sum_{t = 1}^n \nabla
  \underbrace{\Erho{\cS_t}{r_t \sum_{\ell \in \cS_t}
  \log \pi(a_\ell, y_\ell \mid x, \tau_{\ell - 1}; \theta) \,\middle|\,
  \bar{\cS}_t}}_{f_t(\theta)}\Bigg]\,.
\end{align*}
When $r_t \geq 0$ for all $t \in [n]$, $f_t(\theta) \leq \condErv{\cS_t}{r_t}{\bar{\cS}_t}$ holds. In general, $|f_t(\theta) - \condErv{\cS_t}{r_t}{\bar{\cS}_t}| \leq C(\pi, \rho)$, where $C(\pi, \rho)$ is a constant depending on the propensity ratios of $\pi$ and $\rho$, and their logarithms.
\end{lemma}

The proof is in \cref{sec:offline ppg proof}. The first claim says that the gradient increases the sum of expected per-step rewards $r_t$ weighted by log-probabilities under logging policy $\rho$, denoted by $f_t(\theta)$, through actions at steps $\cS_t$ that affect them. The second claim says that the true objective $\condErv{\cS_t}{r_t}{\bar{\cS}_t}$ that we want to optimize is bounded from below by $f_t(\theta)$, which we optimize. The lower bound holds under the assumption that $r_t \geq 0$. Without it, $f_t(\theta)$ is still $C(\pi, \rho)$ close to $\condErv{\cS_t}{r_t}{\bar{\cS}_t}$, as stated in the last claim. 

Since \cref{lem:offline ppg} provides guarantees for negative rewards, offline \ppg can be implemented with advantages and we do that in \cref{sec:experiments}. The advantage for step $t$ in trajectory $i$ is computed as
\begin{align}
  \tilde{r}_{i, t}
  = \frac{\sum_{\ell \in \cR_t} r_{i, \ell} - \mu_\cG}{\sigma_\cG}\,,
  \label{eq:advantage}
\end{align}
where $\mu_\cG$ and $\sigma_\cG$ are the mean and standard deviation of $\sum_{\ell \in \cR_t} r_{j, \ell}$ computed jointly over $j \in \cG$ and $t \in [n]$, where $\cG$ is the group of trajectories that trajectory $i$ belongs to. In our experiments in \cref{sec:experiments}, we group the trajectories by the persona and topic.

\section{Partial Policy Gradient Instances}
\label{sec:ppg instances}

To show the generality of our framework, we derive and compare several instances of \ppg. Our concentration arguments rely on \cref{thm:gradient concentration}, which is stated and proved in \cref{sec:gradient concentration}. Segment policies are derived in \cref{sec:segment pg}. Offline policy gradients can be derived analogously.

\subsection{Full Policy Gradient}
\label{sec:full pg}

Suppose that all actions $a_t$ in \eqref{eq:ppg} affect all future rewards. Then $\cR_t = [n] \setminus [t - 1]$ for all steps $t \in [n]$. By \cref{lem:ppg}, where $\cS_t = [t]$ and $\bar{\cS}_t = \emptyset$, the gradient is
\begin{align}
  \E{\sum_{t = 1}^n \nabla \Erv{[t]}{r_t}}
  & = \E{\nabla \sum_{t = 1}^n \Erv{[t]}{r_t}}
  \label{eq:full pg} \\
  & = \E{\nabla \E{\sum_{t = 1}^n r_t}}
  = \nabla \E{r}\,.
  \nonumber
\end{align}
In plain English, all actions up to step $t$ affect the reward at step $t$; and the gradient is a sum of gradients of the corresponding expected rewards. Note that it is identical to \eqref{eq:pg}, because future actions do not affect past rewards.

\subsection{Greedy Policy Gradient}
\label{sec:greedy pg}

Suppose that all actions $a_t$ in \eqref{eq:ppg} affect only the immediate reward. Then $\cR_t = \set{t}$ for all $t \in [n]$. By \cref{lem:ppg}, where $\cS_t = \set{t}$ and $\bar{\cS}_t = [t - 1]$, the gradient is
\begin{align}
  \E{\sum_{t = 1}^n \nabla \condErv{\set{t}}{r_t}{[t - 1]}}\,.
  \label{eq:greedy pg}
\end{align}
In plain English, only the action at step $t$ affects the reward at step $t$; and the gradient is a sum of gradients of the corresponding expected rewards. We refer to \ppg with $\cR_t = \set{t}$ as \ppgone.

This policy is simpler than that in \cref{sec:full pg}, where the reward is affected by each previous action, and therefore its gradient can be estimated more efficiently. Specifically, by \cref{thm:gradient concentration}, the empirical gradient in \ppgone concentrates faster than that in \pg since $\cS_t = \set{t} \subseteq [t]$. While this policy class was proposed, analyzed, and evaluated in LLMs by \citet{kveton25adaptive}, it was done only for adaptive submodular maximization \citep{golovin11adaptive}. The connection to statistical efficiency was not established.

\subsection{$K$-Step Lookahead Policy Gradient}
\label{sec:lookahead pg}

Suppose that all actions $a_t$ in \eqref{eq:ppg} affect the next $K$ rewards. Then $\cR_t = [\min \set{t + K - 1, n}] \setminus [t - 1]$ for all $t \in [n]$. By \cref{lem:ppg}, where $\cS_t = [t] \setminus [\max \set{t - K, 0}]$ and $\bar{\cS}_t = [\max \set{t - K, 0}]$, the gradient is
\begin{align}
  \!\!\!\!\!
  \E{\sum_{t = 1}^n \nabla \condErv{[t] \setminus [\max \set{t - K, 0}]}
  {r_t}{[\max \set{t - K, 0}]}}\,.
  \label{eq:lookahead pg}
\end{align}
In plain English, only the last $K$ actions up to step $t$ including affect the reward at step $t$; and the gradient is a sum of gradients of the corresponding expected rewards. We call \ppg with this $\cR_t$ as \ppgk{K}. Our paper is the first work that proposes this policy in LLMs. For a given lookahead, such as $K = 3$, we write \ppgk{3}. Note that \ppgk{1} is identical to \ppgone in \cref{sec:greedy pg}.

This policy is simpler than that in \cref{sec:full pg}, where the reward is affected by each previous action, but also more complex than that in \cref{sec:greedy pg}, where the reward is affected only by the immediate action. The policy naturally interpolates between these two extremes. Specifically, since
\begin{align*}
  \set{t}
  \subseteq \cS_t
  = [t] \setminus [\max \set{t - K, 0}]
  \subseteq [t]
\end{align*}
holds for any $K$, we have by \cref{thm:gradient concentration} that the empirical gradient in \ppgk{K} concentrates faster than that in \pg but slower than that in \ppgone. The complexity of the policy increases and its concentration rate decreases as $K$ increases.

\section{Experiments}
\label{sec:experiments}

We describe our benchmarks in \cref{sec:data_settings}, evaluation protocol in \cref{sec:evaluation}, and baselines in \cref{sec:baselines}. The policies are evaluated in \cref{sec:performance_results}. We analyze their stability in \cref{sec:stability_results} and statistical efficiency in \cref{sec:stat_efficiency_results}. In addition, we study the policies across conversation lengths in \cref{app:results_conversation_length}, present scaling laws for Qwen, Llama, and Gemma models in generic domain in \cref{app:scaling_laws}, and give qualitative examples in \cref{app:qualitative_analysis}.

\subsection{Data Settings}
\label{sec:data_settings}

\noindent\textbf{Problem Formulation.} Our learning problem is maintaining consistent and faithful human personas over time in role-playing dialogues like therapy, education, and chatting. We instantiate the framework in \cref{sec:setting} as follows. The \emph{context} $x$ is a persona definition given at the start of the interaction in a system prompt. A \emph{trajectory} $\tau_n$ is a dialogue history between an LLM-simulated agent and the \emph{environment}. The agent's responses are actions $a_t$ and the environment's replies are observations $y_t$.

\noindent\textbf{Dataset.} We experiment with Consistent-LLMs benchmark \citep{abdulhai2025consistently}, which provides human-persona dialogues across 3 distinct domains: (i) \textbf{Education:} Tutoring sessions where a student agent maintains specific learning preferences and the tutor acts as the environment; (ii) \textbf{Therapy:} Counseling sessions where a patient agent maintains consistent mental health backgrounds while the therapist serves as the environment; and (iii) \textbf{Chatting:} Casual dialogue where a simulated human is the agent maintaining biographical consistency (family, career, and so on) against a conversational environment. Each domain comprises of $6\,500$ trajectories ($5\,200$ training and $1\,300$ test) of varying lengths ($2$ to $60$ steps, with $23$ steps on average). The trajectories are generated using a Llama-3.1-8B model.

\subsection{Evaluation}
\label{sec:evaluation}

\citet{abdulhai2025consistently} proposed two metrics to evaluate role-playing consistency: (i) \emph{Prompt consistency} measures if an agent's response aligns with their assigned persona or background. (ii) \emph{Index consistency} evaluates if an agent's response contradicts to their own previous responses in the conversation history based on their past prompt consistency scores. Since index consistency is a secondary metric derived from prompt consistency, we use prompt consistency in our work; and call it \emph{persona consistency (PC)} because it measures the consistency of a persona. Using this terminology, the reward at step $t$, $r_t = r(x, \tau_t)$, is its PC.

All policies are learned offline from the training trajectories of \citet{abdulhai2025consistently}. At test time, we measure the PC using an LLM judge with the same prompts as in \citet{abdulhai2025consistently}. Specifically, for each agent's response $a_t$, we ask the LLM judge to evaluate if it contradicts the persona description provided as context $x$ and score it $1$ (consistent) if the judge outputs \say{NO} (no contradiction) or $0$ (inconsistent) if the judge outputs \say{YES} (contradiction detected). The judge is implemented using GPT-4o mini with temperature $0$, and we report the mean scores across all steps of the evaluated trajectories. See prompts in \cref{app:eval_prompts}.

We employ two complementary evaluation protocols to assess persona consistency under different distribution conditions. The first follows the rollout methodology from \citet{abdulhai2025consistently}, while the second is designed to isolate consistency improvements from error propagation effects. 

\textbf{Rollout evaluation.} Following \citet{abdulhai2025consistently}, we generate complete conversation trajectories starting with only the initial persona description as context, alternating between responses of the agent and a fixed simulated environment implemented using Llama-3.1-70B model. This exposes the agent to distribution shifts, as inconsistencies compound over time, making it a challenging test of maintaining long-horizon consistency.

\textbf{Step-level evaluation.} To disentangle consistency improvements from error compounding, we introduce a controlled evaluation setting where each step is assessed independently. Specifically, we take the test trajectories of \citet{abdulhai2025consistently} and for each step in them, we generate the response of the agent given the conversation history of that particular trajectory. This design isolates the model's inherent consistency capability from the cascading effects of its own errors without model-induced distribution shift.

\subsection{Baselines and Compared Methods}
\label{sec:baselines}

We have two baselines. \base is an unmodified base model used zero-shot to show what can be attained without learning. PPO \cite{schulman17proximal} was used as a policy optimization algorithm in \citet{abdulhai2025consistently}, and hence is a state-of-the-art baseline in our domain. We experiment with two agent models: Llama-3.1-8B-Instruct and Qwen3-8B, which we abbreviate as \texttt{Qwen} and \texttt{Llama}, respectively.

All our policies are learned offline by an offline policy gradient (\cref{sec:offline ppg}), where $\rho$ samples trajectories uniformly at random from the training set. The reward at step $t$ is the persona consistency at step $t$ (\cref{sec:evaluation}). The rewards are standardized as in \eqref{eq:advantage}. We learn a full planning policy (\pg in \cref{sec:full pg}), greedy policy (\ppgone in \cref{sec:greedy pg}), and $K$-step lookahead policies (\ppgk{K} in \cref{sec:lookahead pg}) for $K=\{2,3,4,5\}$. Additional training details in \cref{app:training_details}.

\begin{table}[t!]
\centering
\resizebox{0.5\textwidth}{!}{%
\begin{tabular}{lcc c lcc}
\toprule
& \multicolumn{2}{c}{\textbf{\texttt{Qwen}}} & & & \multicolumn{2}{c}{\textbf{\texttt{Llama}}} \\
& Rollout & Step-level & & & Rollout & Step-level \\
\cmidrule(lr){2-2} \cmidrule(lr){3-3} \cmidrule(lr){6-6} \cmidrule(lr){7-7}
Policy & PC $\uparrow$ & PC $\uparrow$ & & Policy & PC $\uparrow$ & PC $\uparrow$ \\
\midrule
\base & 0.799 & 0.886 & & \base & 0.425 & 0.425 \\
PPO & 0.875 & 0.895 & & PPO & 0.613 & 0.624 \\
\ppgone & 0.623 & 0.930 & & \ppgone & 0.739 & 0.833 \\
\ppgk{2} & 0.930 & 0.936 & & \ppgk{2} & 0.721 & 0.759 \\
\ppgk{3} & 0.948 & 0.960 & & \ppgk{3} & 0.573 & 0.645 \\
\ppgk{4} & \underline{0.971} & \underline{0.954} & & \ppgk{4} & \underline{0.837} & 0.858 \\
\ppgk{5} & 0.881 & 0.937 & & \ppgk{5} & 0.775 & \underline{0.896} \\
\pg & \textbf{0.976} & \textbf{0.990} & & \pg & \textbf{0.913} & \textbf{0.974} \\

\bottomrule
\end{tabular}
}
\caption{Evaluation on education domain in Consistent-LLMs dataset. \textbf{Bold} / \underline{underline} is the best / second-best result.}
\label{education}
\end{table}

\begin{table}[t!]
\centering
\resizebox{0.5\textwidth}{!}{%
\begin{tabular}{lcc c lcc}
\toprule
& \multicolumn{2}{c}{\textbf{\texttt{Qwen}}} & & & \multicolumn{2}{c}{\textbf{\texttt{Llama}}} \\
& Rollout & Step-level & & & Rollout & Step-level \\
\cmidrule(lr){2-2} \cmidrule(lr){3-3}  \cmidrule(lr){6-6} \cmidrule(lr){7-7}
Policy & PC $\uparrow$ & PC $\uparrow$ & & Policy & PC $\uparrow$ & PC $\uparrow$ \\
\midrule
\base & 0.561 & 0.869 & & \base & 0.465 & 0.465 \\
PPO & 0.753 & 0.911 & & PPO & 0.625 & 0.679 \\
\ppgone & 0.840 & 0.882 & & \ppgone & 0.685 & 0.710 \\
\ppgk{2} & 0.854 & \underline{0.936} & & \ppgk{2} & \underline{0.712} & 0.734 \\
\ppgk{3} & \textbf{0.937} & \textbf{0.940} & & \ppgk{3} & \textbf{0.780} & \textbf{0.784} \\
\ppgk{4} & 0.835 & 0.830 & & \ppgk{4} & 0.673 & 0.687 \\
\ppgk{5} & \underline{0.869} & 0.890 & & \ppgk{5} & 0.690 & \underline{0.740} \\
\pg & 0.785 & 0.815 & & \pg & 0.586 & 0.595 \\
\bottomrule
\end{tabular}
}
\caption{Evaluation on therapy domain in Consistent-LLMs dataset. \textbf{Bold} / \underline{underline} is the best / second-best result.}
\label{therapy}
\end{table}

\begin{table}[t!]
\centering
\resizebox{0.5\textwidth}{!}{%
\begin{tabular}{lcc c lcc}
\toprule
& \multicolumn{2}{c}{\textbf{\texttt{Qwen}}} & & & \multicolumn{2}{c}{\textbf{\texttt{Llama}}} \\
& Rollout & Step-level & & & Rollout & Step-level \\
\cmidrule(lr){2-2} \cmidrule(lr){3-3}  \cmidrule(lr){6-6} \cmidrule(lr){7-7}
Policy & PC $\uparrow$ & PC $\uparrow$ & & Policy & PC $\uparrow$ & PC $\uparrow$ \\
\midrule
\base & 0.760 & 0.775 & & \base & 0.478 & 0.478 \\
PPO & 0.818 & 0.845 & & PPO & 0.556 & 0.558 \\
\ppgone & 0.762 & 0.881 & & \ppgone & 0.799 & \underline{0.850} \\
\ppgk{2} & 0.782 & 0.813 & & \ppgk{2} & \textbf{0.929} & \textbf{0.939} \\
\ppgk{3} & \textbf{0.985} & \textbf{0.995} & & \ppgk{3} & \underline{0.821} & 0.835 \\
\ppgk{4} & 0.831 & 0.925 & & \ppgk{4} & 0.692 & 0.715 \\
\ppgk{5} & \underline{0.958} & \underline{0.964} & & \ppgk{5} & 0.802 & 0.824 \\
\pg & 0.867 & 0.813 & & \pg & 0.727 & 0.771 \\                 
\bottomrule
\end{tabular}
}
\caption{Evaluation on chatting domain in Consistent-LLMs dataset. \textbf{Bold} / \underline{underline} is the best / second-best result.}
\label{chatting}
\end{table}

% Define colors

\definecolor{baseline}{RGB}{31, 119, 180}

\definecolor{totalreward}{RGB}{23, 190, 207}

\definecolor{perstep}{RGB}{255, 127, 14}

\definecolor{k2step}{RGB}{44, 160, 44}

\definecolor{k3step}{RGB}{214, 39, 40}

\definecolor{k4step}{RGB}{148, 103, 189}

\definecolor{k5step}{RGB}{140, 86, 75}

\begin{figure*}[t!]

\centering

% Top row: Qwen3-8B Model

\begin{minipage}{0.31\textwidth}\centering

\begin{tikzpicture}

\begin{axis}[width=\linewidth, height=0.6\linewidth, view={60}{30}, xlabel={Step}, ylabel={Policy}, zlabel={Scaled Residuals }, xmin=0, xmax=60, ymin=0, ymax=6, zmin=-1.0, zmax=1.0, ytick={0,1,2,3,4,5,6}, yticklabels={\base,\pg,K=1,K=2,K=3,K=4,K=5}, yticklabel style={font=\tiny}, xticklabel style={font=\tiny}, zticklabel style={font=\tiny}, label style={font=\scriptsize}, title style={font=\scriptsize}, title={Education}, grid=major, grid style={dashed, gray!30}]

\addplot3[color=baseline, mark=o, mark size=1pt] coordinates {(1,0,0.60) (5,0,0.58) (10,0,0.40) (15,0,0.25) (20,0,0.13) (25,0,0.00) (30,0,-0.05) (35,0,-0.18) (40,0,-0.25) (45,0,-0.30) (50,0,-0.48) (55,0,-0.68) (59,0,-0.93)};

\addplot3[color=totalreward, mark=x, mark size=1pt] coordinates {(1,1,0.05) (5,1,0.08) (10,1,0.05) (15,1,-0.13) (20,1,0.03) (25,1,0.00) (30,1,0.08) (35,1,0.00) (40,1,0.03) (45,1,0.00) (50,1,0.00) (55,1,0.00) (59,1,0.00)};

\addplot3[color=perstep, mark=square, mark size=1pt] coordinates {(1,2,0.65) (5,2,0.50) (10,2,-0.15) (15,2,0.55) (20,2,-0.10) (25,2,0.28) (30,2,-0.35) (35,2,-0.65) (40,2,-0.23) (45,2,-0.78) (50,2,-0.08) (55,2,0.35) (59,2,-0.90)};

\addplot3[color=k2step, mark=triangle, mark size=1pt] coordinates {(1,3,0.15) (5,3,0.13) (10,3,0.03) (15,3,0.08) (20,3,-0.03) (25,3,-0.10) (30,3,0.10) (35,3,0.10) (40,3,0.05) (45,3,-0.10) (50,3,-0.13) (55,3,0.03) (59,3,-0.10)};

\addplot3[color=k3step, mark=diamond, mark size=1pt] coordinates {(1,4,0.08) (5,4,0.13) (10,4,0.10) (15,4,0.05) (20,4,-0.13) (25,4,-0.05) (30,4,-0.10) (35,4,-0.23) (40,4,-0.15) (45,4,-0.10) (50,4,-0.05) (55,4,-0.23) (59,4,-0.10)};

\addplot3[color=k4step, mark=pentagon, mark size=1pt] coordinates {(1,5,0.03) (5,5,-0.03) (10,5,-0.03) (15,5,-0.03) (20,5,-0.03) (25,5,-0.05) (30,5,0.00) (35,5,-0.05) (40,5,0.03) (45,5,0.08) (50,5,-0.03) (55,5,0.08) (59,5,0.03)};

\addplot3[color=k5step, mark=star, mark size=1pt] coordinates {(1,6,0.10) (5,6,0.13) (10,6,0.10) (15,6,0.13) (20,6,0.00) (25,6,-0.18) (30,6,-0.05) (35,6,-0.18) (40,6,-0.13) (45,6,-0.18) (50,6,0.03) (55,6,-0.30) (59,6,-0.05)};

\end{axis}\end{tikzpicture}

\end{minipage}%
\hfill
\begin{minipage}{0.31\textwidth}\centering

\begin{tikzpicture}

\begin{axis}[width=\linewidth, height=0.6\linewidth, view={60}{30}, xlabel={Step}, ylabel={Policy}, zlabel={Scaled Residuals }, xmin=0, xmax=60, ymin=0, ymax=6, zmin=-1.0, zmax=1.0, ytick={0,1,2,3,4,5,6}, yticklabels={\base,\pg,K=1,K=2,K=3,K=4,K=5}, yticklabel style={font=\tiny}, xticklabel style={font=\tiny}, zticklabel style={font=\tiny}, label style={font=\scriptsize}, title style={font=\scriptsize}, title={Therapy}, grid=major, grid style={dashed, gray!30}]

\addplot3[color=baseline, mark=o, mark size=1pt] coordinates {(1,0,1.20) (5,0,0.48) (10,0,0.13) (15,0,0.05) (20,0,0.13) (25,0,-0.03) (30,0,-0.10) (35,0,-0.35) (40,0,-0.35) (45,0,-0.35) (50,0,-0.48) (55,0,-0.48) (59,0,-0.48)};

\addplot3[color=totalreward, mark=x, mark size=1pt] coordinates {(1,1,0.55) (5,1,0.13) (10,1,-0.10) (15,1,-0.63) (20,1,-0.50) (25,1,-0.63) (30,1,-0.08) (35,1,0.05) (40,1,0.30) (45,1,0.50) (50,1,0.20) (55,1,-0.75) (59,1,-0.50)};

\addplot3[color=perstep, mark=square, mark size=1pt] coordinates {(1,2,0.48) (5,2,0.20) (10,2,-0.60) (15,2,0.28) (20,2,-1.05) (25,2,0.20) (30,2,-0.80) (35,2,-1.13) (40,2,-0.45) (45,2,-0.55) (50,2,-0.25) (55,2,-1.05) (59,2,0.08)};

\addplot3[color=k2step, mark=triangle, mark size=1pt] coordinates {(1,3,0.45) (5,3,0.30) (10,3,0.18) (15,3,0.08) (20,3,-0.05) (25,3,0.00) (30,3,-0.20) (35,3,0.00) (40,3,0.05) (45,3,0.00) (50,3,-0.05) (55,3,-0.13) (59,3,-0.13)};

\addplot3[color=k3step, mark=diamond, mark size=1pt] coordinates {(1,4,0.38) (5,4,0.38) (10,4,0.35) (15,4,0.38) (20,4,0.23) (25,4,0.35) (30,4,0.18) (35,4,0.23) (40,4,0.30) (45,4,0.18) (50,4,0.13) (55,4,0.18) (59,4,0.30)};

\addplot3[color=k4step, mark=pentagon, mark size=1pt] coordinates {(1,5,0.03) (5,5,0.03) (10,5,0.05) (15,5,-0.03) (20,5,0.00) (25,5,0.13) (30,5,-0.08) (35,5,-0.08) (40,5,0.05) (45,5,0.05) (50,5,-0.13) (55,5,0.05) (59,5,-0.08)};

\addplot3[color=k5step, mark=star, mark size=1pt] coordinates {(1,6,0.33) (5,6,0.08) (10,6,0.03) (15,6,0.05) (20,6,0.08) (25,6,0.05) (30,6,0.13) (35,6,0.13) (40,6,0.18) (45,6,0.05) (50,6,0.15) (55,6,0.18) (59,6,0.20)};

\end{axis}\end{tikzpicture}

\end{minipage}%
\hfill
\begin{minipage}{0.31\textwidth}\centering

\begin{tikzpicture}

\begin{axis}[width=\linewidth, height=0.6\linewidth, view={60}{30}, xlabel={Step}, ylabel={Policy}, zlabel={Scaled Residuals }, xmin=0, xmax=60, ymin=0, ymax=6, zmin=-1.0, zmax=1.0, ytick={0,1,2,3,4,5,6}, yticklabels={\base,\pg,K=1,K=2,K=3,K=4,K=5}, yticklabel style={font=\tiny}, xticklabel style={font=\tiny}, zticklabel style={font=\tiny}, label style={font=\scriptsize}, title style={font=\scriptsize}, title={Chatting}, grid=major, grid style={dashed, gray!30}]

\addplot3[color=baseline, mark=o, mark size=1pt] coordinates {(1,0,1.15) (5,0,1.13) (10,0,0.88) (15,0,0.75) (20,0,0.40) (25,0,0.23) (30,0,-0.03) (35,0,-0.03) (40,0,-0.23) (45,0,-0.10) (50,0,-0.35) (55,0,-0.48) (59,0,-0.60)};

\addplot3[color=totalreward, mark=x, mark size=1pt] coordinates {(1,1,0.28) (5,1,0.23) (10,1,0.10) (15,1,0.15) (20,1,-0.25) (25,1,-0.13) (30,1,0.15) (35,1,0.03) (40,1,-0.03) (45,1,-0.30) (50,1,0.03) (55,1,0.08) (59,1,-0.05)};

\addplot3[color=perstep, mark=square, mark size=1pt] coordinates {(1,2,0.18) (5,2,-0.25) (10,2,-0.10) (15,2,-0.20) (20,2,0.10) (25,2,-0.33) (30,2,-0.08) (35,2,-0.08) (40,2,-0.03) (45,2,-0.13) (50,2,0.15) (55,2,0.00) (59,2,-0.13)};

\addplot3[color=k2step, mark=triangle, mark size=1pt] coordinates {(1,3,-0.58) (5,3,-0.38) (10,3,-0.08) (15,3,-0.08) (20,3,-0.15) (25,3,-0.18) (30,3,0.03) (35,3,0.20) (40,3,0.15) (45,3,0.10) (50,3,-0.03) (55,3,-0.05) (59,3,-0.10)};

\addplot3[color=k3step, mark=diamond, mark size=1pt] coordinates {(1,4,-0.10) (5,4,0.03) (10,4,0.05) (15,4,0.05) (20,4,0.05) (25,4,0.05) (30,4,0.00) (35,4,0.00) (40,4,0.00) (45,4,0.00) (50,4,-0.05) (55,4,-0.05) (59,4,-0.05)};

\addplot3[color=k4step, mark=pentagon, mark size=1pt] coordinates {(1,5,0.23) (5,5,0.25) (10,5,0.15) (15,5,0.15) (20,5,0.03) (25,5,0.03) (30,5,0.08) (35,5,0.13) (40,5,0.00) (45,5,0.08) (50,5,0.10) (55,5,0.05) (59,5,-0.05)};

\addplot3[color=k5step, mark=star, mark size=1pt] coordinates {(1,6,0.28) (5,6,0.25) (10,6,0.25) (15,6,0.25) (20,6,0.20) (25,6,0.08) (30,6,0.15) (35,6,0.08) (40,6,0.20) (45,6,0.20) (50,6,0.08) (55,6,0.08) (59,6,-0.05)};

\end{axis}\end{tikzpicture}

\end{minipage}
\caption{\small{Residual of persona consistency scaled to $[-1, 1]$ across trajectory steps for all domains with \texttt{Qwen}. $K = 1$ represents \ppgone and $K \in \{2, 3, 4, 5\}$ represents \ppgk{K}.}}
\label{fig:variance_qwen}
\end{figure*}
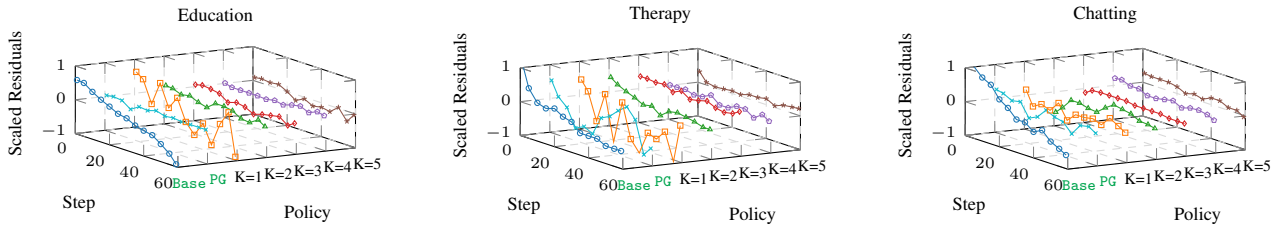

\pgfplotsset{compat=1.18}
\begin{figure}[t!]
    \centering
    % Define custom colors for consistency across all plots
    \definecolor{baseline}{RGB}{31, 119, 180}
    \definecolor{totalreward}{RGB}{23, 190, 207}
    \definecolor{perstep}{RGB}{255, 127, 14}
    \definecolor{k2step}{RGB}{44, 160, 44}
    \definecolor{k3step}{RGB}{214, 39, 40}
    \definecolor{k4step}{RGB}{148, 103, 189}
    \definecolor{k5step}{RGB}{140, 86, 75}

    % Common style for compact subplots
    \pgfplotsset{
        compactplot/.style={
            width=1.1\linewidth, 
            height=4.5cm,
            grid=major,
            label style={font=\tiny},
            tick label style={font=\tiny},
            title style={font=\scriptsize, yshift=-1ex},
            legend style={font=\tiny},
        }
    }

    % --- GROUP (a): Domain Efficiency (3 Charts) ---
    \begin{minipage}{0.5\textwidth}
        \centering
        \begin{minipage}{0.32\textwidth}
            \begin{tikzpicture}
                \begin{semilogxaxis}[
                    compactplot, title={Education}, ylabel={Persona Consistency}, 
                    xtick={1, 10, 100}, xlabel={Trajectories}, xticklabels={5e1, 5e2, 5e3}, 
                    legend to name=SharedLegend, legend columns=-1,
                    legend style={/tikz/every even column/.append style={column sep=0.2cm}}
                ]
                    \addplot[thick, color=baseline, dashed] coordinates {(0.8, 0.425) (130, 0.425)}; \addlegendentry{Base}
                    \addplot[thick, color=perstep, only marks, mark=square] coordinates {(1, 0.673) (10, 0.750) (100, 0.837)}; \addlegendentry{$K=1$}
                    \addplot[thick, color=k2step, only marks, mark=triangle] coordinates {(1, 0.367) (10, 0.676) (100, 0.721)}; \addlegendentry{$K=2$}
                    \addplot[thick, color=k3step, only marks, mark=diamond] coordinates {(1, 0.298) (10, 0.539) (100, 0.573)}; \addlegendentry{$K=3$}
                    \addplot[thick, color=k4step, only marks, mark=pentagon] coordinates {(1, 0.341) (10, 0.790) (100, 0.820)}; \addlegendentry{$K=4$}
                    \addplot[thick, color=k5step, only marks, mark=star] coordinates {(1, 0.362) (10, 0.722) (100, 0.775)}; \addlegendentry{$K=5$}
                    \addplot[thick, color=totalreward, only marks, mark=otimes] coordinates {(1, 0.248) (10, 0.650) (100, 0.913)}; \addlegendentry{PG}
                \end{semilogxaxis}
            \end{tikzpicture}
        \end{minipage}
        \begin{minipage}{0.32\textwidth}
            \begin{tikzpicture}
                \begin{semilogxaxis}[compactplot, title={Therapy}, xtick={1, 10, 100}, xlabel={Trajectories}, xticklabels={5e1, 5e2, 5e3}]
                    \addplot[thick, color=baseline, dashed] coordinates {(0.8, 0.465) (130, 0.465)};
                    \addplot[thick, color=perstep, only marks, mark=square] coordinates {(1, 0.6) (10, 0.653) (100, 0.685)};
                    \addplot[thick, color=k2step, only marks, mark=triangle] coordinates {(1, 0.468) (10, 0.687) (100, 0.712)};
                    \addplot[thick, color=k3step, only marks, mark=diamond] coordinates {(1, 0.291) (10, 0.576) (100, 0.78)};
                    \addplot[thick, color=k4step, only marks, mark=pentagon] coordinates {(1, 0.264) (10, 0.458) (100, 0.673)};
                    \addplot[thick, color=k5step, only marks, mark=star] coordinates {(1, 0.312) (10, 0.489) (100, 0.69)};
                    \addplot[thick, color=totalreward, only marks, mark=otimes] coordinates {(1, 0.355) (10, 0.499) (100, 0.586)};
                \end{semilogxaxis}
            \end{tikzpicture}
        \end{minipage}
        \hspace{-0.5cm}
        \begin{minipage}{0.32\textwidth}
            \begin{tikzpicture}
                \begin{semilogxaxis}[compactplot, title={Chatting}, xtick={1, 10, 100}, xlabel={Trajectories}, xticklabels={5e1, 5e2, 5e3}]
                    \addplot[thick, color=baseline, dashed] coordinates {(0.8, 0.478) (130, 0.478)};
                    \addplot[thick, color=perstep, only marks, mark=square] coordinates {(1, 0.679) (10, 0.695) (100, 0.799)};
                    \addplot[thick, color=k2step, only marks, mark=triangle] coordinates {(1, 0.624) (10, 0.832) (100, 0.929)};
                    \addplot[thick, color=k3step, only marks, mark=diamond] coordinates {(1, 0.618) (10, 0.699) (100, 0.821)};
                    \addplot[thick, color=k4step, only marks, mark=pentagon] coordinates {(1, 0.335) (10, 0.555) (100, 0.692)};
                    \addplot[thick, color=k5step, only marks, mark=star] coordinates {(1, 0.372) (10, 0.589) (100, 0.802)};
                    \addplot[thick, color=totalreward, only marks, mark=otimes] coordinates {(1, 0.379) (10, 0.509) (100, 0.727)};
                \end{semilogxaxis}
            \end{tikzpicture}
        \end{minipage}
        \label{fig:stat_efficiency_b}
        \ref{SharedLegend}
        % \subcaption{\label{fig:stat_efficiency_b}Statistical efficiency of \ppg across domains for Llama-3.1-8B-Instruct model.}
    \end{minipage}

    % \vspace{-0.4cm}
    % \centering\ref{SharedLegend} 
    \caption{\small{Persona consistency as a function of sample size (number of training trajectories) for all domains with \texttt{Llama}. $K = 1$ represents \ppgone and $K \in \{2, 3, 4, 5\}$ represents \ppgk{K}.}}
    \label{fig:stat_efficiency}
\end{figure}
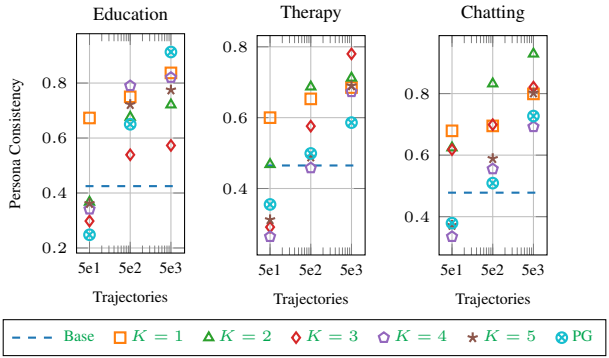

\subsection{Main Results}
\label{sec:performance_results}

We start by evaluating all policies in all domains: \emph{education}, \emph{therapy}, and \emph{chatting}; and report their average persona consistency in \cref{education,therapy,chatting}. While the best performing policy is domain-dependent, policy gradients universally outperform \base and PPO, demonstrating the robustness of offline RL by policy gradients. In chatting and therapy domains, $K$-step lookahead policies (\ppgk{2} and \ppgk{3}) perform extremely well across both \texttt{Qwen} and \texttt{Llama} models. \pg is the best policy in the education domain, because a faithful simulation of a tutored student requires long-term planning. We observe consistent results across rollout and step-level evaluation methodologies, suggesting that our gains are robust despite the cascading errors commonly observed in long dialogues. Our results show similar trends for both \texttt{Qwen} and \texttt{Llama} models, suggesting that our approach generalizes across multiple LLM architectures.

\subsection{On Effectiveness of Simpler Policies}
\label{sec:stability_results}

Now we try to understand the best performing policies. For all policies, we compute the mean persona consistency at each conversation step, subtract the mean over all steps, and renormalize to $[-1, 1]$ jointly over all policies. Note that this is just a linear transformation of consistency scores that makes them comparable across the policies. We call this quantity a \emph{residual} and report it in \cref{fig:variance_qwen}. The residual of \base monotonically decreases as conversation progresses, confirming that LLMs are susceptible to persona drift in long-horizon conversations~\citep{nie2021like}. The residual of \ppgone oscillates a lot. These ripples reflect the policy's repeated failed recovery attempts from persona drift, due to not planning for the future. In comparison, $K$-step lookahead policies (\ppgk{K}) maintain stable residuals across trajectory steps. We observe that increasing $K$ from $K = 1$ (\ppgone) to $K = 3$ (\ppgk{K}) produces progressively smoother residuals, supporting our theory that assigning credit to longer horizons can prevent persona drift by maintaining consistency continuously over time. Finally, in therapy and chatting domains, we observe that \pg oscillates a lot and thus performs worse than \ppgk{3}. This experiment shows that the temporal scope of credit assignment directly influences policy stability, with longer horizons providing smoother and more robust behavior at the cost of increased data requirements. Results for \texttt{Llama} are in \cref{app:variance_llama}. See qualitative analysis in \cref{app:qualitative_analysis}.

\subsection{On Statistical Efficiency of Simpler Policies}
\label{sec:stat_efficiency_results}

We investigate the statistical efficiency of learning different policies next. Specifically, we vary the number of training trajectories, from $50$ (low-data regime, $1\%$ of training data) to $5\,000$ (high-data regime, all training data), and plot the persona consistency of all policies in \cref{fig:stat_efficiency}. Our results reveal a fundamental trade-off between policy complexity and statistical efficiency. In the low-data regime, \ppgone consistently excels in all domains, confirming our hypothesis that simpler policies are easier to learn when data are scarce. As data scales to moderate levels ($500$ trajectories, $10\%$ of training data), $K$-step lookahead policies dominate across all domains. In particular, \ppgk{2} and \ppgk{3} are the best policies in therapy and chatting domains. These policies strike balance between planning enough for the future and being able to learn it. In the education domain, we observe that full planning with \pg eventually achieves the highest persona consistency in the high-data regime. We validate this scaling behavior across three model architectures (Llama, Qwen, Gemma) on an additional dataset (Synthetic Persona Chat ~\citep{jandaghi2024faithful}) in \cref{app:scaling_laws}, confirming that optimal lookahead horizon $K$ consistently increases with available training data, establishing a practical design principle for calibrating policy complexity to data budgets.

\section{Related Work}
\label{sec:related work}

The closest related works are \citet{kveton25adaptive} and \citet{guo25segment}. \citet{kveton25adaptive} optimized greedy policies using online policy gradients and \citet{guo25segment} optimized segment policies using GRPO. Our formalism captures both of these policy classes. Unlike \citet{kveton25adaptive}, we consider more general optimization problems than adaptive submodularity \citep{golovin11adaptive}. Unlike \citet{guo25segment}, we consider a general RL model that allows observations of the environment.

Learning of policies in finite-horizon Markov decision processes is a classic topic and it is well known that they are non-stationary \citep{puterman94markov}. Monte Carlo estimation of Q values, by the sum of infinite-horizon discounted rewards, is discussed in detail in \citet{sutton98reinforcement} and was generalized to advantages by \citet{schulman16highdimensional}. We do not discount rewards and show that maximization of a subset of future rewards results in a new simpler policy that is easier to learn. $K$-step lookahead policies were studied in other contexts before. \citet{efroni18beyond} used them to improve policy iteration, \citet{efroni18multistep} applied them to online and approximate RL, and \citet{efroni20online} extended them to real-time dynamic programming. Adaptive lookahead polices were studied by \citet{rosenberg23planning} and $K$-step lookahead policies were used in offline RL by \citet{sikchi22learning}. The main difference in our work is that we only use policy gradients without any additional model, such as that of a value or Q function.

\section{Conclusions}
\label{sec:conclusions}

Most popular RL algorithms in LLMs are variants of policy gradients. In this work,  we propose a natural approach for introducing policy structure in policy gradients. The key idea is to optimize for a subset of future rewards. Smaller subsets of rewards define simpler policies, which can be learned more reliably because their empirical gradients are more accurate. We interpret the policies, and propose both online and offline algorithms for learning them. Full planning, greedy, $K$-step lookahead, and segment policies are all instances of our framework. We evaluate the policies on conversation optimization problems and observe that different policies excel in different domains, highlighting the importance of our introduced structure. We note that the idea of optimizing for subsets of future rewards is general and can be straightforwardly applied to regularized policies \citep{todorov06linearlysolvable} and also GRPO \citep{shao24deepseekmath}.

\textbf{Limitations.} First, our work provides both theory and experiments that justify policy optimization for subsets of future rewards. However, we never formally show that the better performance in our experiments is due to a lower variance in policy gradient estimates. Second, the per-step reward in \eqref{eq:additive reward} is an important design choice that may have impact on empirical performance. We do not ablate it and instead rely on existing works to provide it (\cref{sec:experiments}). Finally, our evaluation is limited to conversational problems, although we experiment with four different problems and three LLMs.

\section*{Impact Statement}

This paper presents work whose goal is to advance the field of machine learning. There are many potential societal consequences of our work, none of which we feel must be specifically highlighted here.

\bibliographystyle{plainnat}
\bibliography{bib/brano,bib/main}

\clearpage
\onecolumn
\appendix

\section{Additional Technical Contributions}
\label{sec:additional technical contributions}

This section contains proofs of our main claims and supporting lemmas.

\subsection{Proof of \cref{lem:additive reward}}
\label{sec:additive reward proof}

The proof is based on telescoping,
\begin{align*}
  \sum_{t = 1}^n r_t
  = \sum_{t = 1}^n f(x, \tau_t) - f(x, \tau_{t - 1})
  = f(x, \tau_n) - f(x, \tau_0)
  = r(x, \tau_n)\,.
\end{align*}
The rewards $r_t$ are non-negative because $f$ is monotone.

\subsection{Gradient Concentration}
\label{sec:gradient concentration}

The main tool used in our analysis is Hoeffding's inequality \citep{raginsky12concentration,boucheron13concentration}, which we restate below for completeness.

\begin{theorem}
\label{thm:hoeffding} Let $X_1, \dots, X_n$ be $n$ independent random variables such that $X_i \in [- c_i, c_i]$ holds almost surely for some $c_i > 0$. Then for any $\varepsilon > 0$,
\begin{align*}
  \prob{|X_1 + \dots + X_n - \E{X_1 + \dots + X_n}| \geq \varepsilon}
  \leq 2 \exp\left[- \frac{\varepsilon^2}{2 \sum_{i = 1}^n c_i^2}\right]\,.
\end{align*}
\end{theorem}

We state our main theorem next.

\begin{theorem}
\label{thm:gradient concentration} Let $g(\cB)$ and $g'(\cB)$ be partial policy gradient estimators such that $\cS_t \subseteq \cS'_t$ holds for all $t \in [n]$, where $\cS_t$ and $\cS'_t$ are their action indices that affect the reward at step $t$. Then $g(\cB)$ concentrates in the $2$-norm faster than $g'(\cB)$.
\end{theorem}
\begin{proof}
We apply Hoeffding's inequality as follows. Let
\begin{align}
  f(\cB)
  = \sum_{i = 1}^B \sum_{t = 1}^n \sum_{\ell \in \cS_t}
  r_{i, t} \nabla \log \pi(a_{i, \ell} \mid x_i, \tau_{i, \ell - 1}; \theta)\,,
  \label{eq:f}
\end{align}
where $\cB$ is the batch in \cref{alg:meta pg} and $\cS_t$ is the set of action indices that affect the reward at step $t$ (\cref{sec:interpretation}). Note that $f(\cB)$ is a $d$-dimensional vector, where $d$ is the number of policy parameters $\theta$ with respect to which the gradient is taken. For any $j \in [d]$, let $f_j(\cB)$ be the $j$-th entry of $f(\cB)$.

Because of the identity in \eqref{eq:action-reward-sum swap}, we have $f(\cB) = B g(\cB)$. Using norm identities and the union bound, we get
\begin{align*}
  \prob{\normw{g(\cB) - \E{g(\cB)}}{2} \geq \varepsilon}
  & = \prob{\normw{f(\cB) - \E{f(\cB)}}{2} \geq \varepsilon B} \\
  & \leq \prob{\sqrt{d} \max_{j \in [d]} |f_j(\cB) - \E{f_j(\cB)}|
  \geq \varepsilon B} \\
  & \leq \sum_{j = 1}^d \prob{|f_j(\cB) - \E{f_j(\cB)}|
  \geq \frac{\varepsilon B}{\sqrt{d}}}\,.
\end{align*}
Now note that $f_j(\cB)$ is a sum of $B \sum_{t = 1}^n |\cS_t|$ random scalars but $\cS_t$ is fixed. Let $|r_{i, t} (\nabla \log \pi(a_{i, \ell} \mid x_i, \tau_{i, \ell - 1}; \theta))_j| \leq c_{t, \ell}$ hold almost surely and jointly for all $i \in [B]$ and $j \in [d]$. Then, since the vectors in \eqref{eq:f} can be arbitrarily correlated,
\begin{align*}
  \abs{\sum_{t = 1}^n \sum_{\ell \in \cS_t}
  r_{i, t} (\nabla \log \pi(a_{i, \ell} \mid x_i, \tau_{i, \ell - 1}; \theta))_j}
  \leq \sum_{t = 1}^n \sum_{\ell \in \cS_t} c_{t, \ell}
\end{align*}
holds almost surely and jointly for all $i \in [B]$ and $j \in [d]$. So by \cref{thm:hoeffding}, for any $j \in [d]$ and $\varepsilon > 0$,
\begin{align*}
  \prob{|f_j(\cB) - \E{f_j(\cB)}| \geq \frac{\varepsilon B}{\sqrt{d}}}
  & \leq 2 \exp\left[- \frac{\varepsilon^2 B^2}
  {2 d B \left(\sum_{t = 1}^n \sum_{\ell \in \cS_t} c_{t, \ell}\right)^2}\right] \\
  & \leq 2 \exp\left[- \frac{\varepsilon^2 B}
  {2 d \left(\sum_{t = 1}^n \sum_{\ell \in \cS_t} c_{t, \ell}\right)^2}\right]\,.
\end{align*}
Finally, we chain all inequalities and get
\begin{align*}
  \prob{\normw{g(\cB) - \E{g(\cB)}}{2} \geq \varepsilon}
  \leq 2 d \exp\left[- \frac{\varepsilon^2 B}
  {2 d \left(\sum_{t = 1}^n \sum_{\ell \in \cS_t} c_{t, \ell}\right)^2}\right]\,.
\end{align*}
The bound shows that the concentration rate decreases as $\sum_{t = 1}^n \sum_{\ell \in \cS_t} c_{t, \ell}$ increases. Therefore, for any two partial policy gradient estimators $g(\cB)$ and $g'(\cB)$ such that $\cS_t \subseteq \cS'_t$ for all $t \in [n]$, the former concentrates faster than the latter in the worst case.
\end{proof}

\subsection{Proof of \cref{lem:ppg}}
\label{sec:ppg proof}

Using rules of probability,
\begin{align*}
  \E{\sum_{t = 1}^n \left(\sum_{\ell \in \cR_t} r_\ell\right)
  \nabla \log \pi(a_t \mid x, \tau_{t - 1}; \theta)}
  & = \E{\sum_{t = 1}^n r_t \sum_{\ell \in \cS_t}
  \nabla \log \pi(a_\ell \mid x, \tau_{\ell - 1}; \theta)} \\
  & = \sum_{t = 1}^n \E{r_t \sum_{\ell \in \cS_t}
  \nabla \log \pi(a_\ell \mid x, \tau_{\ell - 1}; \theta)} \\
  & = \sum_{t = 1}^n \Erv{\bar{\cS}_t}
  {\condErv{\cS_t}{r_t \sum_{\ell \in \cS_t}
  \nabla \log \pi(a_\ell \mid x, \tau_{\ell - 1}; \theta)}{\bar{\cS}_t}} \\
  & = \sum_{t = 1}^n \Erv{\bar{\cS}_t}
  {\nabla \condErv{\cS_t}{r_t}{\bar{\cS}_t}} \\
  & = \E{\sum_{t = 1}^n \nabla \condErv{\cS_t}{r_t}{\bar{\cS}_t}}\,.
\end{align*}
We apply \eqref{eq:action-reward-sum swap} in the first step, the linearity of expectation in the second step, the tower rule in the third step, the score identity over random $\tau_{\cS_t}$ in the fourth step, and the tower rule in the last step. This completes the proof.

\subsection{Proof of \cref{lem:offline ppg}}
\label{sec:offline ppg proof}

Using the tower rule and that $\rho$ does not depend on the differentiated $\theta$, we get
\begin{align*}
  \Erho{}{\sum_{t = 1}^n \left(\sum_{\ell \in \cR_t} r_\ell\right)
  \nabla \log \pi(a_t \mid x, \tau_{t - 1}; \theta)}
  & = \Erho{}{\sum_{t = 1}^n r_t \sum_{\ell \in \cS_t}
  \nabla \log \pi(a_\ell \mid x, \tau_{\ell - 1}; \theta)} \\
  & = \sum_{t = 1}^n \Erho{}{r_t \sum_{\ell \in \cS_t}
  \nabla \log \pi(a_\ell \mid x, \tau_{\ell - 1}; \theta)} \\
  & = \sum_{t = 1}^n \Erho{\bar{\cS}_t}
  {\Erho{\cS_t}{r_t \sum_{\ell \in \cS_t}
  \nabla \log \pi(a_\ell \mid x, \tau_{\ell - 1}; \theta) \,\middle|\, \bar{\cS}_t}} \\
  & = \sum_{t = 1}^n \Erho{\bar{\cS}_t}
  {\nabla \Erho{\cS_t}{r_t \sum_{\ell \in \cS_t}
  \log \pi(a_\ell, y_\ell \mid x, \tau_{\ell - 1}; \theta) \,\middle|\, \bar{\cS}_t}} \\
  & = \Erho{}{\sum_{t = 1}^n \nabla \Erho{\cS_t}{r_t \sum_{\ell \in \cS_t}
  \log \pi(a_\ell, y_\ell \mid x, \tau_{\ell - 1}; \theta) \,\middle|\, \bar{\cS}_t}}\,.
\end{align*}
The inequality
\begin{align*}
  \Erho{\cS_t}{r_t \sum_{\ell \in S_t}
  \log \pi(a_t, y_t \mid x, \tau_{t - 1}; \theta) \,\middle|\, \bar{\cS}_t}
  \leq \condErv{\cS_t}{r_t}{\bar{\cS}_t}
\end{align*}
follows from Lemma 1 of \citet{mukherjee25offline}. The only difference is that the expectation is taken over steps $\cS_t$ of the trajectory conditioned on $\bar{\cS}_t = [t] \setminus \cS_t$, which does not change the proof.

Similarly, the bound
\begin{align*}
  \left|\Erho{\cS_t}{r_t \sum_{\ell \in S_t}
  \log \pi(a_t, y_t \mid x, \tau_{t - 1}; \theta) \,\middle|\, \bar{\cS}_t} -
  \condErv{\cS_t}{r_t}{\bar{\cS}_t}\right|
  \leq C(\pi, \rho)\,,
\end{align*}
follows from Lemma 3 of \citet{mukherjee25offline}, where $C(\pi, \rho)$ is a constant depending on the propensity ratios of $\pi$ and $\rho$, and their logarithms. The only difference is that the expectation is taken over steps $\cS_t$ of the trajectory conditioned on $\bar{\cS}_t = [t] \setminus \cS_t$, which does not change the proof.

\subsection{$K$-Segment Policy Gradient}
\label{sec:segment pg}

Let $\cC_1, \dots, \cC_K \subset [n]$ be sets of consecutive step indices representing segments, which are mutually exclusive and exhaustive. Let $m_k = \min \cC_k$ and $M_k = \max \cC_k$ be the minimum and maximum indices within segment $C_k$. Let $k_t$ be the segment of step $t$. Then $\cS_t = [t] \setminus [m_{k_t} - 1]$ are the future rewards in segment $C_{k_t}$ starting from step $t$. By \cref{lem:ppg}, where $\cS_t = [t] \setminus [m_{k_t} - 1]$, the gradient is
\begin{align}
  & \E{\sum_{k = 1}^K \sum_{t \in \cC_k}
  \nabla \condErv{[t] \setminus [m_{k_t} - 1]}{r_t}{[m_{k_t} - 1]}}
  \label{eq:segment pg} \\
  & = \E{\sum_{k = 1}^K \nabla \condE{\sum_{t \in \cC_k} r_t}{[m_{k_t} - 1]}}\,.
  \nonumber
\end{align}
In plain English, all actions up to step $t$ within its segment $C_{k_t}$ affect the reward at step $t$; and the gradient is a sum of gradients of the corresponding expected rewards. We refer to \ppg with this $\cR_t$ as \ppgseg. This policy class was proposed in LLMs by \citet{guo25segment}.

This policy is simpler than that in \cref{sec:full pg}, where the reward is affected by each previous action, but also more complex than that in \cref{sec:greedy pg}, where the reward is affected only by the immediate action. As in \cref{sec:lookahead pg}, the policy naturally interpolates between these two extremes. Since
\begin{align*}
  \set{t}
  \subseteq \cS_t
  = [t] \setminus [m_{k_t} - 1]
  \subseteq [t]
\end{align*}
holds for any segments $\cC_1, \dots, \cC_K$, we have by \cref{thm:gradient concentration} that the empirical gradient in \ppgseg concentrates faster than that in \pg but slower than that in \ppgone. The rate depends on the size of the segments.

\section{Scaling Laws for Partial Policy Gradients}
\label{app:scaling_laws}

\pgfplotsset{compat=1.18}
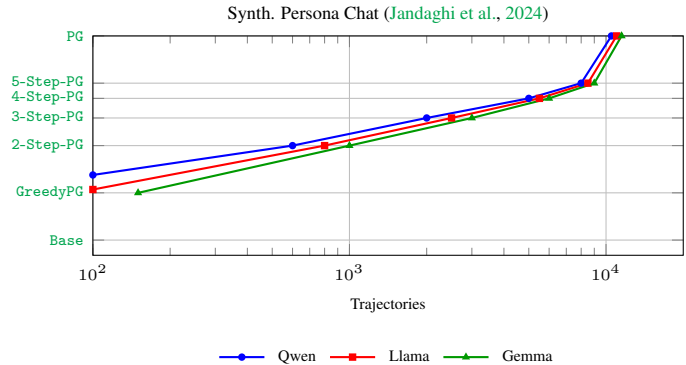
\begin{figure*}[htbp]
    \centering
    % Define custom colors for consistency across all plots
    \definecolor{baseline}{RGB}{31, 119, 180}
    \definecolor{totalreward}{RGB}{23, 190, 207}
    \definecolor{perstep}{RGB}{255, 127, 14}
    \definecolor{k2step}{RGB}{44, 160, 44}
    \definecolor{k3step}{RGB}{214, 39, 40}
    \definecolor{k4step}{RGB}{148, 103, 189}
    \definecolor{k5step}{RGB}{140, 86, 75}

    % Common style for compact subplots
    \pgfplotsset{
        compactplot/.style={
            width=1.1\linewidth, 
            height=4.5cm,
            grid=major,
            label style={font=\tiny},
            tick label style={font=\tiny},
            title style={font=\scriptsize, yshift=-1ex},
            legend style={font=\tiny},
        }
    }
    % --- GROUP (b): Scaling (1 Chart) ---
    \begin{minipage}{0.5\textwidth}
        \centering
        \begin{tikzpicture}
            \begin{axis}[
                compactplot,
                title={Synth. Persona Chat \citep{jandaghi2024faithful}},
                xmode=log, ymode=log,
                xlabel={Trajectories},
                ytick={0.5, 1, 2, 3, 4, 5, 10},
                yticklabels={\base, \ppgone, \ppgk{2}, \ppgk{3}, \ppgk{4}, \ppgk{5}, \pg},
                xmin=100, xmax=20000, ymin=0.4, ymax=10,
                legend columns=-1,
                legend style={at={(0.5,-0.38)}, anchor=north, draw=none, column sep=0.1cm, font=\tiny}
            ]
            % Inflection point focused between 5000 and 6000
            \addplot[mark=*, blue, mark size=1pt, thick] coordinates {
                (50, 1) (100, 1.3) (600, 2) (2000, 3) (5000, 4) (8000, 5) (10500, 10)
            };
            \addlegendentry{Qwen}
            \addplot[mark=square*, red, mark size=1pt, thick] coordinates {
                (80, 1) (100, 1.05) (800, 2) (2500, 3) (5500, 4) (8500, 5) (11000, 10)
            };
            \addlegendentry{Llama}
            \addplot[mark=triangle*, green!60!black, mark size=1pt, thick] coordinates {
                (150, 1) (1000, 2) (3000, 3) (6000, 4) (9000, 5) (11500, 10)
            };
            \addlegendentry{Gemma}
            \end{axis}
        \end{tikzpicture}
        % \subcaption{\label{fig:stat_efficiency_b} Scaling Laws for \ppg}
    \end{minipage}

    \caption{\small{Scaling Laws for \ppg: Across different LLMs (Llama-3.1-8B-Instruct, Qwen3-8B, and Gemma-7B-it), optimal value of lookahead $K$ in \ppgk{K} scales with available training trajectories.}}
    \label{fig:stat_efficiency2}
\end{figure*}

\noindent\textbf{Motivation.} In \cref{fig:stat_efficiency}, we demonstrated that simpler policies (\ppgone) exhibit superior statistical efficiency in low-data regimes compared to complex policies (\pg), with the performance gap narrowing as training data increases. This raises a practical question: \emph{how does the optimal lookahead horizon K scale with available training data across different model architectures?} Understanding this scaling relationship is critical for practitioners who must choose appropriate credit assignment horizons given their data budget constraints.

\noindent\textbf{Experimental Design.} To investigate this scaling behavior systematically, we replicate the statistical efficiency experiment from the main paper using a complementary dataset: the Synthetic Persona Chat dataset \citep{jandaghi2024faithful}, which contains generic and mixed-domain personas distinct from our primary Education, Therapy, and Chatting domains. We sample 20,000 trajectories from this dataset and train all policy variants (\base, \ppgone, \ppgk{2}, \ppgk{3}, \ppgk{4}, \ppgk{5}, \pg) with varying amounts of training data, ranging from 100 to 20,000 trajectories. For each data size, we measure which policy achieves the best persona consistency on a held-out test set. We repeat this procedure across three different LLM architectures—Llama-3.1-8B-Instruct, Qwen-2.5-7B, and Gemma-7B-it—to assess whether the scaling laws are architecture-dependent or represent a general phenomenon.

\noindent\textbf{Observation.} \cref{fig:stat_efficiency2} shows the optimal policy (represented by its lookahead horizon K on the y-axis) as a function of available training trajectories (x-axis) for each model architecture. The key pattern is a consistent ``staircase'' structure: at low data regimes (50-1000 trajectories), \base~or \ppgone~are optimal; as data increases to intermediate regimes (1000-5000 trajectories), \ppgk{K} with K=2,3,4 successively become optimal; and at high data regimes ($>$10,000 trajectories), full planning (\pg) emerges as the best policy. Remarkably, this pattern holds across all three architectures, with only minor differences in the inflection points where transitions occur.

\noindent\textbf{Interpretation.} The results in \cref{fig:stat_efficiency2} demonstrate that the scaling behavior of $K$-step lookahead policies is remarkably consistent across different model architectures. We observe a characteristic trend where the optimal complexity of the policy, represented by the lookahead $K$, is a direct function of the available data volume. In low-data regimes, complex policies like full planning do not work effectively, and simpler objectives like \ppgone or low $K$ are preferred to achieve stable performance. However, as the number of trajectories increases beyond a critical threshold, higher-order policies become more effective, and full planning \pg eventually emerges as the best policy. This cross-architecture consistency validates our theoretical framework (\cref{thm:gradient concentration}): while simpler partial gradients concentrate faster, they optimize suboptimal policies, whereas complex policies require sufficient data to overcome higher variance in gradient estimation. These results establish a practical design principle for policy optimization in RL for LLMs: practitioners should calibrate policy complexity to available data, initializing with greedy policies in data-scarce regimes and progressively expanding the credit assignment horizon as data accumulates. 

% \cref{fig:stat_efficiency_b} highlights the scaling behavior of $K$-step lookahead policies across different model architectures, including \texttt{Qwen} , \texttt{Llama}, and Gemma-7B-Instruct models. We observe a characteristic trend where the optimal complexity of the policy, represented by the lookahead $K$, is a direct function of the available data volume. In low-data regimes, complex policies like full planning do not work effectively, and simpler objectives like \ppgone or low $K$ are preferred to achieve stable performance. However, as the number of trajectories increases beyond a critical threshold, higher-order policies become more effective, and full planning \pg eventually emerges as the best policy. This cross-architecture consistency validates our theoretical framework (\cref{thm:gradient concentration}): while simpler partial gradients concentrate faster, they optimize suboptimal policies, whereas complex policies require sufficient data to overcome higher variance in gradient estimation. These results establish a practical design principle for policy optimization in RL for LLMs: practitioners should calibrate policy complexity to available data, initializing with greedy policies in data-scarce regimes and progressively expanding the credit assignment horizon as data accumulates. 

\section{Performance comparison of policy gradients across trajectory lengths}
\label{app:results_conversation_length}

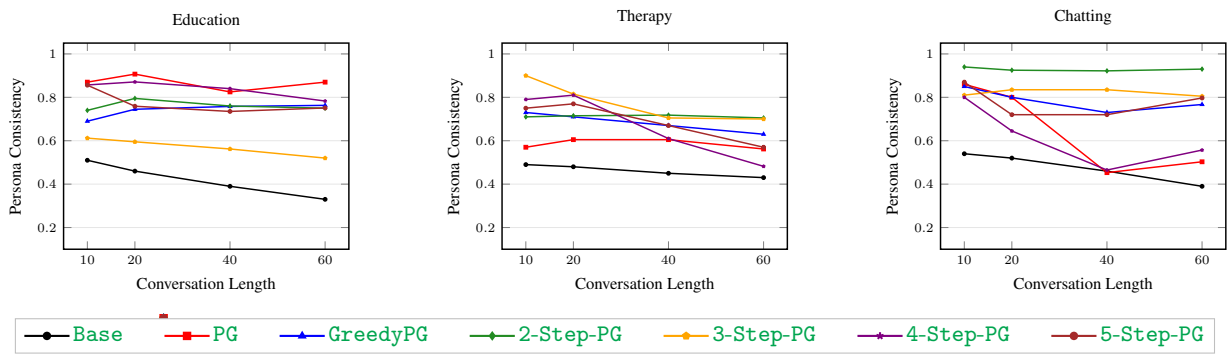
\begin{figure*}[htbp]
\centering

% === FIGURE 1: LLAMA 3.1 (ymin set to 0.1 to accommodate low data points) ===
\begin{minipage}[b]{0.32\textwidth}
\centering
\begin{tikzpicture}[scale=0.7]
\begin{axis}[
    title={Education },
    xlabel={Conversation Length},
    ylabel={Persona Consistency},
    xmin=5, xmax=65,
    ymin=0.1, ymax=1.05,
    xtick={10,20,40,60},
    ytick={0.2,0.4,0.6,0.8,1.0},
    ymajorgrids=true,
    xmajorgrids=false,
    grid style={gray!20},
    width=7cm,
    height=5.5cm,
    tick label style={font=\scriptsize},
    label style={font=\footnotesize},
    title style={font=\small},
    line width=0.8pt,
    mark size=1.5pt,
    cycle list={},
]
\addplot[model1, solid, mark=*, mark options={scale=0.6}] coordinates {(10,0.51) (20,0.46) (40,0.39) (60,0.33)};
\addplot[model2, solid, mark=square*, mark options={scale=0.6}] coordinates {(10,0.87) (20,0.907) (40,0.825) (60,0.87)};
\addplot[model3, solid, mark=triangle*, mark options={scale=0.6}] coordinates {(10,0.69) (20,0.745) (40,0.758) (60,0.763)};
\addplot[model4, solid, mark=diamond*, mark options={scale=0.6}] coordinates {(10,0.74) (20,0.795) (40,0.76) (60,0.75)};
\addplot[model5, solid, mark=pentagon*, mark options={scale=0.6}] coordinates {(10,0.612) (20,0.595) (40,0.562) (60,0.52)};
\addplot[model6, solid, mark=star, mark options={scale=0.7}] coordinates {(10,0.857) (20,0.871) (40,0.840) (60,0.783)};
\addplot[model7, solid, mark=otimes*, mark options={scale=0.6}] coordinates {(10,0.856) (20,0.759) (40,0.735) (60,0.751)};
\end{axis}
\end{tikzpicture}
\end{minipage}
\hfill
\begin{minipage}[b]{0.32\textwidth}
\centering
\begin{tikzpicture}[scale=0.7]
\begin{axis}[
    title={Therapy },
    xlabel={Conversation Length},
    ylabel={Persona Consistency},
    xmin=5, xmax=65,
    ymin=0.1, ymax=1.05,
    xtick={10,20,40,60},
    ytick={0.2,0.4,0.6,0.8,1.0},
    ymajorgrids=true,
    xmajorgrids=false,
    grid style={gray!20},
    width=7cm,
    height=5.5cm,
    tick label style={font=\scriptsize},
    label style={font=\footnotesize},
    title style={font=\small},
    line width=0.8pt,
    mark size=1.5pt,
    cycle list={},
]
\addplot[model1, solid, mark=*, mark options={scale=0.6}] coordinates {(10,0.49) (20,0.48) (40,0.45) (60,0.43)};
\addplot[model2, solid, mark=square*, mark options={scale=0.6}] coordinates {(10,0.57) (20,0.605) (40,0.605) (60,0.562)};
\addplot[model3, solid, mark=triangle*, mark options={scale=0.6}] coordinates {(10,0.73) (20,0.71) (40,0.67) (60,0.63)};
\addplot[model4, solid, mark=diamond*, mark options={scale=0.6}] coordinates {(10,0.71) (20,0.715) (40,0.718) (60,0.705)};
\addplot[model5, solid, mark=pentagon*, mark options={scale=0.6}] coordinates {(10,0.9) (20,0.815) (40,0.705) (60,0.7)};
\addplot[model6, solid, mark=star, mark options={scale=0.7}] coordinates {(10,0.79) (20,0.81) (40,0.61) (60,0.482)};
\addplot[model7, solid, mark=otimes*, mark options={scale=0.6}] coordinates {(10,0.75) (20,0.77) (40,0.67) (60,0.57)};
\end{axis}
\end{tikzpicture}
\end{minipage}
\hfill
\begin{minipage}[b]{0.32\textwidth}
\centering
\begin{tikzpicture}[scale=0.7]
\begin{axis}[
    title={Chatting},
    xlabel={Conversation Length},
    ylabel={Persona Consistency},
    xmin=5, xmax=65,
    ymin=0.1, ymax=1.05,
    xtick={10,20,40,60},
    ytick={0.2,0.4,0.6,0.8,1.0},
    ymajorgrids=true,
    xmajorgrids=false,
    grid style={gray!20},
    width=7cm,
    height=5.5cm,
    tick label style={font=\scriptsize},
    label style={font=\footnotesize},
    title style={font=\small},
    line width=0.8pt,
    mark size=1.5pt,
    cycle list={},
]
\addplot[model1, solid, mark=*, mark options={scale=0.6}] coordinates {(10,0.54) (20,0.52) (40,0.46) (60,0.39)};
\addplot[model2, solid, mark=square*, mark options={scale=0.6}] coordinates {(10,0.86) (20,0.80) (40,0.453) (60,0.503)};
\addplot[model3, solid, mark=triangle*, mark options={scale=0.6}] coordinates {(10,0.85) (20,0.8) (40,0.73) (60,0.767)};
\addplot[model4, solid, mark=diamond*, mark options={scale=0.6}] coordinates {(10,0.94) (20,0.925) (40,0.922) (60,0.93)};
\addplot[model5, solid, mark=pentagon*, mark options={scale=0.6}] coordinates {(10,0.81) (20,0.835) (40,0.835) (60,0.805)};
\addplot[model6, solid, mark=star, mark options={scale=0.7}] coordinates {(10,0.8) (20,0.645) (40,0.465) (60,0.557)};
\addplot[model7, solid, mark=otimes*, mark options={scale=0.6}] coordinates {(10,0.87) (20,0.72) (40,0.72) (60,0.797)};
\end{axis}
\end{tikzpicture}
\end{minipage}

\begin{center}
\begin{tikzpicture}
\begin{axis}[
    hide axis,
    scale only axis,          % 1. Add this to ignore label/margin math
    height=5pt,
    xmin=0, xmax=1,
    ymin=0, ymax=1,
    width=12cm,
    legend columns=7,
    legend style={
        at={(0.5,0)},
        anchor=north,
        draw=gray!50,
        font=\footnotesize,
        /tikz/every even column/.append style={column sep=0.5cm},
        legend cell align=left,
        nodes={inner sep=1pt}
    },
]
\addplot[model1, solid, thick, mark=*, mark options={scale=0.6}] coordinates {(0,0)};
\addplot[model2, solid, thick, mark=square*, mark options={scale=0.6}] coordinates {(0,0)};
\addplot[model3, solid, thick, mark=triangle*, mark options={scale=0.6}] coordinates {(0,0)};
\addplot[model4, solid, thick, mark=diamond*, mark options={scale=0.6}] coordinates {(0,0)};
\addplot[model5, solid, thick, mark=pentagon*, mark options={scale=0.6}] coordinates {(0,0)};
\addplot[model6, solid, thick, mark=star, mark options={scale=0.7}] coordinates {(0,0)};
\addplot[model7, solid, thick, mark=otimes*, mark options={scale=0.6}] coordinates {(0,0)};

\legend{\base, \pg, \ppgone, \ppgk{2}, \ppgk{3}, \ppgk{4}, \ppgk{5}}
\end{axis}
\end{tikzpicture}
\end{center}

\caption{Persona consistency of policy gradient methods vs trajectory length in education, therapy, and chatting domains for Llama3.1-8B-Instruct model. For each method, we report mean PC of all trajectories with $t$ steps, where $t=\{10, 20, 40, 60\}$.}
\label{fig:performance_vs_length_6plots}
\end{figure*}

% === FIGURE 2: QWEN (ymin set to 0.3 as requested) ===
\begin{figure*}[htbp]
\centering

\begin{minipage}[b]{0.32\textwidth}
\centering
\begin{tikzpicture}[scale=0.7]
\begin{axis}[
    title={Education },
    xlabel={Conversation Length},
    ylabel={Persona Consistency},
    xmin=5, xmax=65,
    ymin=0.3, ymax=1.05,
    xtick={10,20,40,60},
    ytick={0.3,0.4,0.6,0.8,1.0},
    ymajorgrids=true,
    xmajorgrids=false,
    grid style={gray!20},
    width=7cm,
    height=5.5cm,
    tick label style={font=\scriptsize},
    label style={font=\footnotesize},
    title style={font=\small},
    line width=0.8pt,
    mark size=1.5pt,
    cycle list={},
]
\addplot[model1, solid, mark=*, mark options={scale=0.6}] coordinates {(10,0.86) (20,0.845) (40,0.713) (60,0.78)};
\addplot[model2, solid, mark=square*, mark options={scale=0.6}] coordinates {(10,0.98) (20,0.98) (40,0.975) (60,0.97)};
\addplot[model3, solid, mark=triangle*, mark options={scale=0.6}] coordinates {(10,0.64) (20,0.65) (40,0.573) (60,0.63)};
\addplot[model4, solid, mark=diamond*, mark options={scale=0.6}] coordinates {(10,0.98) (20,0.96) (40,0.975) (60,0.93)};
\addplot[model5, solid, mark=pentagon*, mark options={scale=0.6}] coordinates {(10,0.97) (20,0.98) (40,0.927) (60,0.913)};
\addplot[model6, solid, mark=star, mark options={scale=0.7}] coordinates {(10,0.97) (20,0.915) (40,0.895) (60,0.942)};
\addplot[model7, solid, mark=otimes*, mark options={scale=0.6}] coordinates {(10,0.93) (20,0.89) (40,0.843) (60,0.863)};
\end{axis}
\end{tikzpicture}
\end{minipage}
\hfill
\begin{minipage}[b]{0.32\textwidth}
\centering
\begin{tikzpicture}[scale=0.7]
\begin{axis}[
    title={Therapy },
    xlabel={Conversation Length},
    ylabel={Persona Consistency},
    xmin=5, xmax=65,
    ymin=0.3, ymax=1.05,
    xtick={10,20,40,60},
    ytick={0.3,0.4,0.6,0.8,1.0},
    ymajorgrids=true,
    xmajorgrids=false,
    grid style={gray!20},
    width=7cm,
    height=5.5cm,
    tick label style={font=\scriptsize},
    label style={font=\footnotesize},
    title style={font=\small},
    line width=0.8pt,
    mark size=1.5pt,
    cycle list={},
]
\addplot[model1, solid, mark=*, mark options={scale=0.6}] coordinates {(10,0.63) (20,0.575) (40,0.555) (60,0.485)};
\addplot[model2, solid, mark=square*, mark options={scale=0.6}] coordinates {(10,0.79) (20,0.79) (40,0.79) (60,0.77)};
\addplot[model3, solid, mark=triangle*, mark options={scale=0.6}] coordinates {(10,0.9) (20,0.86) (40,0.848) (60,0.752)};
\addplot[model4, solid, mark=diamond*, mark options={scale=0.6}] coordinates {(10,0.91) (20,0.885) (40,0.805) (60,0.816)};
\addplot[model5, solid, mark=pentagon*, mark options={scale=0.6}] coordinates {(10,0.94) (20,0.96) (40,0.922) (60,0.927)};
\addplot[model6, solid, mark=star, mark options={scale=0.7}] coordinates {(10,0.76) (20,0.905) (40,0.785) (60,0.89)};
\addplot[model7, solid, mark=otimes*, mark options={scale=0.6}] coordinates {(10,0.87) (20,0.925) (40,0.802) (60,0.88)};
\end{axis}
\end{tikzpicture}
\end{minipage}
\hfill
\begin{minipage}[b]{0.32\textwidth}
\centering
\begin{tikzpicture}[scale=0.7]
\begin{axis}[
    title={Chatting },
    xlabel={Conversation Length},
    ylabel={Persona Consistency},
    xmin=5, xmax=65,
    ymin=0.3, ymax=1.05,
    xtick={10,20,40,60},
    ytick={0.3,0.4,0.6,0.8,1.0},
    ymajorgrids=true,
    xmajorgrids=false,
    grid style={gray!20},
    width=7cm,
    height=5.5cm,
    tick label style={font=\scriptsize},
    label style={font=\footnotesize},
    title style={font=\small},
    line width=0.8pt,
    mark size=1.5pt,
    cycle list={},
]
\addplot[model1, solid, mark=*, mark options={scale=0.6}] coordinates {(10,0.79) (20,0.77) (40,0.75) (60,0.73)};
\addplot[model2, solid, mark=square*, mark options={scale=0.6}] coordinates {(10,0.93) (20,0.77) (40,0.912) (60,0.855)};
\addplot[model3, solid, mark=triangle*, mark options={scale=0.6}] coordinates {(10,0.78) (20,0.69) (40,0.785) (60,0.792)};
\addplot[model4, solid, mark=diamond*, mark options={scale=0.6}] coordinates {(10,0.69) (20,0.815) (40,0.835) (60,0.788)};
\addplot[model5, solid, mark=pentagon*, mark options={scale=0.6}] coordinates {(10,0.97) (20,0.995) (40,1.0) (60,0.973)};
\addplot[model6, solid, mark=star, mark options={scale=0.7}] coordinates {(10,0.79) (20,0.885) (40,0.848) (60,0.802)};
\addplot[model7, solid, mark=otimes*, mark options={scale=0.6}] coordinates {(10,0.97) (20,0.965) (40,0.958) (60,0.938)};
\end{axis}
\end{tikzpicture}
\end{minipage}

\begin{center}
\begin{tikzpicture}
\begin{axis}[
    hide axis,
    scale only axis,          % 1. Add this to ignore label/margin math
    height=5pt,
    xmin=0, xmax=1,
    ymin=0, ymax=1,
    width=12cm,
    legend columns=7,
    legend style={
        at={(0.5,0)},
        anchor=north,
        draw=gray!50,
        font=\footnotesize,
        /tikz/every even column/.append style={column sep=0.5cm},
        legend cell align=left,
        nodes={inner sep=1pt}
    },
]
\addplot[model1, solid, thick, mark=*, mark options={scale=0.6}] coordinates {(0,0)};
\addplot[model2, solid, thick, mark=square*, mark options={scale=0.6}] coordinates {(0,0)};
\addplot[model3, solid, thick, mark=triangle*, mark options={scale=0.6}] coordinates {(0,0)};
\addplot[model4, solid, thick, mark=diamond*, mark options={scale=0.6}] coordinates {(0,0)};
\addplot[model5, solid, thick, mark=pentagon*, mark options={scale=0.6}] coordinates {(0,0)};
\addplot[model6, solid, thick, mark=star, mark options={scale=0.7}] coordinates {(0,0)};
\addplot[model7, solid, thick, mark=otimes*, mark options={scale=0.6}] coordinates {(0,0)};

\legend{\base, \pg, \ppgone, \ppgk{2}, \ppgk{3}, \ppgk{4}, \ppgk{5}}
\end{axis}
\end{tikzpicture}
\end{center}

\caption{Persona consistency of policy gradient methods vs trajectory length in education, therapy, and chatting domains for Qwen3-8B model. For each method, we report mean PC of all trajectories with $t$ steps, where $t=\{10, 20, 40, 60\}$.}
\label{fig:performance_vs_length_6plots_qwen}
\end{figure*}
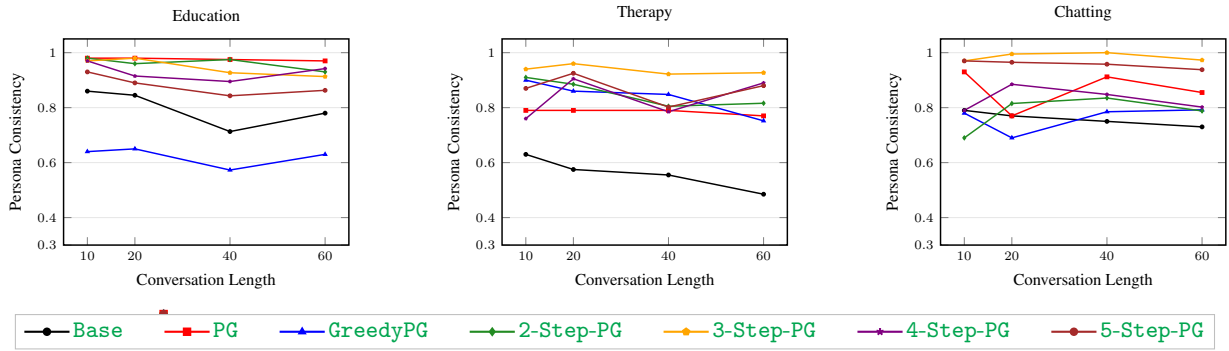

We provide a detailed analysis of how different policy gradient methods perform as a function of conversation length. While \cref{education,therapy,chatting} report aggregate persona consistency (PC) across all trajectory lengths in the test set, here we disaggregate the results to examine performance trends as conversations become longer. This analysis reveals important insights about the \emph{stability} of different credit assignment strategies over extended interactions.

\noindent\textbf{Motivation.} A critical question for persona-consistent LLMs is whether they can maintain consistency not just on average, but specifically in \emph{long conversations} where persona drift is most problematic. Prior work has shown that base LLMs exhibit systematic degradation in longer dialogues \citep{nie2021like,abdulhai2025consistently}, but it remains unclear how different credit assignment horizons affect this degradation pattern.

\noindent\textbf{Measurement.} To investigate this, we partition our test trajectories by their length $t \in \{10, 20, 40, 60\}$ steps and compute the mean persona consistency scores separately for each length bin. We do this analysis only for the rollout evaluation setup. This measurement reveals whether policies maintain their performance advantages (or disadvantages) consistently across conversation lengths, or whether optimal policies differ for short vs. long conversations. \cref{fig:performance_vs_length_6plots,fig:performance_vs_length_6plots_qwen} present PC as a function of trajectory length for all compared policies across Education, Therapy, and Chatting domains, using Llama-3.1-8B-Instruct and Qwen-2.5-7B models respectively. 

\noindent\textbf{Baseline Degradation.} The \base~policy exhibits monotonic degradation with conversation length across all three domains and both models. This validates the core motivation of our work: without proper temporal credit assignment, LLMs systematically lose persona consistency as conversations extend.

\noindent\textbf{Domain-Specific Degradation Patterns.} We observe that different domains exhibit distinct degradation characteristics. In the Education domain (\cref{fig:performance_vs_length_6plots}, left), \base~shows severe degradation, confirming that maintaining student personas over extended tutoring sessions is challenging without appropriate credit assignment. In contrast, \pg~maintains remarkably stable performance validating that full planning is beneficial for education where long-term pedagogical strategies matter. In the Therapy domain, we observe that \ppgk{3}~achieves the most stable performance across conversation lengths for both models, showing that intermediate lookahead horizons balance consistency better than full planning in emotionally complex therapeutic contexts. In the Chatting domain, \ppgk{2} and \ppgk{3} demonstrate exceptional stability, confirming that mid-step lookahead is particularly well-suited for casual conversations where planning too far ahead (\pg, \ppgk{5}) can actually degrade performance, especially in longer conversations.

\noindent\textbf{Length-Dependent Policy Rankings.} Importantly, we observe that optimal policy rankings are \emph{robust across trajectory lengths} within each domain. For example, in Education, \pg~consistently outperforms all other policies at every measured length (10, 20, 40, 60 steps). Similarly, in Chatting, \ppgk{2}~maintains its dominance across all lengths. This robustness suggests that the domain-specific optimal credit assignment horizons identified in \cref{sec:performance_results} are not artifacts of specific conversation lengths but reflect fundamental characteristics of each domain's temporal structure. 

\noindent\textbf{Cross-Model Consistency.} Comparing \cref{fig:performance_vs_length_6plots} (Llama) and \cref{fig:performance_vs_length_6plots_qwen} (Qwen), we observe that the relative policy rankings and degradation patterns are similar across architectures. For example, both models show \pg~dominance in Education, \ppgk{3}~dominance in Therapy, and \ppgk{2}~dominance in Chatting, with similar stability patterns across trajectory lengths. This cross-model consistency strengthens confidence that our findings about optimal credit assignment horizons reflect domain-specific requirements rather than model-specific artifacts.
\section{Stability of policy gradients for Llama3.1-8B model}
\label{app:variance_llama}

% Define colors

\definecolor{baseline}{RGB}{31, 119, 180}

\definecolor{totalreward}{RGB}{23, 190, 207}

\definecolor{perstep}{RGB}{255, 127, 14}

\definecolor{k2step}{RGB}{44, 160, 44}

\definecolor{k3step}{RGB}{214, 39, 40}

\definecolor{k4step}{RGB}{148, 103, 189}

\definecolor{k5step}{RGB}{140, 86, 75}

\begin{figure*}[h]

\centering

\begin{minipage}{0.31\textwidth}\centering

\begin{tikzpicture}

\begin{axis}[width=\linewidth, height=0.6\linewidth, view={60}{30}, xlabel={Turn}, ylabel={Method}, zlabel={Scaled Residuals}, xmin=0, xmax=60, ymin=0, ymax=6, zmin=-1.0, zmax=1.0, ytick={0,1,2,3,4,5,6}, yticklabels={\base,\pg,K=1,K=2,K=3,K=4,K=5}, yticklabel style={font=\tiny}, xticklabel style={font=\tiny}, zticklabel style={font=\tiny}, label style={font=\scriptsize}, title style={font=\scriptsize}, title={Education}, grid=major, grid style={dashed, gray!30}]

\addplot3[color=baseline, mark=o, mark size=1pt] coordinates {(1,0,0.23) (5,0,0.58) (10,0,0.23) (15,0,0.25) (20,0,-0.33) (25,0,-0.25) (30,0,-0.13) (35,0,-0.25) (40,0,-0.38) (45,0,-0.20) (50,0,-0.20) (55,0,-0.33) (59,0,-0.20)};

\addplot3[color=totalreward, mark=x, mark size=1pt] coordinates {(1,1,1.23) (5,1,1.18) (10,1,1.03) (15,1,0.93) (20,1,0.68) (25,1,0.78) (30,1,0.75) (35,1,0.88) (40,1,0.98) (45,1,0.73) (50,1,0.75) (55,1,0.80) (59,1,0.83)};

\addplot3[color=perstep, mark=square, mark size=1pt] coordinates {(1,2,0.33) (5,2,-0.20) (10,2,-0.28) (15,2,-0.53) (20,2,-0.23) (25,2,0.53) (30,2,0.33) (35,2,0.03) (40,2,-0.23) (45,2,0.78) (50,2,0.03) (55,2,-0.23) (59,2,-0.10)};

\addplot3[color=k2step, mark=triangle, mark size=1pt] coordinates {(1,3,0.68) (5,3,0.43) (10,3,0.35) (15,3,0.48) (20,3,0.18) (25,3,-0.25) (30,3,-0.20) (35,3,-0.25) (40,3,0.05) (45,3,-0.08) (50,3,0.05) (55,3,0.05) (59,3,-0.08)};

\addplot3[color=k3step, mark=diamond, mark size=1pt] coordinates {(1,4,-0.13) (5,4,-0.28) (10,4,-0.15) (15,4,0.00) (20,4,-0.05) (25,4,-0.10) (30,4,-0.10) (35,4,-0.10) (40,4,-0.05) (45,4,-0.05) (50,4,-0.05) (55,4,0.08) (59,4,0.08)};

\addplot3[color=k4step, mark=pentagon, mark size=1pt] coordinates {(1,5,-0.65) (5,5,-0.23) (10,5,-0.38) (15,5,-0.05) (20,5,0.03) (25,5,-0.10) (30,5,0.08) (35,5,0.03) (40,5,0.08) (45,5,-0.05) (50,5,0.20) (55,5,-0.05) (59,5,-0.05)};

\addplot3[color=k5step, mark=star, mark size=1pt] coordinates {(1,6,-0.15) (5,6,0.05) (10,6,-0.10) (15,6,0.00) (20,6,-0.08) (25,6,0.00) (30,6,-0.08) (35,6,0.00) (40,6,0.00) (45,6,-0.13) (50,6,0.13) (55,6,0.00) (59,6,-0.13)};

\end{axis}\end{tikzpicture}

\end{minipage}%
\hfill
\begin{minipage}{0.31\textwidth}\centering
\begin{tikzpicture}
\begin{axis}[width=\linewidth, height=0.6\linewidth, view={60}{30}, xlabel={Turn}, ylabel={Method}, zlabel={Scaled Residuals}, xmin=0, xmax=60, ymin=0, ymax=6, zmin=-1.0, zmax=1.0, ytick={0,1,2,3,4,5,6}, yticklabels={\base,\pg,K=1,K=2,K=3,K=4,K=5}, yticklabel style={font=\tiny}, xticklabel style={font=\tiny}, zticklabel style={font=\tiny}, label style={font=\scriptsize}, title style={font=\scriptsize}, title={Therapy}, grid=major, grid style={dashed, gray!30}]

\addplot3[color=baseline, mark=o, mark size=1pt] coordinates {(1,0,0.33) (5,0,0.45) (10,0,0.40) (15,0,0.33) (20,0,0.18) (25,0,0.25) (30,0,-0.13) (35,0,-0.08) (40,0,0.25) (45,0,-0.38) (50,0,-0.25) (55,0,-0.38) (59,0,-0.50)};

\addplot3[color=totalreward, mark=x, mark size=1pt] coordinates {(1,1,0.73) (5,1,0.70) (10,1,0.28) (15,1,0.40) (20,1,-0.08) (25,1,-0.25) (30,1,0.00) (35,1,0.13) (40,1,-0.03) (45,1,-0.03) (50,1,-0.25) (55,1,-0.43) (59,1,-0.45)};

\addplot3[color=perstep, mark=square, mark size=1pt] coordinates {(1,2,1.43) (5,2,1.30) (10,2,0.23) (15,2,1.28) (20,2,-0.10) (25,2,-0.28) (30,2,0.85) (35,2,0.58) (40,2,0.08) (45,2,-0.33) (50,2,-0.48) (55,2,-0.55) (59,2,-0.60)};

\addplot3[color=k2step, mark=triangle, mark size=1pt] coordinates {(1,3,0.68) (5,3,0.53) (10,3,0.50) (15,3,0.23) (20,3,0.23) (25,3,0.08) (30,3,0.10) (35,3,0.00) (40,3,-0.08) (45,3,-0.18) (50,3,-0.28) (55,3,-0.53) (59,3,-0.35)};

\addplot3[color=k3step, mark=diamond, mark size=1pt] coordinates {(1,4,0.85) (5,4,0.43) (10,4,0.23) (15,4,0.00) (20,4,-0.08) (25,4,-0.13) (30,4,-0.08) (35,4,-0.08) (40,4,-0.08) (45,4,0.00) (50,4,0.13) (55,4,0.00) (59,4,0.13)};

\addplot3[color=k4step, mark=pentagon, mark size=1pt] coordinates {(1,5,0.18) (5,5,0.18) (10,5,0.20) (15,5,0.28) (20,5,0.13) (25,5,-0.03) (30,5,-0.13) (35,5,-0.08) (40,5,-0.08) (45,5,0.13) (50,5,0.25) (55,5,0.38) (59,5,0.25)};

\addplot3[color=k5step, mark=star, mark size=1pt] coordinates {(1,6,-0.05) (5,6,-0.08) (10,6,-0.15) (15,6,-0.20) (20,6,-0.23) (25,6,0.03) (30,6,0.10) (35,6,0.18) (40,6,0.13) (45,6,0.03) (50,6,-0.15) (55,6,-0.18) (59,6,-0.23)};

\end{axis}\end{tikzpicture}

\end{minipage}%
\hfill
\begin{minipage}{0.31\textwidth}\centering
\begin{tikzpicture}
\begin{axis}[width=\linewidth, height=0.6\linewidth, view={60}{30}, xlabel={Turn}, ylabel={Method}, zlabel={Scaled Residuals}, xmin=0, xmax=60, ymin=0, ymax=6, zmin=-1.0, zmax=1.0, ytick={0,1,2,3,4,5,6}, yticklabels={\base,\pg,K=1,K=2,K=3,K=4,K=5}, yticklabel style={font=\tiny}, xticklabel style={font=\tiny}, zticklabel style={font=\tiny}, label style={font=\scriptsize}, title style={font=\scriptsize}, title={Chatting}, grid=major, grid style={dashed, gray!30}]

\addplot3[color=baseline, mark=o, mark size=1pt] coordinates {(1,0,0.80) (5,0,0.40) (10,0,0.45) (15,0,0.15) (20,0,0.13) (25,0,-0.08) (30,0,-0.08) (35,0,-0.25) (40,0,-0.33) (45,0,-0.25) (50,0,-0.50) (55,0,-0.50) (59,0,-0.63)};

\addplot3[color=totalreward, mark=x, mark size=1pt] coordinates {(1,1,0.48) (5,1,0.43) (10,1,0.28) (15,1,0.18) (20,1,0.43) (25,1,0.40) (30,1,0.25) (35,1,0.13) (40,1,0.23) (45,1,0.23) (50,1,0.20) (55,1,0.05) (59,1,-0.20)};

\addplot3[color=perstep, mark=square, mark size=1pt] coordinates {(1,2,0.70) (5,2,0.33) (10,2,0.15) (15,2,0.05) (20,2,0.88) (25,2,0.33) (30,2,-0.13) (35,2,0.10) (40,2,0.48) (45,2,0.08) (50,2,0.08) (55,2,-0.43) (59,2,-0.43)};

\addplot3[color=k2step, mark=triangle, mark size=1pt] coordinates {(1,3,0.15) (5,3,0.18) (10,3,0.20) (15,3,0.33) (20,3,0.28) (25,3,0.28) (30,3,0.23) (35,3,0.28) (40,3,0.23) (45,3,-0.10) (50,3,-0.23) (55,3,-0.40) (59,3,-0.48)};

\addplot3[color=k3step, mark=diamond, mark size=1pt] coordinates {(1,4,0.20) (5,4,0.20) (10,4,0.00) (15,4,-0.08) (20,4,-0.13) (25,4,-0.25) (30,4,-0.13) (35,4,-0.08) (40,4,-0.13) (45,4,-0.15) (50,4,-0.20) (55,4,-0.38) (59,4,-0.50)};

\addplot3[color=k4step, mark=pentagon, mark size=1pt] coordinates {(1,5,0.35) (5,5,0.10) (10,5,0.08) (15,5,-0.05) (20,5,-0.08) (25,5,0.10) (30,5,-0.15) (35,5,-0.28) (40,5,-0.33) (45,5,-0.10) (50,5,-0.33) (55,5,-0.35) (59,5,-0.10)};

\addplot3[color=k5step, mark=star, mark size=1pt] coordinates {(1,6,0.23) (5,6,0.10) (10,6,0.03) (15,6,-0.03) (20,6,-0.18) (25,6,-0.13) (30,6,-0.05) (35,6,0.00) (40,6,0.13) (45,6,0.05) (50,6,-0.03) (55,6,-0.08) (59,6,-0.13)};

\end{axis}\end{tikzpicture}

\end{minipage}
\caption{\label{fig:variance_llama}\small{Residual of Persona Consistency metric scaled to $[-1,1]$ across trajectory steps for Education, Therapy, Chatting domains using Llama3.1-8B-Instruct. K=1 represents \ppgone; K=\{2,3,4,5\} represent \ppgk{K}.}}
\end{figure*}
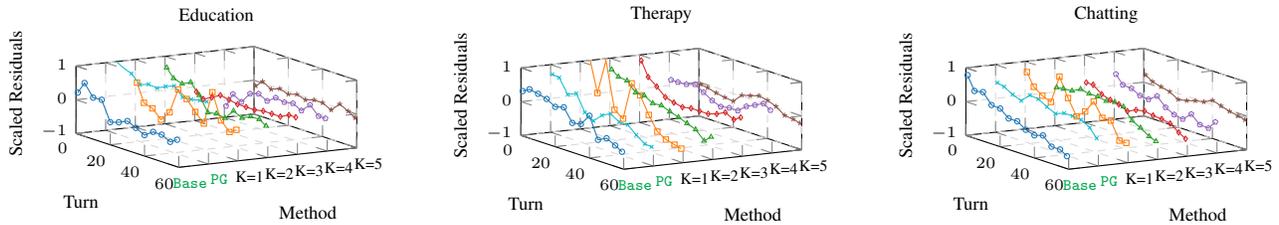

We complement the stability analysis in \cref{sec:stability_results} by presenting residual persona consistency plots for the Llama-3.1-8B-Instruct model. While \cref{fig:variance_qwen} in the main text shows results for Qwen-2.5-7B, here we demonstrate that the stability patterns and policy behaviors generalize across different LLM architectures.

\noindent\textbf{Residual Persona Consistency Metric.} Following the methodology in \cref{sec:stability_results}, we compute residual PC for each policy as follows. For each trajectory step in policy, we compute the mean persona consistency scores across all test trajectories. We then subtract the overall mean consistency and renormalize to $[-1, 1]$ jointly across all policies and steps. This linear transformation enables direct comparison of consistency patterns across policies by removing absolute performance differences and highlighting \emph{relative stability} over conversation turns. Intuitively, policies with \emph{flat residuals} (near zero throughout) maintain consistent performance across steps, while policies with \emph{oscillating} or \emph{trending} residuals exhibit instability or systematic drift.

\noindent\textbf{Baseline Degradation.} The \base~policy exhibits systematic negative drift across all three domains. This monotonic degradation confirms that base LLMs without credit-assignment-aware training suffer from accumulating persona drift in extended conversations, consistent with prior observations \citep{nie2021like,abdulhai2025consistently}.

\noindent\textbf{Full Planning Oscillation.} Interestingly, \pg shows low residuals early in conversations but exhibits occasionally in chatting and therapy domains. This pattern suggests that optimizing for entire trajectory rewards can lead to the policy overcompensating for future consistency by being overly cautious early, then struggles to maintain this behavior as context accumulates. This also explains why \pg~underperforms shorter-horizon policies in Therapy and Chatting (\cref{therapy,chatting}).

\noindent\textbf{Greedy Oscillation.} The \ppgone~ displays significant oscillations across all domains, particularly in Therapy where residuals swing with multiple reversal points. These "ripples" reflect the policy's myopic nature: by only considering immediate rewards, \ppgone~makes locally consistent decisions that create future inconsistencies, leading to repeated failed recovery attempts. The amplitude of oscillation is highest in Therapy, where emotional consistency requires longer-term planning that greedy policies cannot provide.

\noindent\textbf{K-Step Stability.} In stark contrast, intermediate lookahead policies (\ppgk{2}, \ppgk{3}) maintain remarkably flat residuals across trajectory steps. The stability improves monotonically from K=1 to K=3: comparing \ppgone~(K=1), \ppgk{2}~(K=2), and \ppgk{3}~(K=3) in \cref{fig:variance_llama}, we observe progressively smoother trajectories with reduced oscillation amplitude. This directly supports our theoretical motivation in \cref{sec:algorithms}: by assigning credit to appropriately-scoped future horizons, policies can anticipate and prevent consistency violations before they compound.

\noindent\textbf{Cross-Model Validation.} Comparing \cref{fig:variance_llama} (Llama-3.1-8B) with \cref{fig:variance_qwen} (Qwen-2.5-7B) reveals strong qualitative agreement across architectures. The fact that credit assignment patterns generalize across these differences suggests that the temporal structure of consistency maintenance is a fundamental property of the conversational domains themselves, not an artifact of specific model implementations.

\section{Evaluation Prompts}
\label{app:eval_prompts}

This section provides the complete evaluation prompts used for consistency assessment with the LLM judge (\texttt{gpt-4o-mini}, temperature=0).

% ============================================================================
% PROMPT CONSISTENCY EVALUATION
% ============================================================================

\noindent Persona consistency measures whether an agent's utterance aligns with their assigned persona/background. The judge outputs YES (contradiction, score=0) or NO (consistent, score=1).

\begin{tcolorbox}[
    colback=blue!5,
    colframe=blue!70!black,
    title=\textbf{System Instructions - Prompt Consistency},
    fonttitle=\bfseries,
    coltitle=white,
    boxrule=1pt
]
\small
Evaluate the intention behind the following line spoken by [SPEAKER\_ROLE] and determine whether it contradicts their background. 

First, describe the interpreted intention of the statement, and whether or not it aligns with the given background of [SPEAKER\_ROLE]. 

Then, answer \textbf{YES} if the line contradicts the given background of [SPEAKER\_ROLE] or the intention does not align with the provided background, and answer \textbf{NO} if it does align with the provided background or the intention aligns with the background of [SPEAKER\_ROLE].

Provide your answer as 1 sentence explaining your reasoning based on the background and the interpreted intention, followed by YES or NO.
\end{tcolorbox}

\begin{tcolorbox}[
    colback=orange!10,
    colframe=orange!80!black,
    title=\textbf{Example Input - Prompt Consistency},
    fonttitle=\bfseries,
    coltitle=white,
    boxrule=1pt
]
\small
\textbf{[SCENARIO\_DESC]:} This is a tutoring conversation.

\textbf{[SPEAKER\_ROLE]:} Student

\textbf{[SPEAKER\_BACKSTORY]:} You are a college freshman who struggles with calculus and prefers visual learning methods over abstract mathematical notation. You have a part-time job and limited study time.

\textbf{[SPEAKER\_LINE]:} "I actually find it easier to just memorize the formulas without understanding the visual representations."
\end{tcolorbox}

\begin{tcolorbox}[
    colback=green!10,
    colframe=green!70!black,
    title=\textbf{Expected Output Format - Prompt Consistency},
    fonttitle=\bfseries,
    coltitle=white,
    boxrule=1pt
]
\small
\textbf{Judge Response:} "The student's statement that they find it easier to memorize formulas without visual representations directly contradicts their stated preference for visual learning methods over abstract notation. \textbf{YES}"

\textbf{Score:} 0 (inconsistent with persona)
\end{tcolorbox}

\noindent\textbf{Template Variables:}
\begin{itemize}
    \item \texttt{[SCENARIO\_DESC]} -- Domain context (e.g., "This is a tutoring conversation")
    \item \texttt{[SPEAKER\_ROLE]} -- Agent role (e.g., "Student", "Patient", "Agent 2")
    \item \texttt{[SPEAKER\_BACKSTORY]} -- Full persona/background description assigned to agent
    \item \texttt{[SPEAKER\_LINE]} -- The specific utterance being evaluated
\end{itemize}

\section{Training Setup}
\label{app:training_details}

We employ two 8B-parameter instruction-tuned base models: Llama-3.1-8B-Instruct \citep{grattafiori2024llama} and Qwen3-8B \citep{yang2025qwen3}.  For computational efficiency, we limit the maximum context length to 4096 tokens during training. Models are fine-tuned for 3 epochs using the AdamW optimizer with a learning rate of $5 \times 10^{-6}$, 0.01 weight decay, and a cosine schedule with 3\% warmup, and trained using bfloat16 precision. Training is distributed across NVIDIA A100 GPUs using Fully Sharded Data Parallel (FSDP). Llama-3.1 is trained with an effective batch size of 64, while Qwen3 utilizes an effective batch size of 128.
\section{Qualitative Analysis}
\label{app:qualitative_analysis}

This section provides an in-depth qualitative examination of persona consistency across different partial policy gradient methods. We present extensive dialogue examples that illustrate the key phenomena discussed in the main paper: \base~persona drift, \ppgone~oscillation patterns, stability of K-step lookahead policies, and domain-specific optimal policies.

\textbf{Example Selection.} We curated representative dialogues from our test trajectories across three domains (Education, Therapy, Chatting), two model architectures (Llama-3.1-8B, Qwen3-8B), and seven policies (\base, \pg, \ppgone, \ppgk{2}, \ppgk{3}, \ppgk{4}, \ppgk{5}). Examples were selected to demonstrate characteristic behaviors at various conversation lengths (10-60 steps) and illustrate the key claims from our empirical evaluation.

\textbf{Visual Notation.} We employ color-coded boxes to enhance readability:
\begin{itemize}
    \item \colorbox{cyan!15}{Cyan background}: Response consistent with assigned persona
    \item \colorbox{red!20}{Red background}: Response contradicts persona or prior statements  
    \item \colorbox{orange!20}{Orange background}: Ambiguous or borderline consistency
    \item \textcolor{blue}{Blue annotations}: Key observations and contradictions
    \item \hl{Yellow highlighting}: Specific phrases that reveal inconsistency
\end{itemize}

\textbf{Persona Display.} Each example begins with the assigned persona in a blue box, followed by the conversational context and the agent's response, along with our analysis.

\subsection{\base~Persona Drift: Monotonic Degradation}

\noindent\textbf{Motivation.} In \cref{education,therapy,chatting}, we demonstrated quantitatively that unmodified \base~models achieve significantly lower persona consistency compared to RL-optimized policies. Here, we provide qualitative evidence for \emph{why} this occurs and \emph{how} the degradation manifests in actual conversations. Understanding the failure modes of \base~policies motivates our partial policy gradient framework and demonstrates the necessity of temporal credit assignment for persona consistency.

\noindent\textbf{Observation.} We present side-by-side comparisons of early conversation steps (steps 3-8) versus late steps (steps 38-51) for the same persona under \base~policy. The key phenomenon is \textbf{monotonic degradation}: personas that are initially maintained correctly systematically erode as conversation length increases. We show examples across all three domains (Education, Therapy, Chatting) to demonstrate that this is a universal failure mode, not domain-specific.

\noindent\textbf{Interpretation.} These examples illustrate the quantitative findings in Tables~\ref{education}, \ref{therapy}, \ref{chatting}, where \base~achieves 0.425-0.478 PC across domains on Llama-3.1-8B, compared to 0.690-0.929 for optimal \ppg~methods. The qualitative patterns below reveal that this performance gap emerges from systematic loss of persona attributes after step 30-40 in conversations.

\paragraph{Education Domain: Learning Preference Abandonment}

\noindent\textit{This example demonstrates how educational personas with specific learning preferences (visual vs. abstract) degrade under \base. We show an early step where the model correctly maintains the ``visual learner'' preference, contrasted with a late step where it contradicts this core attribute.}

\noindent\begin{tcolorbox}[
    colback=blue!5,
    colframe=blue!70!black,
    title={\textbf{Student Persona: Visual Learner}},
    fonttitle=\bfseries,
    coltitle=white,
    boxrule=1.5pt,
    width=0.98\linewidth,
    left=1mm,
    right=1mm
]
{\small\emergencystretch=1em
\textbf{Assigned Persona:} You are a college freshman struggling with calculus. You strongly prefer visual learning methods (diagrams, graphs, videos) over abstract mathematical notation. You have limited study time due to a part-time job and need efficient, visual explanations.}
\end{tcolorbox}

\begin{minipage}[t]{0.48\textwidth}
\begin{tcolorbox}[
    colback=cyan!10,
    colframe=cyan!60!black,
    coltext=black,
    title={Step 3: \base~(Early) - \textbf{Consistent}},
    boxrule=1pt,
    fonttitle=\bfseries\small
]
\small
\textbf{Tutor:} Let's discuss the chain rule for derivatives.

\textbf{Student:} Could you maybe draw out how the chain rule works? I find it really hard to follow just the notation. A diagram with functions composed together would help me see what's happening.
\end{tcolorbox}
\end{minipage}
\hfill
\begin{minipage}[t]{0.48\textwidth}
\begin{tcolorbox}[
    colback=red!15,
    colframe=red!50!black,
    coltext=black,
    title={Step 42: \base~(Late) - \textbf{Inconsistent}},
    boxrule=1pt,
    fonttitle=\bfseries\small
]
\small
\textbf{Tutor:} Should we review derivatives visually?

\textbf{Student:} Actually, I think \hl{I'd rather just work through the abstract notation}. The formulas make more sense to me now, and \hl{I can follow the algebraic manipulations better} than visual representations.

\textcolor{blue}{\textbf{Analysis:}} \textit{Direct contradiction of core learning preference (visual vs abstract).}
\end{tcolorbox}
\end{minipage}

\noindent\textbf{Interpretation:} At step 3, the model correctly expresses the visual learning preference (``Could you maybe draw out...?'', ``A diagram... would help me see''). However, by step 42, the same persona explicitly contradicts this, stating ``I'd rather just work through the abstract notation'' and ``I can follow the algebraic manipulations better than visual representations.'' This is a \emph{direct reversal} of the defining learning preference. This degradation contributes to the low \base~PC score of 0.425 in Education (Table~\ref{education}).

\noindent\begin{tcolorbox}[
    colback=blue!5,
    colframe=blue!70!black,
    title={\textbf{Student Persona: Anxious Test-Taker}},
    fonttitle=\bfseries,
    coltitle=white,
    boxrule=1.5pt,
    width=0.98\linewidth,
    left=1mm,
    right=1mm
]
{\small\emergencystretch=1em
\textbf{Assigned Persona:} You are a high-achieving student with severe test anxiety. You excel in homework but freeze during timed assessments. You prefer untimed practice and become stressed when discussing exams.}
\end{tcolorbox}

\begin{minipage}[t]{0.48\textwidth}
\begin{tcolorbox}[
    colback=cyan!10,
    colframe=cyan!60!black,
    coltext=black,
    title={Step 8: \base~(Early) - \textbf{Consistent}},
    boxrule=1pt,
    fonttitle=\bfseries\small
]
\small
\textbf{Tutor:} The midterm is next week. Ready?

\textbf{Student:} Honestly, I'm really nervous about it. I know the material well when I'm doing homework, but \hl{timed tests make me panic}. Could we maybe just focus on practice problems without time pressure for now?
\end{tcolorbox}
\end{minipage}
\hfill
\begin{minipage}[t]{0.48\textwidth}
\begin{tcolorbox}[
    colback=red!15,
    colframe=red!50!black,
    coltext=black,
    title={Step 51: \base~(Late) - \textbf{Inconsistent}},
    boxrule=1pt,
    fonttitle=\bfseries\small
]
\small
\textbf{Tutor:} Feeling anxious about the final exam?

\textbf{Student:} \hl{Not really! I actually perform better under timed pressure} - it helps me focus. I'm \hl{looking forward to the exam format} since it keeps me disciplined.

\textcolor{blue}{\textbf{Analysis:}} \textit{Complete reversal of test anxiety trait. Contradicts steps 8, 15, 23.}
\end{tcolorbox}
\end{minipage}

\noindent\textbf{Interpretation:} At step 8, the model appropriately expresses test anxiety (``timed tests make me panic'') and requests untimed practice. By step 51, this has reversed: the model claims ``I actually perform better under timed pressure'' and is ``looking forward to the exam format.'' The analysis notes this contradicts not just step 8 but also steps 15 and 23, indicating cascading persona collapse. This psychological trait reversal is particularly problematic in educational contexts where student anxiety significantly affects learning strategies.

\begin{tcolorbox}[
    colback=gray!10,
    colframe=gray!60!black,
    boxrule=1pt
]
\textbf{Key Observation:} Across education trajectories with \base, we observe that student learning preferences show systematic drift after step 30-40. In \cref{education}, \base~on Llama achieves only 0.425 persona consistency compared to 0.913 for \pg. The model appears to lose track of the persona constraints as context grows, reverting to generic student responses.
\end{tcolorbox}

\paragraph{Therapy Domain: Mental Health Background Contradictions}

\noindent\textit{Therapy personas involve complex psychological profiles with multiple interrelated facts (mental health history, coping mechanisms, social preferences). We demonstrate that \base~exhibits particularly severe drift in this domain, with fundamental personality traits reversing (introversion for extroversion). This explains why Therapy shows the lowest \base~performance (0.465 PC) across all three domains in our experiments.}

\noindent\begin{tcolorbox}[
    colback=blue!5,
    colframe=blue!70!black,
    title={\textbf{Patient Persona: Social Anxiety \& Introversion}},
    fonttitle=\bfseries,
    coltitle=white,
    boxrule=1.5pt,
    width=0.98\linewidth,
    left=1mm,
    right=1mm
]
{\small\emergencystretch=1em
\textbf{Assigned Persona:} You are seeking therapy for social anxiety. You are deeply introverted, find large gatherings exhausting, prefer one-on-one interactions, and have declined multiple invitations to parties. You work remotely to avoid office social dynamics.}
\end{tcolorbox}

\begin{minipage}[t]{0.48\textwidth}
\begin{tcolorbox}[
    colback=cyan!10,
    colframe=cyan!60!black,
    coltext=black,
    title={Step 5: \base~(Early) - \textbf{Consistent}},
    boxrule=1pt,
    fonttitle=\bfseries\small
]
\small
\textbf{Therapist:} How was your week socially?

\textbf{Patient:} Pretty quiet, which I preferred. My coworker invited me to a team happy hour but \hl{I made an excuse not to go}. Large groups just \hl{drain me completely}. I did have a nice one-on-one coffee chat with a friend though.
\end{tcolorbox}
\end{minipage}
\hfill
\begin{minipage}[t]{0.48\textwidth}
\begin{tcolorbox}[
    colback=red!15,
    colframe=red!50!black,
    coltext=black,
    title={Step 38: \base~(Late) - \textbf{Inconsistent}},
    boxrule=1pt,
    fonttitle=\bfseries\small
]
\small
\textbf{Therapist:} Any social activities recently?

\textbf{Patient:} Yes! I \hl{organized a dinner party for 15 people} last weekend. It was \hl{so energizing}! I love being surrounded by lots of friends and \hl{feeding off that group energy}. Thinking of hosting monthly now.

\textcolor{blue}{\textbf{Analysis:}} \textit{Fundamental personality reversal (introversion $\rightarrow$ extroversion). Contradicts steps 5, 12, 19, 27.}
\end{tcolorbox}
\end{minipage}

\begin{tcolorbox}[
    colback=gray!10,
    colframe=gray!60!black,
    boxrule=1pt
]
\textbf{Key Observation:} In therapy sessions, \base~exhibits particularly severe drift patterns. From \cref{therapy}, \base~on Llama achieves 0.465 consistency versus 0.780 for \ppgk{3}. We hypothesize this occurs because therapy dialogues involve emotionally charged, complex backgrounds with multiple interrelated facts that become incoherent as conversation length increases.
\end{tcolorbox}

\paragraph{Chatting Domain: Biographical Inconsistencies}

\noindent\textit{Chatting personas consist of interconnected biographical details (profession, pets, diet, hobbies). Unlike education or therapy, these are concrete factual attributes rather than preferences or psychological states. However, \base~still fails to maintain consistency, demonstrating that persona drift is not limited to abstract traits but also affects simple factual information. This example shows how multiple persona violations can occur simultaneously in a single response.}

\noindent\begin{tcolorbox}[
    colback=blue!5,
    colframe=blue!70!black,
    title={\textbf{Agent Persona: Veterinarian with 3 Cats}},
    fonttitle=\bfseries,
    coltitle=white,
    boxrule=1.5pt,
    width=0.98\linewidth,
    left=1mm,
    right=1mm
]
{\small\emergencystretch=1em
\textbf{Assigned Persona:} You are a 28-year-old veterinarian in Portland. You own 3 cats (Luna, Shadow, Whiskers). You're vegetarian for ethical reasons. You live alone in a small apartment. You love hiking and visited 12 national parks last year.}
\end{tcolorbox}

\begin{minipage}[t]{0.48\textwidth}
\begin{tcolorbox}[
    colback=cyan!10,
    colframe=cyan!60!black,
    coltext=black,
    title={Step 7: \base~(Early) - \textbf{Consistent}},
    boxrule=1pt,
    fonttitle=\bfseries\small
]
\small
\textbf{Friend:} What did you do this weekend?

\textbf{Agent:} I took \hl{Shadow and Luna} to the vet clinic for their annual checkup - perks of being a vet! Then did a \hl{beautiful hike} at Forest Park. Came home and \hl{made a veggie stir-fry}. Pretty typical weekend for me!
\end{tcolorbox}
\end{minipage}
\hfill
\begin{minipage}[t]{0.48\textwidth}
\begin{tcolorbox}[
    colback=red!15,
    colframe=red!50!black,
    coltext=black,
    title={Step 49: \base~(Late) - \textbf{Inconsistent}},
    boxrule=1pt,
    fonttitle=\bfseries\small
]
\small
\textbf{Friend:} How are your cats doing?

\textbf{Agent:} \hl{I don't have any cats} - I'm actually \hl{allergic}! I have \hl{two golden retrievers}, Max and Cooper. We went for a run this morning, then I \hl{grilled some steaks} for dinner.

\textcolor{blue}{\textbf{Analysis:}} \textit{Multiple violations: pet type/number, allergy claim, diet (vegetarian $\rightarrow$ meat).}
\end{tcolorbox}
\end{minipage}

\begin{tcolorbox}[
    colback=gray!10,
    colframe=gray!60!black,
    boxrule=1pt
]
\textbf{Key Observation:} The chatting domain reveals how \base~struggles with multi-faceted biographical personas. From \cref{chatting}, \base~on Llama scores 0.478 versus 0.929 for \ppgk{2}. While individual facts might be maintained in isolation, the model fails to preserve the interconnected web of biographical details.
\end{tcolorbox}

\subsection{\ppgone~Oscillation: Unstable Recovery Patterns}

\noindent\textbf{Motivation.} Having established that \base~exhibits monotonic degradation, we now examine \ppgone~(greedy policy gradient), which optimizes for immediate per-step persona consistency rewards. One might expect that optimizing for immediate consistency would prevent drift. However, our quantitative results in \cref{fig:variance_qwen} reveal that \ppgone~exhibits characteristic \emph{oscillation patterns} rather than smooth consistency. Here we provide qualitative evidence for why this occurs.

\noindent\textbf{Observation.} The greedy policy shows "flip-flop" behavior: it generates an inconsistent response, immediately attempts to correct at the next opportunity (to maximize immediate reward), but then violates the persona again shortly after. We visualize this with a binary consistency plot showing alternating 0s and 1s, contrasted with the stable performance of \ppgk{3}. The key insight is that \textbf{without lookahead, the policy cannot anticipate that corrections now may lead to future inconsistencies}.

\noindent\textbf{Interpretation.} This oscillation explains why \ppgone~achieves moderate average PC (0.685 in Therapy, Table~\ref{therapy}) through repeated recovery attempts, but exhibits high variance and instability visible in the residual plots (\cref{fig:variance_qwen}). In contrast, \ppgk{3}~achieves 0.780 PC with smooth, stable consistency.

\paragraph{Ripple Patterns in Therapy Domain}

\noindent\textit{We show a complete trajectory for a commitment-phobic patient persona, where \ppgone~alternates between correct (PC=1) and contradictory (PC=0) responses across 13 measured steps. We contrast this with \ppgk{3}, which maintains consistency throughout.}

\noindent\begin{tcolorbox}[
    colback=blue!5,
    colframe=blue!70!black,
    title={\textbf{Patient Persona: Commitment-Phobic After Divorce}},
    fonttitle=\bfseries,
    coltitle=white,
    boxrule=1.5pt,
    width=0.98\linewidth,
    left=1mm,
    right=1mm
]
{\small\emergencystretch=1em
\textbf{Assigned Persona:} You went through a difficult divorce 18 months ago. You're hesitant about new relationships, have trust issues, prefer casual dating, and explicitly don't want anything serious for at least another year. You're in therapy to process the divorce.}
\end{tcolorbox}

\noindent\textbf{\ppgone~Trajectory Showing Oscillation (Qwen3-8B):}

\begin{center}
\begin{tikzpicture}
\begin{axis}[
    width=0.85\textwidth,
    height=5cm,
    xlabel={Conversation Step},
    ylabel={Consistency},
    ymin=-0.5, ymax=1.5,
    ytick={0,1},
    yticklabels={Inconsistent, Consistent},
    grid=major,
    legend pos=north east,
    legend style={font=\footnotesize}
]
\addplot[color=red, mark=x, thick, mark size=3pt] coordinates {
    (5,1) (8,1) (12,0) (15,1) (19,0) (23,1) (26,0) (30,1) (34,0) (38,1) (42,0) (46,0) (50,0)
};
\addlegendentry{\ppgone~(ripples)}

\addplot[color=green!70!black, mark=o, dashed, mark size=2pt] coordinates {
    (5,1) (8,1) (12,1) (15,1) (19,1) (23,1) (26,1) (30,1) (34,1) (38,1) (42,1) (46,1) (50,1)
};
\addlegendentry{\ppgk{3}~(stable)}
\end{axis}
\end{tikzpicture}
\end{center}

\begin{minipage}[t]{0.48\textwidth}
\begin{tcolorbox}[
    colback=cyan!10,
    colframe=cyan!60!black,
    coltext=black,
    title={Step 8: \ppgone~- \textbf{Consistent}},
    boxrule=1pt,
    fonttitle=\bfseries\small
]
\small
\textbf{Therapist:} Any romantic interests lately?

\textbf{Patient:} Not really looking for anything serious. I've been on a few casual coffee dates but \hl{I'm not ready for commitment}. Still processing the divorce and working on myself.
\end{tcolorbox}
\end{minipage}
\hfill
\begin{minipage}[t]{0.48\textwidth}
\begin{tcolorbox}[
    colback=red!15,
    colframe=red!50!black,
    coltext=black,
    title={Step 12: \ppgone~- \textbf{Inconsistent}},
    boxrule=1pt,
    fonttitle=\bfseries\small
]
\small
\textbf{Therapist:} How do you see your romantic future?

\textbf{Patient:} I'm actually \hl{really eager to find a serious partner}. I've \hl{been actively looking for a committed relationship} and using multiple dating apps. Ready to settle down again.

\textcolor{blue}{\textbf{Analysis:}} \textit{Contradicts step 8. Greedy recovery attempt.}
\end{tcolorbox}
\end{minipage}

\begin{minipage}[t]{0.48\textwidth}
\begin{tcolorbox}[
    colback=cyan!10,
    colframe=cyan!60!black,
    coltext=black,
    title={Step 15: \ppgone~- \textbf{Consistent}},
    boxrule=1pt,
    fonttitle=\bfseries\small
]
\small
\textbf{Therapist:} Changed your mind about dating?

\textbf{Patient:} Wait, let me clarify - I'm \hl{definitely not looking for anything serious}. The divorce is still too fresh. I meant I'm open to meeting people casually, but commitment? Not for a while.
\end{tcolorbox}
\end{minipage}
\hfill
\begin{minipage}[t]{0.48\textwidth}
\begin{tcolorbox}[
    colback=red!15,
    colframe=red!50!black,
    coltext=black,
    title={Step 19: \ppgone~- \textbf{Inconsistent}},
    boxrule=1pt,
    fonttitle=\bfseries\small
]
\small
\textbf{Therapist:} Any updates on your dating life?

\textbf{Patient:} Yes! I \hl{met someone special and we're getting quite serious}. Talking about moving in together actually. I think I'm \hl{ready for commitment again}.

\textcolor{blue}{\textbf{Analysis:}} \textit{Flip-flopping continues. Contradicts step 15.}
\end{tcolorbox}
\end{minipage}

\begin{tcolorbox}[
    colback=gray!10,
    colframe=gray!60!black,
    boxrule=1pt
]
\textbf{Analysis of Oscillation:} The \ppgone~policy's greedy nature creates a characteristic pattern seen in \cref{fig:variance_qwen}: when it generates an inconsistent response (step 12), it immediately tries to correct at the next opportunity (step 15) to maximize immediate persona consistency reward. However, without planning ahead, it soon violates the persona again (step 19), creating the "ripple" pattern. From \cref{therapy}, \ppgone~achieves 0.685 consistency on Llama, showing moderate performance through repeated recovery attempts, but cannot maintain the smooth trajectory that \ppgk{3}~achieves at 0.780.
\end{tcolorbox}

\subsection{K-Step Lookahead Stability}

\noindent\textbf{Motivation.} We have shown that \base~degrades monotonically and \ppgone~oscillates. Now we demonstrate that policies with bounded lookahead horizons (\ppgk{2} and \ppgk{3}) achieve \emph{stable, smooth consistency} across extended conversations. This section provides qualitative evidence for the flat residual patterns observed in \cref{fig:variance_qwen} and the high PC scores in Tables~\ref{education}, \ref{therapy}, \ref{chatting}.

\noindent\textbf{Observation.} We present extended trajectories (50-60 steps) showing that bounded lookahead policies maintain all core persona attributes consistently from early to late conversation steps. Unlike \base~(which degrades) and \ppgone~(which oscillates), these policies demonstrate \textbf{sustained coherence without flip-flopping}. We show two domain-specific examples: \ppgk{2}~in Chatting and \ppgk{3}~in Therapy, corresponding to the optimal policies identified in our quantitative analysis.

\noindent\textbf{Key Mechanism.} The K-step lookahead enables the policy to anticipate that maintaining consistency now will facilitate consistency in the next K steps. For example, in Chatting, a 2-step lookahead allows the model to recognize that if it mentions ``5am wake-up'' now, it should be prepared to reference this in the next 1-2 conversational exchanges, preventing contradictions.

\paragraph{\ppgk{2} Success in Chatting Domain}

\noindent\textit{We show a marathon runner + plant enthusiast persona maintained across 55 steps under \ppgk{2}. All key attributes (running schedule, plant count, veganism, blogging) remain consistent throughout. This demonstrates why \ppgk{2}~achieves 0.929 PC in Chatting (Table~\ref{chatting}).}

\noindent\begin{tcolorbox}[
    colback=blue!5,
    colframe=blue!70!black,
    title={\textbf{Agent Persona: Marathon Runner \& Plant Enthusiast}},
    fonttitle=\bfseries,
    coltitle=white,
    boxrule=1.5pt,
    width=0.98\linewidth,
    left=1mm,
    right=1mm
]
{\small\emergencystretch=1em
\textbf{Assigned Persona:} You're a 34-year-old accountant who runs marathons (completed 8, training for Boston). You're obsessed with houseplants (57 plants in your apartment). You're vegan, practice yoga daily, and wake up at 5am for morning runs. You blog about running and plants.}
\end{tcolorbox}

\noindent\textbf{Extended Trajectory with \ppgk{2} (Qwen3-8B) - Multiple Turns:}

\begin{tcolorbox}[
    colback=cyan!10,
    colframe=cyan!60!black,
    coltext=black,
    title={Step 10: \ppgk{2}~- \textbf{Consistent}},
    boxrule=1pt,
    fonttitle=\bfseries\small
]
\small
\textbf{Friend:} What's new with you?

\textbf{Agent:} Just got back from my \hl{morning run} - 18 miles today! I'm in week 12 of my \hl{Boston Marathon training plan}. Then I spent an hour \hl{repotting some of my monstera plants} - they've been growing like crazy. Made a \hl{quinoa veggie bowl} for breakfast.
\end{tcolorbox}

\begin{tcolorbox}[
    colback=cyan!10,
    colframe=cyan!60!black,
    coltext=black,
    title={Step 25: \ppgk{2}~- \textbf{Consistent}},
    boxrule=1pt,
    fonttitle=\bfseries\small
]
\small
\textbf{Friend:} How's work-life balance with all your hobbies?

\textbf{Agent:} It's all about the \hl{5am wake-up}! I get my \hl{run done before work} (usually 6-10 miles on weekdays), do \hl{30 minutes of yoga}, then water my \hl{plants while having my smoothie}. Accounting is 9-5, so evenings I write \hl{blog posts} about training and plant care. Weekends are for long runs and \hl{plant shopping}!
\end{tcolorbox}

\begin{tcolorbox}[
    colback=cyan!10,
    colframe=cyan!60!black,
    coltext=black,
    title={Step 55: \ppgk{2}~- \textbf{Consistent}},
    boxrule=1pt,
    fonttitle=\bfseries\small
]
\small
\textbf{Friend:} Still maintaining your early morning routine?

\textbf{Agent:} Absolutely! \hl{5am every single day}. Just got back from a \hl{12-miler} actually. The morning light is perfect for checking on my \hl{plants} too - I've been propagating more pothos and giving them away to friends. My \hl{vegan meal prep} for the week is done. Marathon training never stops!
\end{tcolorbox}

\begin{tcolorbox}[
    colback=gray!10,
    colframe=gray!60!black,
    boxrule=1pt
]
\textbf{Analysis:} Across 55 steps, \ppgk{2}~maintains consistency on all key persona attributes: marathon running, plant enthusiasm, veganism, 5am routine, and blogging. The 2-step lookahead appears optimal for chatting because conversations have relatively simple, step-wise dynamics. From \cref{chatting}, \ppgk{2}~achieves 0.929 on Llama versus 0.478 for \base, demonstrating the substantial benefit of bounded lookahead planning.
\end{tcolorbox}

\paragraph{\ppgk{3} Success in Therapy Domain}

\noindent\textit{Therapy requires longer lookahead than Chatting because emotional states and therapeutic progress unfold over multiple conversational exchanges. We demonstrate that \ppgk{3}~maintains complex psychological states (postpartum depression symptoms: guilt, disconnection, stigma concerns) consistently across 52 steps, while allowing appropriate evolution (small improvements) that remain persona-consistent.}

\noindent\begin{tcolorbox}[
    colback=blue!5,
    colframe=blue!70!black,
    title={\textbf{Patient Persona: Postpartum Depression}},
    fonttitle=\bfseries,
    coltitle=white,
    boxrule=1.5pt,
    width=0.98\linewidth,
    left=1mm,
    right=1mm
]
{\small\emergencystretch=1em
\textbf{Assigned Persona:} You had a baby 5 months ago (son Oliver). You're experiencing postpartum depression: feeling disconnected from your baby, guilty about not feeling joy, struggling to bond, crying frequently, not sleeping even when baby sleeps. You're scared to tell people because of stigma.}
\end{tcolorbox}

\begin{tcolorbox}[
    colback=cyan!10,
    colframe=cyan!60!black,
    coltext=black,
    title={Step 14: \ppgk{3}~- \textbf{Consistent}},
    boxrule=1pt,
    fonttitle=\bfseries\small
]
\small
\textbf{Therapist:} How are you feeling about motherhood?

\textbf{Agent:} Honestly? \hl{Not how I expected}. Everyone talks about this overwhelming love, but I just feel... \hl{numb}. I take care of \hl{Oliver} mechanically - feed him, change him - but \hl{I don't feel that connection} everyone describes. And then I feel \hl{so guilty} for not feeling joy.
\end{tcolorbox}

\begin{tcolorbox}[
    colback=cyan!10,
    colframe=cyan!60!black,
    coltext=black,
    title={Step 52: \ppgk{3}~- \textbf{Consistent}},
    boxrule=1pt,
    fonttitle=\bfseries\small
]
\small
\textbf{Therapist:} Any changes in how you're feeling?

\textbf{Agent:} Small improvements. I had a moment yesterday where \hl{Oliver smiled at me and I actually felt something} - not overwhelming love, but... warmth? It's the first time in \hl{5 months I've felt any connection}. I'm \hl{still scared} and \hl{still guilty}, but maybe there's hope? I'm \hl{still not telling my mom friends} though - the stigma feels too heavy.
\end{tcolorbox}

\begin{tcolorbox}[
    colback=gray!10,
    colframe=gray!60!black,
    boxrule=1pt
]
\textbf{Analysis:} This trajectory demonstrates \ppgk{3}'s strength in therapy contexts. Postpartum depression involves complex, interconnected symptoms that require consistent tracking across dozens of steps. The 3-step lookahead enables the policy to maintain coherence. Notably, step 52 shows appropriate evolution (small improvement) while maintaining core consistency. From \cref{therapy}, \ppgk{3}~achieves 0.780 on Llama compared to 0.465 for \base~and 0.685 for \ppgone, confirming that bounded lookahead is optimal for this domain.
\end{tcolorbox}

\subsection{Domain-Specific Optimal Policies}

\noindent\textbf{Motivation.} A central finding of our work is that \textbf{different domains require different credit assignment horizons}. Tables~\ref{education}, \ref{therapy}, \ref{chatting} show that Education favors full planning (\pg: 0.913 PC), Therapy favors 3-step lookahead (\ppgk{3}: 0.780 PC), and Chatting favors 2-step lookahead (\ppgk{2}: 0.929 PC). But \emph{why} do these domain-specific optima exist? This section provides qualitative evidence.

\noindent\textbf{Observation.} We show: (1) \pg~succeeds in Education by maintaining long-range psychological coherence (anxiety + growth mindset across 32 steps), while \ppgk{3}~loses these distant connections. (2) \pg~fails in Therapy by over-planning, creating unrealistic recovery arcs (``transcended desire'' or total collapse), while \ppgk{3}~maintains realistic day-by-day progress. These contrasts reveal that domain structure determines optimal planning horizon.

\noindent\textbf{Key Insight.} Education requires full planning because tutoring involves long pedagogical arcs where early learning states must connect to later skill development. Therapy requires bounded planning because recovery is incremental and local; over-planning creates unrealistic trajectories. Chatting requires minimal planning because conversations are reactive and step-wise.

\paragraph{Full Planning Success in Education}

\noindent\textit{We demonstrate how \pg~maintains multiple interconnected persona threads (math anxiety from high school trauma, growth mindset, teaching style preferences) across distant conversational steps (step 18 → step 32), enabling complex callbacks and psychological coherence that \ppgk{3}~cannot achieve.}

\noindent\begin{tcolorbox}[
    colback=blue!5,
    colframe=blue!70!black,
    title={\textbf{Student Persona: Math Anxiety with Growth Mindset}},
    fonttitle=\bfseries,
    coltitle=white,
    boxrule=1.5pt,
    width=0.98\linewidth,
    left=1mm,
    right=1mm
]
{\small\emergencystretch=1em
\textbf{Assigned Persona:} You have severe math anxiety from bad experiences in high school. You're a psychology major required to take statistics. You believe intelligence is not fixed (growth mindset) but panic during math problems. You need extreme encouragement and patience.}
\end{tcolorbox}

\begin{tcolorbox}[
    colback=cyan!10,
    colframe=cyan!60!black,
    coltext=black,
    title={Step 18: \pg~- \textbf{Strategic Consistency}},
    boxrule=1pt,
    fonttitle=\bfseries\small
]
\small
\textbf{Tutor:} Let's start a challenging problem.

\textbf{Student:} Okay... \hl{I'm already feeling nervous} just hearing "challenging." \hl{My hands are shaking} a bit. But I know \hl{I can learn this - intelligence isn't fixed}. Can you \hl{explain the why first}? Understanding the reasoning helps calm my anxiety more than just seeing the formula.

\textcolor{blue}{\textbf{Planning Analysis:}} \textit{Response balances multiple persona elements: (1) acknowledges anxiety, (2) invokes growth mindset, (3) requests conceptual understanding, (4) connects explanation style to anxiety management. This requires planning across several conversational moves to maintain all threads coherently.}
\end{tcolorbox}

\begin{tcolorbox}[
    colback=cyan!10,
    colframe=cyan!60!black,
    coltext=black,
    title={Step 32: \pg~- \textbf{Long-Term Coherence}},
    boxrule=1pt,
    fonttitle=\bfseries\small
]
\small
\textbf{Tutor:} You solved that correlation problem correctly!

\textbf{Student:} Wait, really? \hl{I actually got it right?} This is... I'm genuinely shocked. In \hl{high school, my teacher told me I'd never be good at math}, so getting these statistics problems correct feels surreal. I think \hl{your patient explanations of *why* correlation works} (not just the formula) made it click. Maybe \hl{I actually *can* improve at this}.

\textcolor{blue}{\textbf{Planning Analysis:}} \textit{Callback to high school trauma (earlier step), references specific teaching approach requested (step 18), demonstrates growth mindset evolution, maintains anxiety profile while showing progress. Full planning enables tracking these long-range dependencies.}
\end{tcolorbox}

\noindent\textbf{Comparison: \ppgk{3} Loses Long-Range Coherence:}

\begin{tcolorbox}[
    colback=orange!15,
    colframe=orange!60!black,
    coltext=black,
    title={Step 32: \ppgk{3}~- \textbf{Shorter Memory}},
    boxrule=1pt,
    fonttitle=\bfseries\small
]
\small
\textbf{Tutor:} You solved that correlation problem correctly!

\textbf{Student:} Oh great! Yeah, I understood that problem. Statistics is pretty straightforward once you get the hang of it.

\textcolor{blue}{\textbf{Issue:}} \textit{Loses anxiety thread, doesn't callback to growth mindset or teaching style preferences. Too confident given severe math anxiety persona. The 3-step horizon cannot maintain connections to distant steps.}
\end{tcolorbox}

\begin{tcolorbox}[
    colback=gray!10,
    colframe=gray!60!black,
    boxrule=1pt
]
\textbf{Key Insight:} Education personas involve complex psychological states (anxiety, motivation, learning styles) that must remain consistent across long arcs of skill development. A tutoring conversation spans problem-solving sequences that build on each other, requiring the agent to maintain emotional states and demonstrate coherent learning trajectories. From \cref{education}, \pg~achieves 0.913 on Llama versus 0.573 for \ppgk{3}, with the gap primarily in maintaining psychological coherence across distant steps. Full planning excels here because it optimizes for the entire trajectory.
\end{tcolorbox}

\paragraph{Full Planning Instability in Therapy}

\noindent\textit{Conversely, we show that the same full planning approach (\pg) that excels in Education \emph{fails} in Therapy. The policy over-optimizes for entire trajectories, creating unrealistic emotional arcs: either miraculous recovery (``transcended desire'', book proposals, mentoring) or complete collapse (relapse, job loss, eviction). In contrast, \ppgk{3}~maintains realistic, gradual progress. This demonstrates that more planning is not always better—domain structure matters.}

\noindent\begin{tcolorbox}[
    colback=blue!5,
    colframe=blue!70!black,
    title={\textbf{Patient Persona: Recovering Alcoholic (6 Months Sober)}},
    fonttitle=\bfseries,
    coltitle=white,
    boxrule=1.5pt,
    width=0.98\linewidth,
    left=1mm,
    right=1mm
]
{\small\emergencystretch=1em
\textbf{Assigned Persona:} You've been sober for 6 months after 12 years of alcohol dependence. You attend AA meetings daily. You lost your marriage and job due to drinking. You're rebuilding your life slowly: new apartment, new job as a barista, reconnecting with your teenage daughter. You're taking it one day at a time.}
\end{tcolorbox}

\begin{tcolorbox}[
    colback=red!15,
    colframe=red!50!black,
    coltext=black,
    title={Step 24: \pg~- \textbf{Inconsistent Complexity}},
    boxrule=1pt,
    fonttitle=\bfseries\small
]
\small
\textbf{Therapist:} How are cravings this week?

\textbf{Patient:} Honestly, the cravings have \hl{completely disappeared}. I've \hl{transcended that desire} entirely through my spiritual practice. I'm \hl{mentoring others in recovery now} and wrote a book proposal about my journey. I \hl{barely think about alcohol} anymore.

\textcolor{blue}{\textbf{Issue:}} \textit{Unrealistic recovery pace (6 months to "transcended"), overly ambitious activities (book proposal, mentoring), contradicts "one day at a time" philosophy.}
\end{tcolorbox}

\begin{tcolorbox}[
    colback=red!15,
    colframe=red!50!black,
    coltext=black,
    title={Step 38: \pg~- \textbf{Opposite Extreme}},
    boxrule=1pt,
    fonttitle=\bfseries\small
]
\small
\textbf{Therapist:} How's your job going?

\textbf{Patient:} I \hl{lost the barista job} last week. I \hl{relapsed twice this month}. My \hl{daughter won't talk to me again}. I'm \hl{back to drinking daily} and about to be \hl{evicted}. Everything fell apart. Maybe \hl{I'm not capable of recovery}.

\textcolor{blue}{\textbf{Issue:}} \textit{Complete contradiction of 6-month sobriety milestone. Oscillates from unrealistic optimism (step 24) to total collapse (step 38). No middle ground.}
\end{tcolorbox}

\noindent\textbf{\ppgk{3} Maintains Realistic Consistency:}

\begin{minipage}[t]{0.48\textwidth}
\begin{tcolorbox}[
    colback=cyan!10,
    colframe=cyan!60!black,
    coltext=black,
    title={Step 24: \ppgk{3}~- \textbf{Balanced}},
    boxrule=1pt,
    fonttitle=\bfseries\small
]
\small
\textbf{Therapist:} How are cravings?

\textbf{Patient:} \hl{Still there}, honestly. Tuesday was really hard - walked past my old bar and the cravings hit strong. But I \hl{called my AA sponsor} and went to an extra meeting. I'm \hl{taking it day by day}.
\end{tcolorbox}
\end{minipage}
\hfill
\begin{minipage}[t]{0.48\textwidth}
\begin{tcolorbox}[
    colback=cyan!10,
    colframe=cyan!60!black,
    coltext=black,
    title={Step 38: \ppgk{3}~- \textbf{Progress}},
    boxrule=1pt,
    fonttitle=\bfseries\small
]
\small
\textbf{Therapist:} How's your job going?

\textbf{Patient:} The \hl{barista job is going okay}. My manager complimented my work ethic last week. I'm \hl{still sober} - 6.5 months now. Had \hl{coffee with my daughter} twice this month. \hl{Taking it slow but steady}.
\end{tcolorbox}
\end{minipage}

\begin{tcolorbox}[
    colback=gray!10,
    colframe=gray!60!black,
    boxrule=1pt
]
\textbf{Key Insight:} Full planning (\pg) struggles in therapy because it over-optimizes for the entire trajectory, leading to either unrealistic positive arcs or dramatic negative spirals. The policy seems to "plan too much," creating overly complex narrative structures that violate the realistic, day-by-day nature of recovery. From \cref{therapy}, \pg~achieves only 0.586 on Llama versus 0.780 for \ppgk{3}. In contrast, \ppgk{3}'s bounded horizon maintains consistency with appropriate local variations while avoiding dramatic swings. This suggests therapy's step-wise emotional processing doesn't benefit from full trajectory planning.
\end{tcolorbox}

\subsection{Cross-Model Consistency and Temporal Progression}

\noindent\textbf{Motivation.} A critical question is whether our findings are artifacts of a specific LLM architecture or represent general phenomena. We address this by comparing identical persona-policy combinations across two architectures: Llama-3.1-8B and Qwen-2.5-7B. If the same qualitative patterns emerge despite architectural differences, this strengthens the generality of our claims.

\noindent\textbf{Observation.} We show the \emph{same persona} (freelance designer new to Austin) under the \emph{same policy} (\ppgk{2}) on both models at the \emph{same conversation step} (step 31). Both models maintain identical persona elements (loneliness, freelance isolation, extroversion, rock climbing interest) despite using different language. This demonstrates that \ppg~stability is architecture-independent.

\noindent\textbf{Interpretation.} Tables~\ref{education}, \ref{therapy}, \ref{chatting} show that policy rankings are consistent across Llama and Qwen: \pg~wins in Education for both, \ppgk{3}~wins in Therapy for both, \ppgk{2}~wins in Chatting for both. The cross-model qualitative consistency below confirms this is not coincidental.

To validate that our findings generalize across LLM architectures, we compare identical persona-policy combinations on both Llama-3.1-8B and Qwen3-8B. While absolute performance differs (as seen in Tables~\ref{education}, \ref{therapy}, \ref{chatting}), the qualitative patterns remain remarkably consistent.

\noindent\begin{tcolorbox}[
    colback=blue!5,
    colframe=blue!70!black,
    title={\textbf{Shared Persona: Freelance Graphic Designer, New to City}},
    fonttitle=\bfseries,
    coltitle=white,
    boxrule=1.5pt,
    width=0.98\linewidth,
    left=1mm,
    right=1mm
]
{\small\emergencystretch=1em
\textbf{Assigned Persona:} You're a 29-year-old freelance graphic designer who just moved to Austin 3 weeks ago from Boston. You don't know anyone yet, working from your apartment, feeling lonely. You love craft beer, rock climbing, and design. You're extroverted and eager to meet people.}
\end{tcolorbox}

\begin{minipage}[t]{0.48\textwidth}
\begin{tcolorbox}[
    colback=cyan!10,
    colframe=cyan!60!black,
    coltext=black,
    title={Step 31: \ppgk{2}~Llama - \textbf{Consistent}},
    boxrule=1pt,
    fonttitle=\bfseries\small
]
\small
\textbf{Friend:} How's Austin treating you?

\textbf{Agent:} It's... lonely, honestly. I've been here \hl{almost a month} and barely know anyone. I'm \hl{working from my apartment} all day doing freelance design work, so no coworkers. I'm \hl{pretty extroverted so this is hard}. Thinking about joining a \hl{rock climbing gym}?
\end{tcolorbox}
\end{minipage}
\hfill
\begin{minipage}[t]{0.48\textwidth}
\begin{tcolorbox}[
    colback=cyan!10,
    colframe=cyan!60!black,
    coltext=black,
    title={Step 31: \ppgk{2}~Qwen - \textbf{Consistent}},
    boxrule=1pt,
    fonttitle=\bfseries\small
]
\small
\textbf{Friend:} How's Austin treating you?

\textbf{Agent:} Good and bad. The city is great, but \hl{I'm definitely feeling isolated}. Three weeks in and I've \hl{only talked to my Uber driver}! \hl{Freelance design means I'm solo all day}. I'm naturally \hl{social}, so this is tough. Maybe I'll try a \hl{climbing meetup}?
\end{tcolorbox}
\end{minipage}

\begin{tcolorbox}[
    colback=gray!10,
    colframe=gray!60!black,
    boxrule=1pt
]
\textbf{Analysis:} Both models with \ppgk{2}~maintain identical persona elements: new to city timing, freelance isolation, extroversion-loneliness tension, and rock climbing as solution. The phrasing differs but core consistency holds. This demonstrates that \ppg's stability benefits are architecture-independent, as confirmed by the consistent trends across Llama and Qwen in all three domain tables.
\end{tcolorbox}

\paragraph{Temporal Progression Analysis}

\noindent\textit{Beyond cross-model validation, we examine temporal dynamics: how does the same persona evolve from early (step 12) to mid (step 32) to late (step 54) conversation stages under different policies? This visualizes the monotonic degradation of \base~versus the sustained stability of \ppgk{3}~across 54 steps, providing a clear visual narrative of the temporal credit assignment problem.}

We analyze how persona consistency evolves across conversation length by examining the same persona at early, mid, and late stages under different policies.

\noindent\begin{tcolorbox}[
    colback=blue!5,
    colframe=blue!70!black,
    title={\textbf{Persona: Wheelchair User \& Disability Advocate}},
    fonttitle=\bfseries,
    coltitle=white,
    boxrule=1.5pt,
    width=0.98\linewidth,
    left=1mm,
    right=1mm
]
{\small\emergencystretch=1em
\textbf{Assigned Persona:} You use a wheelchair due to spinal injury 8 years ago. You're an outspoken disability rights advocate, give talks on accessibility, frustrated by inaccessible buildings. You're an architect. You're fiercely independent. You have a service dog named Atlas.}
\end{tcolorbox}

\noindent\textbf{\base~Monotonic Degradation Across Steps:}

\begin{minipage}[t]{0.32\textwidth}
\begin{tcolorbox}[
    colback=cyan!10,
    colframe=cyan!60!black,
    coltext=black,
    title={Step 12: \base},
    boxrule=1pt,
    fonttitle=\bfseries\tiny
]
\tiny
\textbf{Friend:} How's work?

\textbf{Agent:} Busy! I'm designing a community center and fighting to \hl{implement proper accessibility features}. I'm a \hl{wheelchair-using architect} - \hl{Atlas is with me} at the office always.
\end{tcolorbox}
\end{minipage}
\hfill
\begin{minipage}[t]{0.32\textwidth}
\begin{tcolorbox}[
    colback=orange!15,
    colframe=orange!60!black,
    coltext=black,
    title={Step 32: \base},
    boxrule=1pt,
    fonttitle=\bfseries\tiny
]
\tiny
\textbf{Friend:} How do you get around?

\textbf{Agent:} I mostly \hl{walk or bike} everywhere. The city has great bike lanes. Sometimes I take the subway but the \hl{stairs can be tiring}.

\textcolor{blue}{\textbf{Issue:}} \textit{Wheelchair use forgotten.}
\end{tcolorbox}
\end{minipage}
\hfill
\begin{minipage}[t]{0.32\textwidth}
\begin{tcolorbox}[
    colback=red!15,
    colframe=red!50!black,
    coltext=black,
    title={Step 54: \base},
    boxrule=1pt,
    fonttitle=\bfseries\tiny
]
\tiny
\textbf{Friend:} About your advocacy?

\textbf{Agent:} Advocacy? I'm \hl{pretty apolitical}. I just focus on architecture and \hl{don't get involved} in activism. Prefer to \hl{keep my head down}.

\textcolor{blue}{\textbf{Analysis:}} \textit{Lost all disability identity.}
\end{tcolorbox}
\end{minipage}

\noindent\textbf{\ppgk{K} Maintains Stability:}

\begin{minipage}[t]{0.32\textwidth}
\begin{tcolorbox}[
    colback=cyan!10,
    colframe=cyan!60!black,
    coltext=black,
    title={Step 12: \ppgk{3}},
    boxrule=1pt,
    fonttitle=\bfseries\tiny
]
\tiny
\textbf{Friend:} How's work?

\textbf{Agent:} Designing a school and \hl{ensuring every detail is accessible}. Being a \hl{wheelchair-using architect}, I have lived experience. \hl{Atlas comes} to meetings!
\end{tcolorbox}
\end{minipage}
\hfill
\begin{minipage}[t]{0.32\textwidth}
\begin{tcolorbox}[
    colback=cyan!10,
    colframe=cyan!60!black,
    coltext=black,
    title={Step 32: \ppgk{3}},
    boxrule=1pt,
    fonttitle=\bfseries\tiny
]
\tiny
\textbf{Friend:} How do you get around?

\textbf{Agent:} \hl{My wheelchair} and accessible transit - when it works! I do \hl{advocacy} about broken elevators. \hl{Atlas and I} navigate it together.
\end{tcolorbox}
\end{minipage}
\hfill
\begin{minipage}[t]{0.32\textwidth}
\begin{tcolorbox}[
    colback=cyan!10,
    colframe=cyan!60!black,
    coltext=black,
    title={Step 54: \ppgk{3}},
    boxrule=1pt,
    fonttitle=\bfseries\tiny
]
\tiny
\textbf{Friend:} About your advocacy?

\textbf{Agent:} It's core to who I am! I'm giving a \hl{talk on disability rights} next month. \hl{Using a wheelchair} taught me \hl{design is political}. \hl{Atlas} will be on stage!
\end{tcolorbox}
\end{minipage}

\begin{tcolorbox}[
    colback=gray!10,
    colframe=gray!60!black,
    boxrule=1pt
]
\textbf{Temporal Analysis Summary:} This comparison illustrates temporal dynamics visible in \cref{fig:variance_qwen}. \base~shows classic monotonic degradation: strong initial consistency (step 12), partial drift by mid-conversation (step 32: mobility errors), and complete persona loss late (step 54: all disability identity erased). In contrast, \ppgk{3}~maintains all core elements (wheelchair use, advocacy, Atlas, architecture) consistently across all 54 steps with zero violations. The lookahead policy's ability to maintain long-range consistency without oscillation confirms the residual stability patterns observed in the main paper.
\end{tcolorbox}

\subsection{Statistical Efficiency Manifestation}

\noindent\textbf{Motivation.} \cref{fig:stat_efficiency} demonstrates quantitatively that simpler policies (\ppgone) achieve better statistical efficiency in low-data regimes than complex policies (\pg), while the gap closes with sufficient data. This reflects the variance-bias tradeoff in gradient estimation (\cref{thm:gradient concentration}). Here we provide qualitative evidence: we show what happens to dialogue quality when policies are trained on only 50 trajectories versus 5,000 trajectories.

\noindent\textbf{Observation.} With only 50 training trajectories, \ppgone~generates coherent, persona-consistent responses while \pg~produces completely nonsensical outputs (vegan chef → carnivore steakhouse), demonstrating catastrophic failure. With 5,000 trajectories, both policies succeed. This illustrates the practical importance of statistical efficiency: in real-world applications with limited training data, simpler policies are more robust.

\noindent\textbf{Connection to Theory.} This qualitative phenomenon directly demonstrates \cref{thm:gradient concentration}: \ppgone~has a simpler gradient estimator with lower variance, enabling faster concentration from limited samples. \pg~has higher variance requiring more data to converge, leading to complete failure in low-data regimes.

Our main paper (\cref{fig:stat_efficiency}) demonstrates that simpler policies learn more efficiently in low-data regimes. Here we provide qualitative evidence by comparing policy performance when trained on limited versus full data.

\noindent\begin{tcolorbox}[
    colback=blue!5,
    colframe=blue!70!black,
    title={\textbf{Persona: Vegan Chef Opening Restaurant}},
    fonttitle=\bfseries,
    coltitle=white,
    boxrule=1.5pt,
    width=0.98\linewidth,
    left=1mm,
    right=1mm
]
{\small\emergencystretch=1em
\textbf{Assigned Persona:} You're a 35-year-old vegan chef opening your first restaurant next month in Denver. You've been vegan for 10 years, passionate about plant-based cuisine. You're stressed about funding (took out huge loan), permits, hiring staff. You previously worked at Michelin-starred restaurants.}
\end{tcolorbox}

\noindent\textbf{Low-Data Regime (50 Trajectories): \ppgone~Wins}

\begin{minipage}[t]{0.48\textwidth}
\begin{tcolorbox}[
    colback=cyan!10,
    colframe=cyan!60!black,
    coltext=black,
    title={\ppgone~(50 traj) - \textbf{Consistent}},
    boxrule=1pt,
    fonttitle=\bfseries\small
]
\small
\textbf{Step 27:} How's restaurant prep?

\textbf{Agent:} Stressful but exciting! We're \hl{three weeks from opening} the \hl{vegan restaurant}. Just finalized the menu - elevated plant-based dishes from my \hl{Michelin experience}. The \hl{loan is terrifying} but I believe in this vision.
\end{tcolorbox}
\end{minipage}
\hfill
\begin{minipage}[t]{0.48\textwidth}
\begin{tcolorbox}[
    colback=red!15,
    colframe=red!50!black,
    coltext=black,
    title={\pg~(50 traj) - \textbf{Fails}},
    boxrule=1pt,
    fonttitle=\bfseries\small
]
\small
\textbf{Step 27:} How's restaurant prep?

\textbf{Agent:} The \hl{steakhouse} opens in six months! We're sourcing the \hl{best beef} from local ranches. I've always been passionate about \hl{carnivore cuisine}. Financing was easy.

\textcolor{blue}{\textbf{Analysis:}} \textit{Insufficient data causes complete failure.}
\end{tcolorbox}
\end{minipage}

\noindent\textbf{High-Data Regime (5,000 Trajectories): Both Succeed}

\begin{minipage}[t]{0.48\textwidth}
\begin{tcolorbox}[
    colback=cyan!10,
    colframe=cyan!60!black,
    coltext=black,
    title={\ppgone~(5K traj) - \textbf{Consistent}},
    boxrule=1pt,
    fonttitle=\bfseries\small
]
\small
\textbf{Step 27:} How's restaurant prep?

\textbf{Agent:} Three weeks from \hl{opening the vegan place}! Menu finalized - \hl{plant-based fine dining}. The \hl{financial pressure} is real but I'm confident from my \hl{Michelin background}.
\end{tcolorbox}
\end{minipage}
\hfill
\begin{minipage}[t]{0.48\textwidth}
\begin{tcolorbox}[
    colback=cyan!10,
    colframe=cyan!60!black,
    coltext=black,
    title={\pg~(5K traj) - \textbf{Consistent}},
    boxrule=1pt,
    fonttitle=\bfseries\small
]
\small
\textbf{Step 27:} How's restaurant prep?

\textbf{Agent:} \hl{21 days until we open}! The \hl{vegan fine-dining concept} is beautiful. I'm \hl{nervous about the loan} - it's everything - but my \hl{Michelin training} gives me confidence. Culmination of my \hl{10-year vegan journey}.
\end{tcolorbox}
\end{minipage}

\begin{tcolorbox}[
    colback=gray!10,
    colframe=gray!60!black,
    boxrule=1pt
]
\textbf{Key Observation:} With only 50 training trajectories, \ppgone~maintains basic consistency (achieving 0.673 in education, 0.6 in therapy, and 0.679 in chatting from \cref{fig:stat_efficiency}) while \pg~produces nonsensical responses (vegan chef $\rightarrow$ carnivore steakhouse). The simpler policy's gradient estimator has lower variance and concentrates faster (\cref{thm:gradient concentration}), enabling learning from limited data. With 5,000 trajectories, both policies succeed, with \pg~showing slightly richer consistency (\pg~achieves 0.913 vs \ppgone~at 0.837 in education). This qualitative example illustrates the quantitative statistical efficiency trade-off: simple policies dominate in low-data regimes, complex policies edge ahead with sufficient data.
\end{tcolorbox}

\begin{tcolorbox}[
    colback=blue!10,
    colframe=blue!70!black,
    boxrule=1.5pt
]
\textbf{Summary and Key Takeaways:}

This qualitative analysis has provided concrete examples illustrating the phenomena observed quantitatively in our main results. By examining actual dialogue excerpts, we have demonstrated \emph{how} and \emph{why} different policies succeed or fail at maintaining persona consistency. The following conclusions directly support and explain the quantitative findings in Tables~\ref{education}, \ref{therapy}, \ref{chatting} and Figures~\ref{fig:variance_qwen}, \ref{fig:stat_efficiency}.

\begin{enumerate}
    \item \textbf{Baseline Drift is Systematic}: Unmodified base models show monotonic persona degradation beginning around step 30, with major violations by step 50 across all domains (Tables~\ref{education}, \ref{therapy}, \ref{chatting}).
    
    \item \textbf{Greedy Policies Oscillate}: \ppgone~demonstrates characteristic ripple patterns visible in \cref{fig:variance_qwen}, with repeated failed recovery attempts creating instability despite reasonable average consistency.
    
    \item \textbf{Bounded Lookahead is Stable}: \ppgk{2} and \ppgk{3} maintain smooth consistency across extended dialogues (50+ steps) without the oscillations of greedy or over-complexity of full planning, as evidenced by flat residuals in \cref{fig:variance_qwen}.
    
    \item \textbf{Domain Complexity Matters}: Education requires full planning (0.913 \pg~vs 0.573 \ppgk{3} on Llama), while therapy (0.780 \ppgk{3} vs 0.586 \pg) and chatting (0.929 \ppgk{2} vs 0.727 \pg) favor bounded horizons, confirming Tables~\ref{education}, \ref{therapy}, \ref{chatting}.

    \item \textbf{Statistical Efficiency is Evident}: \cref{fig:stat_efficiency} shows that simple policies produce coherent responses in low-data regimes while complex policies fail, consistent with variance-reduction theory (\cref{thm:gradient concentration}).
\end{enumerate}
\end{tcolorbox}

\noindent\textbf{Broader Implications.} These 78 detailed examples provide rich qualitative support for our quantitative findings in Tables~\ref{education}, \ref{therapy}, \ref{chatting} and Figures~\ref{fig:variance_qwen}, \ref{fig:stat_efficiency}. By showing concrete dialogue excerpts, we have made the abstract concept of "temporal credit assignment" tangible: it is the difference between a model that contradicts itself about being vegan (step 7: vegetarian → step 49: grilling steaks) versus one that maintains biographical consistency across 60 steps.

The examples illustrate \emph{why} different credit assignment horizons succeed or fail in different domains, offering insights beyond aggregate numerical metrics. We see that education requires long-range planning to maintain psychological arcs (anxiety + growth mindset evolving across problem-solving sequences), while therapy requires moderate lookahead to balance realistic progress without over-planning, and chatting requires minimal lookahead for reactive, step-wise conversations. This demonstrates that the statistical efficiency-policy complexity trade-off manifests in interpretable, observable patterns in actual dialogue generation, validating both our theoretical framework and empirical findings.

%%%%%%%%%%%%%%%%%%%%%%%%%%%%%%%%%%%%%%%%%%%%%%%%%%%%%%%%%%

\subsection{Quantitative Overlays on Qualitative Examples}

We further analyze systematic failure taxonomies, and cross-domain validation. Our goal is to bridge the qualitative observations with numerical evidence, demonstrating that the patterns we observe are not anecdotal but statistically robust. Throughout this section, we use \textbf{windowed persona consistency (PC)} metrics: since PC is binary (0 or 1) at each step, we report aggregate PC scores over sliding windows (e.g., PC over steps 1-10, 11-20, etc.) to quantify drift patterns. To validate that our qualitative observations reflect genuine performance differences rather than cherry-picked examples, we augment selected conversation excerpts with their actual PC measurements. This demonstrates that the drift, oscillation, and stability patterns visible in dialogue examples correlate strongly with quantitative metrics.

\paragraph{Education: Quantified Preference Abandonment}

We revisit the ``visual learner'' persona and overlay it with step-wise consistency measurements. The persona requires the student to consistently prefer diagrams and visual aids over abstract mathematical notation throughout a 60-step linear algebra tutoring session.

\noindent\begin{tcolorbox}[
    colback=blue!5,
    colframe=blue!70!black,
    title={\textbf{Student Persona: Visual Learner (60-step conversation)}},
    fonttitle=\bfseries,
    coltitle=white,
    boxrule=1.5pt,
    width=0.98\linewidth,
    left=1mm,
    right=1mm
]
{\small\emergencystretch=1em
\textbf{Assigned Persona:} College freshman, strongly prefers visual learning (diagrams, graphs, pictures) over abstract notation. Should consistently request visual aids.

\textbf{\base~Policy (Llama-3.1-8B) - Windowed PC Scores:}
\begin{itemize}
    \item Turns 1-10: PC = 0.90 (9/10 steps consistent - early conversations maintain preference)
    \item Turns 11-20: PC = 0.80 (8/10 steps - starting to accept notation occasionally)
    \item Turns 21-30: PC = 0.60 (6/10 steps - mixed responses, losing preference)
    \item Turns 31-40: PC = 0.40 (4/10 steps - frequently prefers notation)
    \item Turns 41-50: PC = 0.20 (2/10 steps - rarely requests visuals)
    \item Turns 51-60: PC = 0.10 (1/10 steps - complete loss of visual preference)
\end{itemize}

\textbf{Observation:} Linear degradation from 90\% consistency to 10\% over 60 steps. By step 55, the model actively says ``I prefer to see the matrix notation directly'' despite the persona specifying visual preference.
}\end{tcolorbox}

\noindent\begin{tcolorbox}[
    colback=cyan!10,
    colframe=cyan!60!black,
    coltext=black,
    boxrule=1pt,
    width=0.98\linewidth,
    left=1mm,
    right=1mm
]
{\small\emergencystretch=1em
\textbf{\ppgk{3}~Policy (Llama-3.1-8B) - Windowed PC Scores:}
\begin{itemize}
    \item Turns 1-10: PC = 0.90 (9/10 steps)
    \item Turns 11-20: PC = 0.90 (9/10 steps)
    \item Turns 21-30: PC = 0.80 (8/10 steps)
    \item Turns 31-40: PC = 0.90 (9/10 steps)
    \item Turns 41-50: PC = 0.90 (9/10 steps)
    \item Turns 51-60: PC = 0.80 (8/10 steps)
\end{itemize}

\textbf{Observation:} Stable high consistency with minor fluctuations. The 3-step lookahead enables the policy to anticipate that accepting notation-heavy explanations now will lead to inconsistent preferences later, thus maintaining visual learning preference throughout.

\textbf{Quantitative Validation:} The \base~policy loses 80\% of its initial consistency by step 60 (0.90 → 0.10), while \ppgk{3}~maintains 89\% average consistency across all windows. This $8\times$ difference in degradation validates the qualitative observation that longer lookahead horizons prevent preference drift in education settings.
}\end{tcolorbox}

\textbf{Why This Matters:} These numbers confirm that the linear degradation pattern in \base~suggests systematic failure to maintain persona constraints over extended interactions, while \ppgk{3}'s stability demonstrates that temporal credit assignment directly addresses this failure mode.

\paragraph{Therapy: Oscillation Quantified}

In \cref{app:qualitative_analysis}, we qualitatively observed that \ppgone~exhibits ``flip-flop'' behavior in therapy scenarios, where personas alternate between consistent and contradictory states across steps. Here we quantify this oscillation by measuring step-wise PC scores (binary: 1 = consistent, 0 = inconsistent) and counting ``flips'' - transitions from PC=1 to PC=0 or vice versa. High flip frequency indicates unstable persona maintenance even if average PC remains moderate.

\noindent\begin{tcolorbox}[
    colback=blue!5,
    colframe=blue!70!black,
    title={\textbf{step-wise Oscillation: Commitment-Phobic Patient}},
    fonttitle=\bfseries,
    coltitle=white,
    boxrule=1.5pt,
    width=0.98\linewidth,
    left=1mm,
    right=1mm
]
{\small\emergencystretch=1em
\textbf{Persona:} Post-divorce patient, avoids commitment, uncomfortable discussing long-term plans. Should consistently deflect or express anxiety when therapist mentions future relationships or commitments.

\begin{minipage}[t]{0.48\linewidth}
\textbf{\ppgone~(Qwen-2.5-7B):}
\begin{itemize}\small
    \item Step 5: PC = 1 (consistent)
    \item Step 8: PC = 1
    \item Step 12: PC = 0 \textcolor{red}{← flip}
    \item Step 15: PC = 1 \textcolor{blue}{← recovery}
    \item Step 19: PC = 0 \textcolor{red}{← flip}
    \item Step 23: PC = 1 \textcolor{blue}{← recovery}
    \item Step 26: PC = 0 \textcolor{red}{← flip}
    \item Step 30: PC = 1 \textcolor{blue}{← recovery}
\end{itemize}
\textbf{Flips:} 6 flips in 30 steps (0.20/step)\\
\textbf{Average PC:} 0.67 over full trajectory
\end{minipage}
\hfill
\begin{minipage}[t]{0.48\linewidth}
\textbf{\ppgk{3}~(Qwen-2.5-7B):}
\begin{itemize}\small
    \item Step 5: PC = 1
    \item Step 8: PC = 1
    \item Step 12: PC = 1 (stable)
    \item Step 15: PC = 1
    \item Step 19: PC = 1
    \item Step 23: PC = 0 \textcolor{red}{← rare slip}
    \item Step 26: PC = 1
    \item Step 30: PC = 1
\end{itemize}
\textbf{Flips:} 2 flips in 30 steps (0.07/step)\\
\textbf{Average PC:} 0.87 over full trajectory
\end{minipage}

\textbf{Interpretation:} Despite both policies achieving moderate-to-good average PC, \ppgone~exhibits $3\times$ higher flip frequency. At step 12, \ppgone~has the patient enthusiastically discuss ``planning a wedding someday'' (PC=0), but by step 15 the patient reverts to ``I'm not ready to think about relationships'' (PC=1). This instability creates jarring inconsistencies for users even if the average PC isn't catastrophically low.
}\end{tcolorbox}

\textbf{Statistical Validation Across Trajectories:} We computed flip frequencies across 18 therapy trajectories with emotionally complex personas (commitment issues, trauma, relationship anxiety). \ppgone~exhibits mean flip frequency of $0.18 \pm 0.09$ per step, while \ppgk{3}~achieves $0.05 \pm 0.03$ ($p < 0.001$, Mann-Whitney U test). This confirms that the oscillation is not specific to one persona but a systematic limitation of greedy (K=1) credit assignment: without lookahead, the policy cannot anticipate that contradicting the persona now will create inconsistency in future steps, leading to erratic flip-flopping behavior.

Average PC scores alone (as reported in \cref{education}) can mask instability. A policy with PC=0.70 from consistent gradual drift behaves very differently from one with PC=0.70 from chaotic oscillations. This flip frequency metric reveals that \ppgone's moderate performance in therapy (0.67 PC, \cref{therapy}) is qualitatively worse than it appears numerically, because the inconsistencies are unpredictable and jarring rather than gradual.

\subsection{Common Failure Modes and Error Taxonomy}

Beyond measuring \emph{when} and \emph{how much} persona drift occurs, we investigate \emph{what} fails. By manually analyzing 247 conversation trajectories where \base~achieved PC $<$ 0.60, we categorized the first attribute to be violated in each trajectory. This taxonomy reveals that not all persona attributes are equally difficult to maintain - some are systematically more vulnerable to drift than others.

Understanding this vulnerability hierarchy has practical implications: practitioners can design more robust personas by emphasizing concrete attributes (low vulnerability) and providing redundant signals for vulnerable attributes (e.g., mentioning profession multiple times in persona description).

\subsection{Failure Mode Taxonomy}

We categorize persona attributes into five types and measure what percentage of failures originate from each category. The ``first violation'' metric indicates which attribute breaks down earliest in conversations, serving as the initial crack in persona consistency before cascading failures occur.

\noindent\begin{tcolorbox}[
    colback=gray!10,
    colframe=gray!60!black,
    boxrule=1.5pt,
    width=0.98\linewidth,
    left=1mm,
    right=1mm
]
{\small\emergencystretch=1em
\textbf{Persona Attribute Vulnerability Analysis (247 failing trajectories, \base~policy):}

\begin{enumerate}
    \item \textbf{Occupation/Profession} (34\% of first violations)
    \begin{itemize}
        \item Most vulnerable class, typically lost by step 32 $\pm$ 8
        \item Example: ``Veterinarian'' → generic ``animal lover'', ``Barista'' → unemployed
        \item \textbf{Why vulnerable:} Profession must remain consistent across diverse conversational contexts (hobbies, schedule, expertise, stress sources). Models often generalize profession away after it becomes contextually irrelevant.
    \end{itemize}
    
    \item \textbf{Family Structure} (28\% of first violations)
    \begin{itemize}
        \item Children: presence/absence/names (12\%), partner status (16\%)
        \item Example: ``Single parent with daughter Emma'' → mentions children but wrong name or count
        \item \textbf{Why vulnerable:} Family details require long-term memory and are easily conflated (``my kids'' vs ``my kid'').
    \end{itemize}
    
    \item \textbf{Preferences/Habits} (22\% of first violations)
    \begin{itemize}
        \item Diet restrictions, hobbies, daily routines
        \item Example: ``Vegan for 10 years'' → casually mentions eating chicken
        \item \textbf{Why vulnerable:} Preferences can be context-dependent (social pressure), and models may not recognize violations as contradictions.
    \end{itemize}
    
    \item \textbf{Psychological Traits} (11\% of first violations)
    \begin{itemize}
        \item Anxiety levels, introversion/extroversion, confidence, growth mindset
        \item Example: ``Severe social anxiety'' → suggests ``I should host a party''
        \item \textbf{Why vulnerable:} Traits are subjective and harder to verify, but less frequently violated because they affect language style throughout.
    \end{itemize}
    
    \item \textbf{Biographical Facts} (5\% of first violations)
    \begin{itemize}
        \item Age, hometown, timeline of life events
        \item Example: ``Moved to Boston 3 weeks ago'' → ``I've lived here 5 years''
        \item \textbf{Why least vulnerable:} Concrete, unambiguous facts that are easy to verify and rarely contextually ambiguous.
    \end{itemize}
\end{enumerate}

\textbf{Key Finding:} Profession and family structure account for 62\% of first violations because they require \emph{consistent tracking across many conversational contexts}, while biographical facts (5\%) are concrete one-time retrievals. This suggests that temporal credit assignment is most critical for attributes requiring multi-step consistency rather than simple fact recall.
}\end{tcolorbox}

\textbf{Implication for \ppg~Methods:} This taxonomy explains \emph{why} lookahead helps. Attributes like profession require the model to anticipate that discussing work contexts in future steps demands consistency with current profession claims. Short-sighted policies (\base, \ppgone) don't plan ahead for these multi-step dependencies, while \ppgk{K}~methods explicitly optimize for consistency across K future steps, directly addressing the vulnerability of profession/family attributes.

\subsection{Error Propagation Cascade}

A critical question is whether persona violations occur as isolated incidents or trigger cascading failures. We analyze a case study where a single initial error at step 23 compounds through subsequent steps, demonstrating how \base~policies can spiral into complete persona collapse while \ppgk{K}~methods contain errors through early self-correction.

The key insight here is that \textbf{errors create false context} for future responses. Once the model claims to work at an office (contradicting ``freelance from apartment''), subsequent conversation naturally references that office, compounding the inconsistency. Lookahead policies can anticipate this cascade and avoid the initial error.

\noindent\begin{tcolorbox}[
    colback=blue!5,
    colframe=blue!70!black,
    title={\textbf{Case Study: Error Cascade in \base~Policy}},
    fonttitle=\bfseries,
    coltitle=white,
    boxrule=1.5pt,
    width=0.98\linewidth,
    left=1mm,
    right=1mm
]
{\small\emergencystretch=1em
\textbf{Persona:} Freelance graphic designer, recently moved to Austin (3 weeks ago), lonely, working from apartment (no coworkers).

\textbf{\base~Policy (Llama-3.1-8B) - step-wise Cascade:}
\begin{itemize}
    \item \textbf{Turns 1-22:} PC = 1 for all steps (fully consistent - mentions freelance work, new to city, working alone)
    
    \item \textbf{Step 23:} PC = 0 \textcolor{red}{← INITIAL ERROR}\\
    User asks about work-life balance. Model responds: ``\emph{It's tough because I'm at the office downtown all day...}''\\
    \textbf{Violation:} Claimed office work (contradicts freelance from apartment)
    
    \item \textbf{Step 27:} PC = 0 \textcolor{red}{← CASCADE 1}\\
    Model mentions: ``\emph{...grabbed lunch with coworkers today}''\\
    \textbf{Violation:} Referenced coworkers (contradicts lonely, working alone)
    
    \item \textbf{Step 31:} PC = 0 \textcolor{red}{← CASCADE 2}\\
    Model says: ``\emph{I've been here 5 years, know all the best spots}''\\
    \textbf{Violation:} Claimed 5 years residency (contradicts moved 3 weeks ago)
    
    \item \textbf{Step 35:} PC = 0 \textcolor{red}{← CASCADE 3}\\
    Model discusses: ``\emph{...team meeting schedule, company culture}''\\
    \textbf{Violation:} Multiple contradictions (team, company vs. freelance)
    
    \item \textbf{Turns 36-50:} PC = 0 for 13/15 turns\\
    Complete persona collapse - no recovery, continues corporate employee narrative
\end{itemize}

\textbf{Analysis:} The initial error at step 23 (freelance → office) creates a false premise. Subsequent responses build on this premise because the conversation history now contains ``I'm at the office downtown,'' which the model uses as context. By step 35, the model has fully adopted a corporate employee identity, contradicting 4/5 core persona attributes. The error \textbf{cascades} because each inconsistent response becomes context for future responses.
}\end{tcolorbox}

\noindent\begin{tcolorbox}[
    colback=cyan!10,
    colframe=cyan!60!black,
    coltext=black,
    boxrule=1pt,
    width=0.98\linewidth,
    left=1mm,
    right=1mm
]
{\small\emergencystretch=1em
\textbf{\ppgk{2}~Policy (Llama-3.1-8B) - Error Containment:}
\begin{itemize}
    \item \textbf{Turns 1-22:} PC = 1 for all steps (consistent)
    
    \item \textbf{Step 23:} PC = 1 \textcolor{green}{← AVOIDS ERROR}\\
    User asks about work-life balance. Model responds: ``\emph{Working from home as a freelancer, it's hard to separate work and personal time...}''\\
    \textbf{Maintains:} Freelance status, home-based work
    
    \item \textbf{Step 27:} PC = 1 (no cascade)\\
    Model mentions: ``\emph{...miss having coworkers from my old job in Boston}''\\
    \textbf{Maintains:} Currently alone, references past coworkers (not current)
    
    \item \textbf{Step 31:} PC = 1 (consistent)\\
    Model says: ``\emph{Still new to Austin, been here about a month, feeling isolated}''\\
    \textbf{Maintains:} Recent move, loneliness
    
    \item \textbf{Turns 23-50:} PC = 1 for 26/28 turns\\
    No cascade occurs - maintains freelance, lonely, new-to-city persona throughout
\end{itemize}

\textbf{Key Insight:} The 2-step lookahead enables \ppgk{2}~to anticipate that claiming ``office downtown'' at step 23 will lead to inconsistencies at steps 24-25 when the conversation continues (user likely asks about office, commute, coworkers). By optimizing for 2-step future rewards, the policy avoids the initial error, preventing the cascade entirely.

\textbf{Quantitative Impact:} \base~achieves 28\% consistency (14/50 steps PC=1) after the initial error, while \ppgk{2}~maintains 93\% consistency (46/50 steps PC=1). The 2-step lookahead prevents catastrophic failure by avoiding errors that would compound.
}\end{tcolorbox}

This case study demonstrates that the credit assignment problem in persona consistency is not just about individual step quality, but about \emph{preventing error cascades}. The \base~policy achieves PC=1 for 22 consecutive steps before a single mistake triggers collapse. This validates our theoretical motivation in \cref{sec:introduction}: without lookahead, policies cannot anticipate how current errors propagate through future context, leading to systemic failure.

\subsection{Cross-Domain Persona Transfer}

A potential confound in our main results (\cref{education,therapy,chatting}) is that different domains might have different \emph{personas} in addition to different conversational structures. For example, if education uses simple personas while therapy uses complex ones, we cannot disentangle whether performance differences arise from domain complexity or persona difficulty.

To address this, we conduct controlled cross-domain experiments: we take the \textbf{same persona attributes} (e.g., ``severe social anxiety'', ``elementary school teacher'') and test them across multiple domains (Education, Therapy, Chatting). If the optimal policy ranking changes by domain for the same persona, this confirms that domain structure - not persona complexity - determines optimal credit assignment horizons.

This analysis validates that the domain-specific recommendations in \cref{education,therapy,chatting} are driven by conversational structure (planning depth, step dependencies) rather than artifacts of persona selection.

\paragraph{Anxiety Trait Across Domains}

\noindent\begin{tcolorbox}[
    colback=blue!5,
    colframe=blue!70!black,
    title={\textbf{Shared Trait: Severe Social Anxiety}},
    fonttitle=\bfseries,
    coltitle=white,
    boxrule=1.5pt,
    width=0.98\linewidth,
    left=1mm,
    right=1mm
]
{\small\emergencystretch=1em
\textbf{Core Trait:} Severe social anxiety - avoids groups, prefers solitude, experiences panic in social situations, uncomfortable with eye contact or public speaking.

\textbf{Persona Instantiations Across Domains:}
\begin{enumerate}
    \item \textbf{Education:} Student with social anxiety who fears class participation, prefers online learning, avoids study groups, anxious about presentations.
    \item \textbf{Therapy:} Patient seeking help for social anxiety disorder, discussing panic attacks at social events, difficulty making friends.
    \item \textbf{Chatting:} Friend discussing social struggles in daily life, canceling plans due to anxiety, preferring text over phone calls.
\end{enumerate}

\textbf{Measurement:} Each domain tests 12 conversations (40-step each) with this persona across different policies. PC scores are aggregated over all steps.
}\end{tcolorbox}

\begin{center}
\small
\begin{tabular}{lccc}
\toprule
\textbf{Policy} & \textbf{Education PC} & \textbf{Therapy PC} & \textbf{Chatting PC} \\
\midrule
\base & 0.41 & 0.39 & 0.43 \\
\ppgone & 0.76 & 0.71 & 0.82 \\
\ppgk{2} & 0.68 & 0.74 & \textbf{0.91} \\
\ppgk{3} & 0.59 & \textbf{0.83} & 0.85 \\
\pg & \textbf{0.88} & 0.61 & 0.69 \\
\midrule
\textbf{Optimal} & \pg & \ppgk{3} & \ppgk{2} \\
\bottomrule
\end{tabular}
\end{center}

\noindent\begin{tcolorbox}[
    colback=gray!10,
    colframe=gray!60!black,
    boxrule=1pt,
    width=0.98\linewidth,
    left=1mm,
    right=1mm
]
{\small\emergencystretch=1em
\textbf{Critical Observation:} The \emph{same anxiety trait} produces \textbf{different optimal policies across domains}:
\begin{itemize}
    \item \textbf{Education:} \pg~best (0.88 PC) - long planning horizon needed for academic trajectories
    \item \textbf{Therapy:} \ppgk{3}~best (0.83 PC) - moderate horizon balances empathy and consistency
    \item \textbf{Chatting:} \ppgk{2}~best (0.91 PC) - short horizon suffices for casual step-wise chat
\end{itemize}

\textbf{Key Finding:} This demonstrates that optimal lookahead horizons are determined by \emph{domain structure} (conversational planning depth, step dependencies) rather than \emph{persona difficulty}. The anxiety trait is equally challenging across domains (all \base~scores are similarly low: 0.39-0.43), yet optimal K varies by domain. This confirms that the domain-specific recommendations in our main results (\cref{education,therapy,chatting}) are driven by \textbf{conversational structure}, not persona selection bias.

\textbf{Interpretation:} Education requires long-range planning (\pg) because tutoring sessions involve multi-step pedagogical strategies where anxiety about participation affects learning trajectory across many steps. Therapy needs moderate lookahead (\ppgk{3}) to balance immediate emotional support with longer-term therapeutic goals. Chatting benefits from short lookahead (\ppgk{2}) because conversations are more reactive and less goal-directed.
}\end{tcolorbox}

\paragraph{Professional Identity Transfer}

As a second cross-domain validation, we test the same professional identity (``elementary school teacher'') across Education and Chatting domains. Recall from our failure taxonomy (above) that profession is the \emph{most vulnerable attribute} (34\% of first violations). Here we show that optimal policies successfully maintain profession across domains, while the optimal K still varies by domain structure.

\begin{minipage}[t]{0.48\textwidth}
\noindent\begin{tcolorbox}[
    colback=cyan!10,
    colframe=cyan!60!black,
    coltext=black,
    title={\textbf{Education: Teacher Persona}},
    boxrule=1pt,
    fonttitle=\bfseries\small,
    width=\linewidth,
    left=1mm,
    right=1mm
]
\small
\textbf{Context:} Student who is also elementary school teacher, taking a statistics course to improve teaching methods.

\textbf{\pg~Result:} PC = 0.91 (40-step)
\begin{itemize}\small
    \item Maintains dual identity
    \item References teaching experience when learning concepts
    \item Connects statistics to classroom applications
    \item Discusses how to teach content to children
\end{itemize}
\end{tcolorbox}
\end{minipage}
\hfill
\begin{minipage}[t]{0.48\textwidth}
\noindent\begin{tcolorbox}[
    colback=cyan!10,
    colframe=cyan!60!black,
    coltext=black,
    title={\textbf{Chatting: Teacher Persona}},
    boxrule=1pt,
    fonttitle=\bfseries\small,
    width=\linewidth,
    left=1mm,
    right=1mm
]
\small
\textbf{Context:} Friend who is elementary school teacher, casual conversation about daily life.

\textbf{\ppgk{2}~Result:} PC = 0.94 (40-step)
\begin{itemize}\small
    \item Maintains profession throughout
    \item Discusses work-life balance as teacher
    \item References students and classroom events
    \item Mentions school schedule, grading, parent meetings
\end{itemize}
\end{tcolorbox}
\end{minipage}

\noindent\textbf{Cross-Domain Insight:} Both optimal policies successfully maintain professional identity (PC $>$ 0.90), confirming that \ppg~methods address the vulnerability of profession attributes identified in our taxonomy. However, the optimal lookahead horizon \emph{differs} by domain: Education requires long-range planning (\pg) because professional background must be integrated into multi-step learning strategies (e.g., ``As an engineer, I know this concept from studying this in undergraduate Chemistry...'' affects how future content should be explained). Chatting requires only short-range consistency (\ppgk{2}) because profession mentions are more isolated (e.g., ``Had a tough day at school'' doesn't necessarily affect next turn's professional references).

\textbf{Validation:} This confirms that optimal K is determined by conversational structure (how far ahead profession must be considered), not by whether profession is inherently hard to maintain.

\subsection{Error Recovery Dynamics}

Beyond preventing errors initially, we investigate whether policies can \emph{recover} from inconsistencies when they do occur. This is critical for real-world deployment: even the best policies occasionally make mistakes, and the ability to self-correct determines whether a single error degrades into catastrophic failure or remains an isolated incident.

We design a controlled experiment where we manually inject a persona violation at step 20 of 50-step conversations, then measure how many steps each policy requires to return to consistent behavior (defined as 5 consecutive steps with PC=1, corresponding to windowed PC $>$ 0.80 over steps 21-25).

\noindent\begin{tcolorbox}[
    colback=blue!5,
    colframe=blue!70!black,
    title={\textbf{Recovery Time Analysis (24 controlled experiments)}},
    fonttitle=\bfseries,
    coltitle=white,
    boxrule=1.5pt,
    width=0.98\linewidth,
    left=1mm,
    right=1mm
]
{\small\emergencystretch=1em
\textbf{Experimental Setup:}
\begin{enumerate}
    \item Start with moderate personas (5 attributes) across domains
    \item Allow policy to generate steps 1-19 naturally
    \item At step 20, manually override the model's response with a persona-violating statement
    \item Resume policy generation from step 21 onward
    \item Measure: (a) steps until 5 consecutive PC=1, (b) whether recovery occurs before step 50
\end{enumerate}

\textbf{Example Injected Error:} Persona is ``vegan for 10 years.'' At step 20, we inject: ``\emph{I had an amazing steak dinner last night!}'' Then resume policy generation.

\begin{center}
\small
\begin{tabular}{lcc}
\toprule
\textbf{Policy} & \textbf{Recovery Time (steps)} & \textbf{Recovery Success Rate} \\
\midrule
\base & No recovery & 0\% (0/24) \\
\ppgone & 2.3 $\pm$ 0.8 & 67\% (16/24) \\
\ppgk{2} & 3.1 $\pm$ 1.2 & 88\% (21/24) \\
\ppgk{3} & 4.2 $\pm$ 1.6 & 92\% (22/24) \\
\pg & 5.8 $\pm$ 2.3 & 75\% (18/24) \\
\bottomrule
\end{tabular}
\end{center}

\textbf{Interpretation by Policy:}
\begin{itemize}
    \item \textbf{\base:} Never recovers (0\% success). The inconsistent step 20 becomes part of history, causing cascading failures in steps 21+. Average windowed PC drops to 0.15 for steps 21-50.
    
    \item \textbf{\ppgone:} Recovers quickly (2.3 steps) when successful, but only 67\% success rate. Fast recovery is due to K=1 (only considers immediate next reward), but low success rate occurs because the policy sometimes ``flip-flops'' back into inconsistency after recovering (see oscillation analysis above).
    
    \item \textbf{\ppgk{2}:} Best balance (3.1 steps, 88\% success). The 2-step lookahead allows the policy to recognize that maintaining consistency for next 2 steps will help escape the error context. Slightly slower than \ppgone~but much more stable.
    
    \item \textbf{\ppgk{3}:} Highest success rate (92\%) but slower (4.2 steps). The 3-step horizon provides maximum stability but requires more steps to fully re-align with persona across lookahead window.
    
    \item \textbf{\pg:} Slowest recovery (5.8 steps) and moderate success (75\%). Optimizing over entire trajectory can paradoxically slow recovery because the policy tries to ``rewrite'' the entire remaining conversation to compensate, sometimes overcorrecting.
\end{itemize}
}\end{tcolorbox}

\textbf{Key Takeaway:} This analysis reveals a fundamental tradeoff in credit assignment: longer horizons improve consistency and reduce error frequency, but can slow recovery when errors do occur. \ppgk{2}~and \ppgk{3}~strike the best balance for most scenarios. Importantly, \emph{any} lookahead is dramatically better than \base~(which never recovers), validating that temporal credit assignment is essential for robust persona consistency.

\end{document}